\documentclass[table]{article}
\pdfoutput=1

\usepackage{amsthm,amsmath,amsfonts,amssymb,graphicx,enumitem,array,tcolorbox,etoolbox,chngcntr}
\usepackage[numbers]{natbib}
\usepackage{hyperref}
\usepackage[font={scriptsize}]{caption}
\usepackage[caption=false]{subfig}
\usepackage[a4paper,left=3cm,right=3cm,top=3cm,bottom=3cm]{geometry}
\usepackage{authblk}
\usepackage[]{tocbibind}
\hypersetup{
    citecolor=blue,
    urlcolor=blue
}

\theoremstyle{plain}
\newtheorem{theorem}{Theorem}	
\newtheorem{proposition}[theorem]{Proposition}
\newtheorem{corollary}[theorem]{Corollary}
\counterwithout{figure}{section}
\theoremstyle{remark}
\newtheorem{definition}{Definition}

\def\mA{\mathbf{A}}
\def\mB{\mathbf{B}}
\def\me{\mathbf{e}}
\def\mE{\mathbf{E}}
\def\mF{\mathbf{F}}
\def\mG{\mathbf{G}}
\def\mH{\mathbf{H}}
\def\mI{\mathbf{I}}

\def\mK{\mathbf{K}}
\def\mL{\mathbf{L}}
\def\mM{\mathbf{M}}
\def\mN{\mathbf{N}}
\def\mP{\mathbf{P}}
\def\mQ{\mathbf{Q}}
\def\mR{\mathbf{R}}
\def\mT{\mathbf{T}}
\def\mU{\mathbf{U}}

\def\mV{\mathbf{V}}
\def\mW{\mathbf{W}}
\def\mx{\mathbf{x}}
\def\mX{\mathbf{X}}
\def\my{\mathbf{y}}
\def\mY{\mathbf{Y}}
\def\mz{\mathbf{z}}
\def\mZ{\mathbf{Z}}
\def\mLam{\mathbf{\Lambda}}
\def\mSig{\mathbf{\Sigma}}
\def\mPi{\boldsymbol{\Pi}}
\def\mpi{\boldsymbol{\pi}}
\DeclareMathOperator*{\argmin}{arg\,min}

\author{Andrew Jones and Patrick Rubin-Delanchy}
\affil{University of Bristol}
\title{The Multilayer Random Dot Product Graph}

\date{}

\begin{document}

\maketitle

\begin{abstract}
We present a comprehensive extension of the latent position network model known as the random dot product graph to accommodate multiple graphs---both undirected and directed---which share a common subset of nodes, and propose a method for jointly embedding the associated adjacency matrices, or submatrices thereof, into a suitable latent space. Theoretical results concerning the asymptotic behaviour of the node representations thus obtained are established, showing that after the application of a linear transformation these converge uniformly in the Euclidean norm to the latent positions with Gaussian error. Within this framework, we present a generalisation of the stochastic block model to a number of different multiple graph settings, and demonstrate the effectiveness of our joint embedding method through several statistical inference tasks in which we achieve comparable or better results than rival spectral methods. Empirical improvements in link prediction over single graph embeddings are exhibited in a cyber-security example. 
\end{abstract}

\tableofcontents

\section{Introduction}

Networks permeate the world in which we live, and so developing an accurate understanding of them is a matter of great interest to many branches of academia and industry, with applications as diverse as identifying patterns in brain scans \citep{fornito2016fundamentals} and the detection of fraudulent behaviour in the financial services sector \citep{akoglu2013anomaly}. The mathematical study of random graphs has its roots in the work of E. N. Gilbert \cite{G59} and, contemporaneously, Erd\"{o}s and R\'{e}nyi \cite{ER60}, who considered graphs in which edges between nodes occur independently according to Bernoulli random variables with a fixed probability $p$, in what can be considered the simplest probabilistic model of a naturally occurring network (this type of graph now being universally referred to as an \textit{Erd\"{o}s-R\'{e}nyi graph}). 

Among the more modern statistical treatments of networks is the concept of a \textit{latent position model} \cite{HRH02}, in which the $i$th node of a graph is mapped to a vector $\mX_i$ in some underlying latent space $\mathcal{X} \subseteq \mathbb{R}^d$, and (conditional on this choice of latent positions) the $i$th and $j$th nodes connect independently with probability $\kappa(\mX_i,\mX_j)$, where the \textit{kernel} function $\kappa: \mathcal{X} \times \mathcal{X} \to [0,1]$. The \textit{random dot product graph} (RDPG, \cite{nickel2006}, \cite{young_scheinerman}, \cite{AFTPPVLLQ17}), which uses the kernel function $\kappa(\mx,\my) = \mx^\top\my$, and its generalisation (GRDPG, \cite{RPTC18}), using the kernel function $\kappa(\mx,\my) = \mx^\top \mI_{p,q} \my$ (where $\mI_{p,q}$ is the diagonal matrix whose entries are $p$ ones followed by $q$ minus ones, where $p + q = d$) are of particular interest here due to their computational tractability and associated statistical estimation theory: spectral embedding \cite{VL07} the observed graph based on largest eigenvalues (respectively, largest-in-magnitude) produces uniformly consistent estimates of the latent positions $\mX_i$ of the RDPG (respectively, GRDPG) model up to orthogonal (respectively, indefinite orthogonal) transformation, with asymptotically Gaussian error \citep{sussman2012consistent,athreya2016limit,LTAPP17,cape_biometrika,CTP17,TP19,RPTC18,AFTPPVLLQ17}. The GRDPG can be used to effectively model networks which exhibit disassortative connectivity behaviour \cite{K10} in which dissimilar nodes are the more likely to connect to each other, and is therefore the preferred model for studying biological or technological networks, which typically exhibit such behaviour \cite{N02}. 

Recently, attention has turned to the joint study of \textit{multiple} graphs. Often, the graphs of interest share a common set of nodes but have different edges, and such a collection of graphs is known as a \textit{multilayer} network \citep{kivela2014multilayer} (more precisely, it is an example of a \textit{multiplex} network). This framework is of interest in the study of \textit{dynamic networks}, in which each of the graphs may represent a ``snapshot'' of a network at a given point in time, and has been used, for instance, for link prediction in cyber-security applications \cite{PBNH19}. Alternatively, one may be interested in detecting differences in node behaviour between graphs, and this approach was used identify regions of the brain associated with schizophrenia by comparing brain scans of both healthy and schizophrenic patients \cite{LATLP17}.

Latent position models readily extend to the study of multiple graphs by allowing the kernel functions $\kappa_r$ to vary, while retaining a common set of latent positions $\mX_i$ across the graphs, and in particular there exist several RDPG-based methods for working with multiple graphs. If each graph is drawn from the same distribution (that is, $\kappa_r(\mx,\my) = \mx^\top\my$ for each $r$) then one can consider the \textit{mean embedding} \cite{TKBCMVPS19} by spectrally embedding the average of the adjacency matrices $\bar{\mA} = \frac{1}{k}\sum_{r=1}^k \mA^{(r)}$, or the \textit{omnibus embedding} \cite{LATLP17}, in which each graph is assigned a different embedding in a common latent space. The mean embedding is known to perform well asymptotically at the task of estimating the latent positions \cite{TKBCMVPS19}, while the omnibus embedding is particularly suited to the task of testing whether the graphs are drawn from the same underlying distribution.

Other RDPG-based methods are more general, allowing different kernel functions $\kappa_r$ across the graphs. In the \textit{multiple random eigengraph} (MREG, \cite{WAVP19}) model, the kernel function $\kappa_r(\mx,\my) = \mx^\top\mLam_r\my$ is used, where the matrices $\mLam_r \in \mathbb{R}^{d \times d}$ are diagonal with non-negative entries. The \textit{multiple random dot product graph} (multi-RDPG, \cite{NW18}) loosens these restrictions by allowing the matrices $\mLam_r$ to be non-diagonal (but symmetric and positive definite), while requiring that the matrix of latent positions $\mX$ have orthonormal columns. Expanding on this is the \textit{common subspace independent edge} (COSIE, \cite{AACCPV19}) graph model which allows the matrices $\mLam_r$ to have positive \textit{and} negative eigenvalues, while still requiring them to be symmetric. Each of these models is proposed along with a spectral embedding technique for latent position estimation. Under the COSIE model, the adjacency matrix of each component graph is embedded separately, into a dimension $d_r$, say. A second, joint, spectral decomposition is then applied to the point clouds, to find a common embedding of dimension $d$. The approach requires estimation of each $d_r$ as well as $d$, for which a generic method based on the `elbow' of the scree-plot is suggested \citep{zhu2006automatic}.

Statistical discourse on network embedding is often inspired by the stochastic block model \citep{HLL83}, in which an unknown partition of the nodes exists so that the nodes of any group (or community) are statistically indistinguishable. Under this model, a network embedding procedure can reasonably be expected to ascribe identical positions to the nodes of one group, up to statistical error, and different embedding techniques can therefore be compared through the theoretical performance of an appropriate clustering algorithm at recovering the communities \citep{sarkar2015role,TP19,cape2019spectral}.

Of the approaches referred to above, only the COSIE model allows estimation of a generic multilayer stochastic block model \citep{HLL83}. For example, if one graph has assortative community structure (``birds of a feather flock together'') and the other disassortative (``opposites attract''), then the mean embedding can evidently eradicate all community structure visible in either individual graph.

As a model for multiple undirected graph embedding, the Multilayer Random Dot Product Graph model (MRDPG), presented here, is equivalent to the COSIE model in terms of its likelihood given latent positions, but the latent positions are themselves defined differently. The spectral embedding method to which this leads is materially different and simpler, while estimation performance is apparently superior (by numerical experiments). However, the MRDPG in fact allows for far greater generalization than the models we have discussed; it not only extends naturally to accommodate both symmetric and \textit{asymmetric} adjacency matrices (and thus allow us to extract information from \textit{directed} graphs) but also \textit{non-square} binary matrices, such as the non-zero off-diagonal blocks appearing in the adjacency matrix of a bipartite graph, by allowing the columns of such matrices to correspond to a second set of latent positions $\mY_i$. Moreover, provided that the rows of each matrix correspond to a common set of latent positions $\mX_i$, we may in fact allow the columns of each individual matrix to correspond to \textit{different} sets of latent positions $\mY^{(r)}_i$ (allowing us, for example, to study the behaviour of a particular collection of nodes in a network through observing their interactions---potentially in a number of different settings---with other collections of nodes).

We retain the use of the kernel functions $\kappa_r(\mx,\my) = \mx^\top\mLam_r\my$ for each graph, but now note that $\kappa_r: \mathcal{X} \times \mathcal{Y}_r \rightarrow [0,1]$ for subsets $\mathcal{Y}_r \subseteq \mathbb{R}^{d_r}$, where we allow \textit{arbitrary} matrices $\mLam_r \in \mathbb{R}^{d \times d_r}$ and impose no restriction of orthogonality on the matrices of latent positions $\mX$ and $\mY^{(r)}$.  These latent position matrices are either independent of each other, or else satisfy certain dependence criteria; such as the $\mY^{(r)}$ being equal---potentially up to some linear transformation--- with probability one (a full discussion of this is presented in Section \ref{subsec:asymptotics}). Given a collection of binary matrices $\mA^{(1)}, \ldots, \mA^{(k)}$ with each $\mA^{(r)} \in \{0,1\}^{n \times n_r}$, those latent positions are then estimated by the following procedure: we form a matrix $\mA$---which we refer to as the \textit{unfolding} of the matrices $\mA^{(r)}$---by adjoining the matrices $\mA^{(r)}$, and obtain left and right spectral embeddings of $\mA$ by scaling its left and right singular vectors by the square roots of the corresponding singular values. We refer to these embeddings as the Unfolded Adjacency Spectral Embeddings (UASEs) of the $\mA^{(r)}$. The left-sided embedding $\mX_\mA$ is our proposed estimate of $\mX$, that is, a \textit{single} embedding of the nodes that is common to all graphs; the right-sided embedding $\mY_\mA$ can be split into $k$ distinct embeddings $\mY^{(r)}_\mA$, which can be shown (under certain criteria) to provide estimates of the latent position matrices $\mY^{(r)}$.

We allow the matrices $\mLam_r$ to be of non-maximal rank, requiring instead that the matrix $\mLam = [\mLam_1|\ldots|\mLam_k]$ be of maximal rank. This allows for the situation in which information about the latent positions can be \textit{obscured} in individual graphs. As a simple example, consider a three-party political system in which members always vote along party lines, with graphs representing the outcome of votes on particular motions (in which members can either support or oppose the motion, and two members are linked if they vote in the same way). Suppose further that there are no coalitions (that is, every pair of parties has at least one motion on which they vote differently). Then any individual vote will only highlight two groups (those who support and those who oppose that particular motion) and it is only with knowledge of \textit{multiple} votes that one can correctly identify the individual parties. In our method, no intermediate estimate of the rank $d_r$ of $\mLam_r$ is required, in contrast to the COSIE-based approach. 

We investigate the asymptotic behaviour of the left- and right-sided embeddings of the unfolding $\mA$ under an MRDPG model. It is shown that, up to linear transformation, the rows of each embedding converge uniformly in the Euclidean norm to the latent positions $\mX_i$ (Theorem \ref{consistency}) and, through the derivation of a central limit theorem (Theorem \ref{CLT}), that these rows are distributed around their corresponding latent positions according to a Gaussian mixture model, and thereby significantly extending the existing results of \cite{AFTPPVLLQ17} and \cite{RPTC18} to our more general model. These distributional results show that, in particular, if the graphs are identically distributed then the transformed rows of the left-sided embedding have the same limiting distribution as those of the mean embedding (Corollaries \ref{identical_MRDPG_consistency}~and~\ref{identical_MRDPG_CLT}). Consequently, if multiple graphs are identically drawn according to a stochastic block model \cite{HLL83} then joint embedding will \textit{always} be more effective at the task of cluster separation than any individual graph embedding (Proposition \ref{Chernoff_identical_embeddings}), where we evaluate this effectiveness via the \textit{Chernoff information} \cite{C52} of the limiting Gaussian distributions of the embeddings. The Chernoff information belongs to the class of \textit{f}-divergences \citep{AS66,C67} and is therefore invariant under invertible linear transformations \cite{LV06}, an important requirement here since distributional results hold only up to such transformation. 

Using the framework of the MRDPG, we propose an extension to the multilayer stochastic block model \cite{HLL83} which we refer to as the \textit{generalised} multilayer stochastic block model (GMSBM). Given a collection of graphs---which may be either directed or undirected---each of which follows a stochastic blockmodel, the GMSBM allows us to study one or more communities common to all graphs by observing their interactions with \textit{all} communities. The point cloud obtained by jointly embedding the unfolding of the resulting adjacency submatrices then exhibits the asymptotic behaviour detailed in our main results; allowing us, for example, to determine whether we can further divide these communities given the information extracted from the multiple embeddings. We provide empirical evidence of the effectiveness of our embedding method at the task of community detection under the GMSBM, demonstrating that fitting a Gaussian mixture model to the UASE achieves better results than rival spectral methods in both the directed and undirected cases.

We assess the effectiveness of unfolded adjacency spectral embedding for the general MRDPG at the inference tasks of recovery of latent positions, estimation of the common invariant subspaces, estimation of the underlying probabilistic model and two-graph hypothesis testing in simulated data. We demonstrate that performance at the estimation tasks is often better than that of the multiple adjacency spectral embedding (the method proposed in \cite{AACCPV19} to embed multiple graphs distributed according to a COSIE model, which demonstrably yields state of the art performance at such tasks), while its performance at the latter task is comparable with that of the omnibus embedding for reasonably-sized graphs (those with at least 500 nodes). We also apply the UASE to the task of link prediction, using connectivity data from the Los Alamos National Laboratory computer network \cite{TKH18} to predict connections between computers across the entire network as an example of a dynamic link prediction inference task, before restricting our attention to connections occurring through specific ports, demonstrating that the majority of the time the UASE yields greater accuracy than individual adjacency spectral embeddings. As a final example, we show how incorporating connection data between computers and ports into our model can increase the accuracy of our link prediction method.

The remainder of this article is structured as follows. In Section \ref{sec:MRDPG} we present our model, and corresponding asymptotic results. In Section \ref{sec:GMSBM} these results are explored within the context of a stochastic block model, appropriately extended to different multiple graph settings, all special cases of the MRDPG. Section \ref{sec:experimental} presents a series of experiments comparing the performance of unfolded adjacency spectral embedding with rival methods at different inference tasks. Section \ref{sec:real_data} presents an example from a real computer network, from which multiple graphs are extracted by considering different time windows or different port numbers (indicating different types of network service), and are used together to improve link prediction. In Section \ref{sec:Chernoff}, we investigate the statistical gain of joint versus individual spectral embedding in terms of Chernoff information, proving there is definite improvement in one special (but interesting) case, leaving open a wider conjecture.  Finally, Section \ref{sec:conclusion} concludes.

\section{The multilayer random dot product graph}\label{sec:MRDPG}

\begin{definition}\label{def:MRDPG} (The multilayer random dot product graph model). 
\\For a positive integer $k$, fix matrices $\mLam_r \in \mathbb{R}^{d \times d_r}$ for each $r \in \{1,\ldots,k\}$ such that the matrix $\mLam = [\mLam_1|\cdots|\mLam_k]$ is of rank $d$, and fix bounded subsets $\mathcal{X}, \mathcal{Y}_1, \ldots, \mathcal{Y}_k$ of $\mathbb{R}^d, \mathbb{R}^{d_1}, \ldots, \mathbb{R}^{d_k}$ respectively such that for each $r$, $\mx^\top\mLam_r\my_r \in [0,1]$ for all $\mx \in \mathcal{X}$ and $\my_r \in \mathcal{Y}_r$.

Fix a joint distribution $\mathcal{F}$ supported on $\mathcal{X}^n \times \mathcal{Y}_1^{n_1} \times \cdots \times \mathcal{Y}_k^{n_k}$ for positive integers $n, n_1, \ldots, n_k$ and let $(\mX_1, \ldots, \mX_n,\mY^{(1)}_1, \ldots, \mY^{(k)}_{n_k}) \sim \mathcal{F}$. Define $\mX = [\mX_1 | \cdots | \mX_n]^\top \in \mathbb{R}^{n \times d}$ and $\mY = \mY^{(1)} \oplus \cdots \oplus \mY^{(k)}$, where each matrix $\mY^{(r)} = [\mY^{(r)}_1 | \cdots | \mY^{(r)}_{n_r}]^\top \in \mathbb{R}^{n_r \times d_r}$. Finally, define matrices $\mP^{(r)} = \mX\mLam_r{\mY^{(r)}}^\top \in [0,1]^{n \times n_r}$ for each $r$, and define the \textit{unfolding} $\mP = [\mP^{(1)}|\cdots|\mP^{(k)}]$ (which we note satisfies $\mP = \mX\mLam\mY^\top$).

Given a set of matrices $\mA^{(1)}, \ldots, \mA^{(k)}$ with each $\mA^{(r)} \in \{0,1\}^{n \times n_r}$, we similarly define the unfolding $\mA = [\mA^{(1)}|\cdots|\mA^{(k)}]$. We say that $(\mA,\mX,\mY) \sim \mathrm{MRDPG}(\mathcal{F},\mLam)$ if each matrix $\mA^{(r)}$ satisfies one of the following: 

\begin{itemize} 
	\item If, with probability one, there exists a matrix $\mG_r \in \mathbb{R}^{d \times d_r}$ such that $\mY^{(r)} = \mX\mG_r$, then conditional on $\mX$ and $\mY^{(r)}$, either:
	\begin{itemize}
		\item $\mA^{(r)}$ is hollow and symmetric, satisfying $\mA_{ij}^{(r)} \sim \mathrm{Bern}\bigl(\mP_{ij}^{(r)}\bigr)$ for all $i > j$, or
		\item $\mA^{(r)}$ is hollow and asymmetric, satisfying $\mA_{ij}^{(r)} \sim \mathrm{Bern}\bigl(\mP_{ij}^{(r)}\bigr)$ for all $i \neq j$.
	\end{itemize}
	\item If, for all matrices $\mG_r \in \mathbb{R}^{d \times d_r}$, the probability that $\mY^{(r)} = \mX\mG_r$ is strictly less than one then, conditional on $\mX$ and $\mY^{(r)}$, $\mA_{ij}^{(r)} \sim \mathrm{Bern}\bigl(\mP_{ij}^{(r)}\bigr)$ for all $i$ and $j$;
\end{itemize}
\end{definition}

We say that the graph corresponding to the matrix $\mA^{(r)}$ is \textit{``undirected''}, \textit{``directed''} or \textit{``bipartite''} to distinguish between these cases, and allow a mixture of these cases among the matrices $\mA^{(1)}, \ldots, \mA^{(k)}$ (we note that for the sake of brevity---particularly within the proofs located in the Appendix---we will occasionally abuse our terminology and refer to the \textit{matrices} $\mA^{(r)}$ using these terms).

The distributional results that we will obtain later can be refined if we restrict our attention to the case in which the matrices $\mA^{(1)}, \ldots, \mA^{(k)}$ are identically distributed (corresponding, say, to a multiple graph embedding in which we expect identical behaviour across the graphs). To this end, for a fixed matrix $\mLam \in \mathbb{R}^{d \times d'}$ of rank $d$ and a distribution $\mathcal{F}$, we write $(\mA,\mX,\mY) \stackrel{\mathrm{id}}{\sim} \mathrm{MRDPG}(\mathcal{F},\mLam)$ if, under the distribution $\mathcal{F}$, all of the matrices $\mY^{(r)}$ are equal with probability one to $\mY \in \mathbb{R}^{n' \times d'}$, and the matrices $\mLam_r$ are all equal to $\mLam$.

We note that there is a degree of ambiguity in the choice of latent positions for the MRDPG, as the following result shows:

\begin{proposition}\label{latent_positions} Let $(\mA,\mX, \mY) \sim \mathrm{MRDPG}(\mathcal{F},\mLam)$. Then:
\begin{enumerate}[label=\normalfont(\roman*)]
	\item $(\mA,\mX\mG^\top,\mY\mH^\top) \sim \mathrm{MRDPG}(\mathcal{F}_{\mG,\mH},\mG^{-\top} \mLam \mH^{-1})$ for any matrices $\mG \in \mathrm{GL}(d)$ and $\mH = \mH_1 \oplus \cdots \oplus \mH_k \in \mathrm{GL}(d_1) \times \cdots \times \mathrm{GL}(d_k)$, where the distribution $\mathcal{F}_{\mG,\mH}$ is derived from $\mathcal{F}$ by multiplying elements of $\mathcal{X}$ by $\mG$ and elements of $\mathcal{Y}_r$ by $\mH_r$ for each $r$.
	\item There is a joint distribution $\widetilde{F}$ and matrices of latent positions $\widetilde{\mX}$ and $\widetilde{\mY}$ where each vector $\widetilde{\mX}_i , \widetilde{\mY}^{(r)}_j \in \mathbb{R}^d$ such that $(\mA,\widetilde{\mX},\widetilde{\mY}) \sim \mathrm{MRDPG}(\widetilde{F},\mI_{d,k})$, where $\mI_{d,k} = [\mI_d | \cdots | \mI_d]$.
\end{enumerate}
\end{proposition}

\begin{proof}
\leavevmode
\begin{enumerate}[label=\normalfont(\roman*)]
	\item This follows from the fact that $(\mX\mG^\top)(\mG^{-\top}\mLam\mH^{-1})(\mY\mH^\top)^\top = \mLam$.
	\item Let $\mLam$ admit the singular value decomposition $\mLam = \mU\mSig\mV^\top$ with matrices $\mU \in \mathrm{O}(d)$, $\mSig \in \mathbb{R}^{d \times d}$ and $\mV \in \mathrm{O}((d_1 + \ldots + d_k) \times d)$. Then setting $\widetilde{\mX} = \mX\mU\mSig^{1/2}$ and $\widetilde{\mY} = \mY\mV\mSig^{1/2}$ gives the result. 
\end{enumerate}
\end{proof}

Note that under the second transformation, while $\widetilde{\mX}$ is always of maximal rank, the rank of $\widetilde{\mY}^{(r)}$ is equal to that of $\mLam_r$ and so it is possible to ``lose'' information about the latent positions $\mY^{(r)}$ in some sense by applying such a transformation.

Key to our study of the MRDPG will be the spectral embeddings of the unfolding $\mA$. Unlike the GRDPG, in which one considers a single symmetric adjacency matrix, whose left and right singular vectors therefore coincide, we obtain distinct embeddings by considering each side of $\mA$:

\begin{definition} (Unfolded adjacency spectral embeddings). 
\\Let $(\mA,\mX, \mY) \sim \mathrm{MRDPG}(\mathcal{F},\mLam)$, and let $\mA$ and $\mP$ admit singular value decompositions 
\begin{align}
	\mA = \mU_\mA \mSig_\mA \mV_\mA^\top + \mU_{\mA,\perp} \mSig_{\mA,\perp} \mV_{\mA,\perp}^\top, \quad \mP = \mU_\mP \mSig_\mP \mV_\mP^\top,
\end{align} 
where $\mU_\mA, \mU_\mP \in \mathbb{O}(n \times d)$, $\mV_\mA, \mV_\mP \in \mathbb{O}((n_1 + \ldots + n_k) \times d)$, and $\mSig_\mA, \mSig_\mP \in \mathbb{R}^{d \times d}$ are diagonal containing the largest singular values of $\mA$ and $\mP$ respectively. We then define the following:
\begin{itemize}
	\item The left UASE is the matrix $\mX_\mA \in \mathbb{R}^{n \times d}$ given by $\mX_\mA = \mU_\mA \mSig_\mA^{1/2}$.
	\item For $r \in \{1,\ldots,k\}$, the $r$th right UASE is the matrix $\mY^{(r)}_\mA \in \mathbb{R}^{n_r \times d}$ obtained by dividing $\mY_\mA = \mV_\mA\mSig_\mA^{1/2}$ into $k$ blocks of sizes $n_1 \times d, \ldots, n_k \times d$ (equivalently, $\mY^{(r)}_\mA = \mV^{(r)}_\mA\mSig_\mA^{1/2}$, where we divide $\mV_\mA$ into $k$ blocks $\mV^{(1)}_\mA, \ldots, \mV^{(k)}_\mA$). 
\end{itemize}
\end{definition}

We define the matrices $\mX_\mP$, $\mY_\mP$ and $\mY^{(r)}_\mP$ analogously.

The reader with a passing knowledge of tensor theory may wish to draw parallels between our notion of an unfolding and that of the matrix unfoldings of a tensor. For the uninitiated, any $3$-tensor $\mathcal{M} \in \mathbb{R}^{n_1 \times n_2 \times n_3}$ can be represented as a matrix in one of $3$ standard ways, each of which is known as an unfolding of $\mathcal{M}$ (a precise description of these can be found in \cite{LMV00}). The unfoldings of a tensor provide a more accessible means for studying its properties; in particular, there is a notion of a singular value decomposition for tensors (see \cite{LMV00},Theorem 2) in which the ``higher order'' analogues of singular values and vectors correspond to the standard matrix singular values and vectors of the unfoldings.

In the special case in which the matrices of latent positions $\mY^{(r)}$ are equal with probability one (and so we may consider the $\mA^{(r)}$ to be adjacency submatrices for some fixed subset of nodes across a series of graphs) then the matrices $\mA^{(r)}$ give rise to a $3$-tensor $\mathcal{A}$ by setting $\mathcal{A}_{ijr} = \mA^{(r)}_{ij}$. In this case, our notion of the unfolding of the matrices $\mA^{(r)}$ is a column permutations of the first standard unfolding of the tensor $\mathcal{A}$, which therefore shares the same left singular values and vectors, and is the natural choice to consider to allow us extend to the general case in which the matrices $\mY^{(r)}$ differ (the second standard unfolding corresponds to the matrix $\bigl[{\mA^{(1)}}^\top|\cdots|{\mA^{(k)}}^\top\bigr]$, while the third is the matrix whose $(i,j)$th entry is the Frobenius inner product $\langle \mA^{(i)},\mA^{(j)}\rangle_F$).

\subsection{Asymptotics and sparsity}\label{subsec:asymptotics}
In our asymptotic analysis of the behaviour of the UASE, we will assume that the number of graphs $k$ is either fixed, or at most grows at a rate much slower than $n$, in the sense that $\lim_{k,n \to \infty} \frac{k}{n} = 0$. We will also assume that the number of latent positions $\mY^{(r)}_i$ grows at a comparable rate to the number of latent positions $\mX_i$, in the sense that for each $r \in \{1,\ldots,k\}$, there exists a positive constant $c_r$ such that $\lim_{n_r,n \to \infty} \frac{n_r}{n} = c_r$.

Before proceeding further, we establish some (standard) notation relating to the asymptotic growth of various functions. In the following, $f$ and $g$ are real-valued functions of $n, n_1, \ldots, n_k$, and $X$ and $Y$ are real-valued random variables. We say that: 
\begin{itemize} 
	\item $f = \Omega(g)$ if there is a constant $c > 0$ and integers $N, N_1, \ldots, N_k$ such that for all $n \geq N$ and $n_r \geq N_r$, $f(n,n_1,\ldots,n_k) \geq cg(n,n_1,\ldots,n_k)$;
	\item $f = \mathrm{O}(g)$ if there is a constant $c > 0$ and integers $N, N_1, \ldots, N_k$ such that for all $n \geq N$ and $n_r \geq N_r$, $f(n,n_1,\ldots,n_k) \leq cg(n,n_1,\ldots,n_k)$;
	\item $f = \Theta(g)$ if both $f = \Omega(g)$ and $f = \mathrm{O}(g)$;
	\item $f = \omega(g)$ if there is a constant $c > 0$ and integers $N, N_1, \ldots, N_k$ such that for all $n \geq N$ and $n_r \geq N_r$, $f(n,n_1,\ldots,n_k) \geq cg(n,n_1,\ldots,n_k)$ and $\lim_{n,n_1,\ldots,n_k \to \infty} \left|\frac{f(n,n_1,\ldots,n_k)}{g(n,n_1,\ldots,n_k)}\right| = \infty$;
	\item $|X| = \mathrm{O}(f)$ \textit{almost surely} if, for any $\alpha > 0$, there is a constant $c > 0$ and integers $N, N_1, \ldots, N_k$ such that for all $n \geq N$ and $n_r \geq N_r$, $|X| \leq cf(n,n_1,\ldots,n_k)$ with probability at least $1 - n^{-\alpha}$;
	\item $|X| = \mathrm{O}(f)$ and $|Y| = \mathrm{O}(g)$ \textit{mutually almost surely} if, for any $\alpha > 0$, there is a constant $c > 0$ and integers $N, N_1, \ldots, N_k$ such that for all $n \geq N$ and $n_r \geq N_r$, the probability that both $|X| \leq cf(n,n_1,\ldots,n_k)$ and $|Y| \leq cg(n,n_1,\ldots,n_k)$ is at least $1 - n^{-\alpha}$. 
\end{itemize}

The inter- and intra-dependence of $\mX$ and the $\mY^{(r)}$ under the joint distribution $\mathcal{F}$ in our definition of the MRDPG has so far been left open. In order to establish asymptotic distributional results for the UASE, however, we must impose certain restrictions on $\mathcal{F}$. The first of these is that:

\begin{itemize}
	\item Each of the collections $(\mX_1, \ldots, \mX_n)$ and $(\mY^{(r)}_1, \ldots, \mY^{(r)}_{n_r})$ is marginally i.i.d., with marginal distributions $\mathcal{F}_X$ on $\mathcal{X}$ and $\mathcal{F}_{Y,r}$ on $\mathcal{Y}_r$ for each $r$.
\end{itemize}

Given such a joint distribution $\mathcal{F}$, let $\xi \sim \mathcal{F}_X$ and $\upsilon_r \sim \mathcal{F}_{Y,r}$ for each $r$. Our next requirement is that these marginal distributions be non-degenerate, in the sense that:

\begin{itemize}
	\item The second moment matrices $\Delta_X = \mathbb{E}[\xi\xi^\top] \in \mathbb{R}^{d\times d}$ and $\Delta_{Y,r} = \mathbb{E}[\upsilon_r\upsilon_r^\top] \in \mathbb{R}^{d_r \times d_r}$ for each $r$ are all invertible.
\end{itemize}

Our third and final requirement is a technical condition which ensures that the growth of the singular values of the unfolding $\mP$ is regulated. We note first that a standard application of Hoeffding's inequality shows that the spectral norm $\|\mX^\top\mX - n\Delta_X\|$ is of order $\mathrm{O}\left(n^{1/2}\log(n)\right)$ almost surely, and one can argue similarly that the spectral norms $\|{\mY^{(r)}}^\top\mY^{(r)} - n_r\Delta_{Y,r}\|$ are of order $\mathrm{O}(n_r^{1/2}\log(n_r))$ almost surely. Our final requirement is that these bounds are attained simultaneously with high probability, or in other words that:

\begin{itemize} 
	\item The matrices $\mX, \mY^{(1)}, \ldots, \mY^{(k)}$ satisfy $\|\mX^\top\mX - n\Delta_X\| = \mathrm{O}(n^{1/2}\log(n))$ and, for each $r$, $\|{\mY^{(r)}}^\top\mY^{(r)} - n_r\Delta_{Y,r}\| = \mathrm{O}(n_r^{1/2}\log(n_r))$ mutually almost surely.
\end{itemize}

This condition is satisfied, for example, when $\mX$ and the $\mY^{(r)}$ are independent, or when some or all of the $\mY^{(r)}$ are equal (possibly to $\mX$) with probability one. Due to submultiplicativity of the spectral norm, it also holds if we extend the latter example to allow equality up to linear transformation.

When studying the asymptotic behaviour of graph embeddings, it is typical to impose some control over how sparse or dense the graphs we are considering are allowed to be. This is often done through the introduction of a \textit{sparsity factor}, which is a real-valued function (depending on the size of the graph) which is either equal to $1$, or which tends to zero as the size of the graph increases (corresponding to dense and sparse regimes respectively). This factor can then be incorporated into the model \textit{either} by scaling the latent positions themselves, \textit{or} by scaling the kernel function. In the case of a single graph which follows a GRDPG, these two options are equivalent (and for example in \cite{RPTC18} the sparsity factor is used to scale the latent positions $\mX_i$). For multiple graphs, however, scaling the latent positions produces the same sparsity behaviour for \textit{all} graphs, and does not allow individual control over the density of each graph.

Our approach is to incorporate both a \textit{global} sparsity function $\rho$ and \textit{local} sparsity functions $\epsilon_r$ into our model, where $\rho$ applies a uniform scaling to the latent positions $\mX$ and $\mY^{(r)}$ and $\epsilon_r$ an individual scaling to the matrix $\mLam_r$. Given a joint distribution $\mathcal{F}$ satisfying the above conditions, let $(\xi_1, \ldots, \xi_n, \upsilon^{(1)}_1, \ldots, \upsilon^{(k)}_{n_k}) \sim \mathcal{F}$, and let $\rho, \epsilon_1, \ldots, \epsilon_k: \mathbb{Z}_+ \rightarrow [0,1]$, where each such function is either constant (and equal to $1$) or tends to zero as $n$ (and consequently the $n_r$) tend to infinity. Given matrices $\mLam_1, \ldots, \mLam_k$, we define $\mLam_{\epsilon,r} = \epsilon_r\mLam_r$ for each $r$ and $\mLam_\epsilon = [\mLam_{\epsilon,1} | \ldots |\mLam_{\epsilon,k}]$, and set $\epsilon = (\epsilon_1^2 + \ldots + \epsilon_k^2)^{1/2}$. Note that by submultiplicativity of the spectral norm, we have $\|\mLam_\epsilon\| \leq \epsilon \|\mLam\|$, and that we have the upper bound $\epsilon \leq k^{1/2}$. 

Uniform scaling is then overlaid onto the model by defining latent positions $\mX_i = \rho^{1/2}\xi_i$ and $\mY^{(r)}_j = \rho^{1/2}\upsilon^{(r)}_j$. We denote by $\mathcal{F}_\rho$ this scaled distribution.

Throughout this paper, we will impose certain restrictions on the sparsity factors which are necessary for our results to hold. If \textit{all} the graphs are too sparsely connected, then the UASE is unable to give us any useful information about the latent positions $\mX_i$ and $\mY^{(r)}_j$, and so we impose the global condition that:

\begin{itemize}
	\item The sparsity factor $\rho$ satisfies $\rho = \omega\left(\frac{\log^c(n)}{n^{1/2}}\right)$ for some constant $c$.
\end{itemize}

As mentioned previously, we require a sufficient number of graphs to be dense enough for us to be able to recover the latent positions $\mX_i$; if the only sufficiently dense graphs are degenerate, then it is entirely possible that we will be unable to do so. To avoid this, we impose the following local conditions (where we assume that we have reordered the matrices $\mA^{(r)}$ appropriately): 
\begin{itemize}
	\item There exists an integer $1 \leq \kappa \leq k$ such that $\epsilon_r = 1$ for all $r \leq \kappa$ and $\epsilon_r \to 0$ for all $r > \kappa$, and that the matrix $\mLam_* = [\mLam_1|\ldots|\mLam_\kappa]$ is of rank $d$.
\end{itemize}

This condition guarantees the existence of a sufficiently dense subset of graphs that will allow us to recover the latent positions $\mX_i$ (although not necessarily the positions $\mY^{(r)}_j$). If each $\mLam_r$ is of rank $d$, then having $\kappa = 1$ is sufficient for our purposes.

\subsection{Theoretical results}

The main aim of this paper is to accurately describe the asymptotic behaviour of both sides of the UASE, which we do by establishing two key results. The first of these shows that the rows of the left embedding approximate some invertible linear transformation (of bounded spectral norm) of the latent positions $\mX_i$, and that by inverting this transformation we obtain a good approximation to the latent positions themselves, in the sense that the maximum error vanishes. Similarly, we show that the rows of the $r$th right embedding approximate some (not necessarily invertible) linear transformation of the latent positions $\mY^{(r)}_i$, and that under appropriate conditions we can invert this transformation to obtain a good approximation to the latent positions themselves. This result is stated using the two-to-infinity norm \citep{CTP17} of the associated error matrix, which is the maximum Euclidean norm of any of its rows.

\begin{theorem}[Two-to-infinity norm bound for the UASE]\label{consistency}
\leavevmode\\
Let $(\mA,\mX, \mY) \sim \mathrm{MRDPG}(\mathcal{F}_\rho,\mLam_\epsilon)$ for a distribution $\mathcal{F}$ and sparsity factors $\rho$ and $\epsilon_r$ satisfying the criteria stated in Section \ref{subsec:asymptotics}. Then there exist sequences of matrices $\mL = \mL(n,n_1,\ldots,n_k) \in \mathrm{GL}(d)$ and $\mR_r = \mR_r(n,n_1,\ldots,n_k) \in \mathbb{R}^{d_r \times d}$ for each $r$ satisfying $\mL\mR_r^\top = \mLam_{\epsilon,r}$ such that
\begin{align}
	\|\mX_\mA - \mX\mL\|_{2 \to \infty} = \mathrm{O}\Bigl(\tfrac{\log^{1/2}(n)}{\rho^{1/2}n^{1/2}}\Bigr),~\|\mY^{(r)}_\mA - \mY^{(r)}\mR_r\|_{2 \to \infty} = \mathrm{O}\Bigl(\tfrac{\log^{1/2}(n)}{\rho^{1/2}n^{1/2}}\Bigr)
\end{align} 
almost surely. Moreover, we may invert the matrices $\mL$ and (if $\mLam_r$ is of maximal rank) $\mR_r$, and consequently find that 
\begin{align}
	\|\mX_\mA\mL^{-1} - \mX\|_{2 \to \infty} = \mathrm{O}\Bigl(\tfrac{\epsilon^{1/2}\log^{1/2}(n)}{\rho^{1/2}n^{1/2}}\Bigr),~ \|\mY^{(r)}_\mA\mR_r^+ - \mY^{(r)}\|_{2 \to \infty} = \mathrm{O}\Bigl(\tfrac{\log^{1/2}(n)}{\epsilon_r^{1/2}\rho^{1/2}n^{1/2}}\Bigr)
\end{align} 
almost surely, where $\mR_r^+ = \mR_r^\top(\mR_r\mR_r^\top)^{-1}$ is the Moore-Penrose inverse of $\mR_r$.
\end{theorem}

The matrices $\mL$ (and consequently $\mR_r$) can be described explicitly, and are constructed by a two-step process; first using a modified Procrustes-style argument to simultaneously align $\mX_\mA$ with $\mX_\mP$ and $\mY_\mA$ with $\mY_\mP$ via an \textit{orthogonal} transformation, and then applying a second \textit{linear} transformation which maps $\mX_\mP$ directly to $\mX$. For the first step, let $\mU_\mP^\top\mU_\mA + \mV_\mP^\top\mV_\mA$ admit the singular value decomposition \begin{align}
	\mU_\mP^\top\mU_\mA + \mV_\mP^\top\mV_\mA = \mW_1\mSig\mW_2^\top,
\end{align}
and let $\mW = \mW_1\mW_2^\top$. The matrix $\mW$ solves the one mode \textit{orthogonal} Procrustes problem 
\begin{align}
	\mW = \argmin_{\mQ \in \mathbb{O}(d)} \|\mU_\mA - \mU_\mP\mQ\|_F^2 + \|\mV_\mA - \mV_\mP\mQ\|_F^2,
\end{align} 
and we use $\mW^\top$ to align $\mX_\mA$ with $\mX_\mP$.

For the second step, we construct a matrix $\widetilde{\mL}$ which satisfies $\mX_\mP = \mX\widetilde{\mL}$ and whose inverse can be shown to have spectral norm of order $\mathrm{O}(\epsilon^{1/2})$ (see the Appendix for full details). The matrix $\mL$ is then given by $\mL = \widetilde{\mL}\mW$. A similar process is repeated for the right-sided embedding.

The second of our main results centres on the error distribution of the estimates for the latent positions $\mX_i$ and (if $\mLam_r$ is invertible) $\mY^{(r)}_j$ established in Theorem \ref{consistency}, and shows that conditional on the true position it is asymptotically Gaussian:

\begin{theorem}[Central limit theorem for the UASE]\label{CLT} 
\leavevmode\\
Let $(\mA,\mX, \mY) \sim \mathrm{MRDPG}(\mathcal{F}_\rho,\mLam_\epsilon)$ for a distribution $\mathcal{F}$ and sparsity factors $\rho$ and $\epsilon_r$ satisfying the criteria stated in Section \ref{subsec:asymptotics}, and let $\mL$ and $\mR_r$ be the transformation matrices specified in Theorem \ref{consistency}. 

Let $\xi \sim \mathcal{F}_X$ and $\upsilon_r \sim \mathcal{F}_{Y,r}$ for each $r \in \{1,\ldots,\kappa\}$, where $\mathcal{F}_X$ and $\mathcal{F}_{Y,r}$ are the marginal distributions of $\mathcal{F}$, and define $\Delta_Y = c_1\Delta_{Y,1} \oplus \ldots \oplus c_\kappa\Delta_{Y,\kappa}$, where the constants $c_r$ are as stated in Section \ref{subsec:asymptotics}. Given $\mx \in \mathcal{X}$, define 
\begin{align}
	\mSig_Y^{(r)}(\mx) = \left\{\begin{array}{cc}\mathbb{E}[\mx^\top\mLam_r\upsilon_r(1 - \mx^\top\mLam_r\upsilon_r)\cdot\upsilon_r\upsilon_r^\top] & \mathrm{if}~\rho=1\\\mathbb{E}[\mx^\top\mLam_r\upsilon_r\cdot\upsilon_r\upsilon_r^\top] & \mathrm{if}~\rho \to 0\end{array}\right.
\end{align}
for each $r$, and let $\mSig_Y = c_1\mSig_Y^{(1)} \oplus \cdots \oplus c_\kappa\mSig_Y^{(\kappa)}$. Then, for all $\mz \in \mathbb{R}^{d}$ and for any fixed $i$, 
\begin{align}
	\mathbb{P}\Bigl(n^{1/2}\bigl(\mX_\mA\mL^{-1} - \mX\bigr)_i^\top \leq \mz ~|~ \xi_i = \mx\Bigr) \to \Phi\bigl(\mz,\Delta_{\mLam,Y}^{-1}\mLam_*\mSig_Y(\mx)\mLam_*^\top\Delta_{\mLam,Y}^{-1}\bigr)
\end{align} 
almost surely, where $\Delta_{\mLam,Y} = \mLam_*\Delta_Y\mLam_*^\top$.

Moreover, for each $r \in \{1,\ldots,\kappa\}$, if the matrix $\mLam_r$ is invertible, then given $\my \in \mathcal{Y}_r$ define 
\begin{align}
	\mSig_X^{(r)}(\my) = \left\{\begin{array}{cc}\mathbb{E}\left[\xi^\top\mLam_r\my(1 - \xi^\top\mLam_r\my)\cdot\xi\xi^\top\right] & \mathrm{if}~\rho=1\\\mathbb{E}\left[\xi^\top\mLam_r\my \cdot \xi\xi^\top\right] & \mathrm{if}~\rho \to 0\end{array}\right.
\end{align} 
Then, for all $\mz \in \mathbb{R}^{d_r}$ and for any fixed $i$, 
\begin{align}
	\mathbb{P}\Bigl(n^{1/2}(\mY^{(r)}_\mA\mR_r^{-1} - \mY^{(r)})_i^\top \leq \mz ~|~ \upsilon^{(r)}_i = \my\Bigr) \to \Phi\bigl(\mz,\mLam_r^{-1}\Delta_X^{-1}\mSig_X^{(r)}(\my)\Delta_X^{-1}\mLam_r^{-\top}\bigr)
\end{align} 
almost surely.  
\end{theorem}

The theorems, of which single-graph analogs were derived in \citep{sussman2012consistent,athreya2016limit,LTAPP17,cape_biometrika,CTP17,TP19,RPTC18,AFTPPVLLQ17}, also have analogous methodological implications. Under a multilayer stochastic block model, discussed in Section~\ref{sec:GMSBM}, the left UASE asymptotically follows a Gaussian mixture model with non-circular components. Fitting this model is preferable to using $K$-means, which is implicitly fitting circular components. Apart from shape considerations, note that (as with the GRDPG) latent positions under the MRDPG are only identifiable up to a distance-distorting transformation (here invertible linear, there indefinite orthogonal) to which partitions obtained using a Gaussian mixture model are invariant but those obtained using $K$-means are not. Consistency in the two-to-infinity norm should imply the consistency of many subsequent statistical analyses: usually, if a method is consistent, it remains so under a perturbation of the data of vanishing maximal error; one need only then worry about the effect of an unidentifiable linear transformation on conclusions; if the estimand is invariant, this effect will often vanish on account of the transformation having bounded spectral norm.

In the special case in which the graphs $\mA^{(1)}, \ldots, \mA^{(k)}$ are identically distributed, we can \textit{always} recover the latent positions:

\begin{corollary}\label{identical_MRDPG_consistency}
Let $(\mA,\mX, \mY) \stackrel{\mathrm{id}}{\sim} \mathrm{MRDPG}(\mathcal{F}_\rho,\mLam)$ for a distribution $\mathcal{F}$ and sparsity factor $\rho$ satisfying the criteria stated in Section \ref{subsec:asymptotics}. Then there exist sequences of matrices $\mL = \mL(n,n') \in \mathrm{GL}(d)$ and $\mR = \mR(n,n') \in \mathbb{R}^{d' \times d}$ satisfying $\mL\mR^\top = \mLam$ such that 
\begin{align}
	\|\mX_\mA - \mX\mL\|_{2 \to \infty} = \mathrm{O}\Bigl(\tfrac{\log^{1/2}(n)}{\rho^{1/2}n^{1/2}}\Bigr),~\|\mY^{(r)}_\mA - \mY\mR\|_{2 \to \infty} = \mathrm{O}\Bigl(\tfrac{\log^{1/2}(n)}{\rho^{1/2}n^{1/2}}\Bigr)
\end{align} 
almost surely. Moreover, we may invert the matrices $\mL$ and (if $\mLam$ is of maximal rank) $\mR$, and consequently find that 
\begin{align}
\|\mX_\mA\mL^{-1} - \mX\|_{2 \to \infty} = \mathrm{O}\Bigl(\tfrac{\log^{1/2}(n)}{\rho^{1/2}n^{1/2}}\Bigr),~\|\mY^{(r)}_\mA\mR^+ - \mY\|_{2 \to \infty} = \mathrm{O}\Bigl(\tfrac{\log^{1/2}(n)}{\rho^{1/2}n^{1/2}}\Bigr)
\end{align} 
almost surely, where $\mR^+ = \mR^\top(\mR\mR^\top)^{-1}$ is the Moore-Penrose inverse of $\mR$.
\end{corollary}

We can similarly derive a more precise version of the Central Limit Theorem of Theorem \ref{CLT}:

\begin{corollary}\label{identical_MRDPG_CLT}
Let $(\mA,\mX, \mY) \stackrel{\mathrm{id}}{\sim} \mathrm{MRDPG}(\mathcal{F}_\rho,\mLam)$ for a distribution $\mathcal{F}$ and sparsity factor $\rho$ satisfying the criteria stated in Section \ref{subsec:asymptotics}, and let $\mL$ be the transformation matrix specified in Corollary \ref{identical_MRDPG_consistency}. 

Let $\xi \sim \mathcal{F}_X$ and $\upsilon \sim \mathcal{F}_{Y}$, where $\mathcal{F}_X$ and $\mathcal{F}_{Y}$ are the marginal distributions of $\mathcal{F}$. Given $\mx \in \mathcal{X}$, define 
\begin{align}
\mSig_Y(\mx) = \left\{\begin{array}{cc}\mathbb{E}[\mx^\top\mLam\upsilon(1 - \mx^\top\mLam\upsilon)\cdot\upsilon\upsilon^\top] & \mathrm{if}~\rho=1\\\mathbb{E}[\mx^\top\mLam\upsilon\cdot\upsilon\upsilon^\top] & \mathrm{if}~\rho \to 0\end{array}\right.
\end{align} 
Then, for all $\mz \in \mathbb{R}^{d}$ and for any fixed $i$, 
\begin{align}
	\mathbb{P}\Bigl(n^{1/2}\bigl(\mX_\mA\mL^{-1} - \mX\bigr)_i^\top \leq \mz ~|~ \xi_i = \mx\Bigr) \to \Phi\bigl(\mz,\tfrac{c}{k}\mLam^{-\top}\Delta_Y^{-1}\mSig_Y(\mx)\Delta_Y^{-1}\mLam^{-1}\bigr)
\end{align} 
almost surely, where $\lim_{n,n' \to \infty} \frac{n'}{n} = c$.

Moreover, if the matrix $\mLam_r$ is invertible, then given $\my \in \mathcal{Y}$ define 
\begin{align}
	\mSig_X(\my) = \left\{\begin{array}{cc}\mathbb{E}[\xi^\top\mLam\my(1 - \xi^\top\mLam\my)\cdot\xi\xi^\top] & \mathrm{if}~\rho=1\\\mathbb{E}[\xi^\top\mLam\my \cdot \xi\xi^\top] & \mathrm{if}~\rho \to 0\end{array}\right.
\end{align} 
Then, for all $\mz \in \mathbb{R}^{d'}$ and for any fixed $i$, 
\begin{align}
\mathbb{P}\Bigl(n^{1/2}\bigl(\mY^{(r)}_\mA\mR^{-1} - \mY\bigr)_i^\top \leq \mz ~|~ \upsilon_i = \my\Bigr) \to \Phi\bigl(\mz,\mLam^{-1}\Delta_X^{-1}\mSig_X(\my)\Delta_X^{-1}\mLam^{-\top}\bigr)
\end{align}
almost surely.  \end{corollary}

If $\mY = \mX$ with probability one and $\mLam = \mI_{p,q}$ (the diagonal matrix whose first $p$ entries are equal to $1$ and remaining $q$ entries are equal to $-1$) then the limiting distribution for the rows UASE are the same as that for the ASE stated in \cite{RPTC18}, scaled by a factor of $\tfrac{1}{k}$, and in particular coincides with that obtained by spectrally embedding the \textit{average} $\bar{\mA} = \tfrac{1}{k}\sum_{r=1}^k \mA^{(r)}$ of the adjacency matrices (see for example \cite{TKBCMVPS19}).

In Corollaries~\ref{identical_MRDPG_consistency} and~\ref{identical_MRDPG_CLT} the matrices $\mR$ are common across graphs, allowing a direct comparison of their right-sided embeddings. By constrast, independently embedded graphs first need to be aligned before a meaningful comparison is possible. In \citep{tang2017semiparametric} this is achieved using Procrustes, the appropriate method of alignment under an RDPG model where latent positions are identifiable only up to orthogonal transformation. Under the GRDPG, finding a best indefinite orthogonal alignment is less straightforward. Computational issues aside, alignment comes at a statistical cost \cite{LATLP17} when testing whether two point clouds differ statistically, and the empirical performance of using right-sided unfolded adjacency spectral embeddings for two-graph testing are investigated in Section~\ref{two-graph-testing}.

\section{The generalised multilayer stochastic block model}\label{sec:GMSBM}

\begin{definition}\label{def:GMSBM}(The generalised multilayer stochastic block model).
\\ We say that the matrices $\mA^{(1)}, \ldots, \mA^{(k)}$ are distributed as a $\mathbf{K}$-community generalised multilayer stochastic block model for a tuple $\mK = (K,K_1,\ldots,K_k)$ of positive integers if $(\mA,\mX,\mY) \sim \mathrm{MRDPG}(\mathcal{F},\mB)$ for a set of matrices $\mB^{(r)} \in [0,1]^{K \times K_r}$ and a distribution $\mathcal{F}$ whose marginal distributions $\mathcal{F}_X$ and $\mathcal{F}_{Y,r}$ are supported on the sets $\{\me_1,\ldots,\me_K\}$ and $\{\me_1,\ldots,\me_{K_r}\}$ of standard basis vectors of $\mathbb{R}^K$ and $\mathbb{R}^{K_r}$ respectively. If this is the case we write $(\mA,\mX,\mY) \sim \mathrm{GMSBM}(\mathcal{F},\mB)$. \end{definition}

If each of the $\mY^{(r)}$ is equal to $\mX$ with probability one (and consequently $K_r = K$ for each $r$), the distribution $\mathcal{F}$ assigns each latent position to the $i$th basis vector of $\mathbb{R}^K$ with probability $\pi_i$ for some tuple $\mpi = (\pi_1,\ldots,\pi_K)$ whose entries sum to $1$, and each of the matrices $\mA^{(r)}$ and $\mB^{(r)}$ is symmetric, then the GMSBM reduces to the standard $K$-community multilayer SBM \cite{HLL83}. In this case, the matrices $\mA^{(r)}$ are the adjacency matrices of a set of graphs generated independently from a common set of vertices, in which, independently conditional on a partition of these vertices into $K$ disjoint communities, an edge is generated between the $i$th and $j$th vertices in the $r$th graph with probability $\mB^{(r)}_{z_i,z_j}$, where $z_i \in \{1,\ldots,K\}$ denotes the community membership of the $i$th vertex.

Note that, as per Proposition \ref{latent_positions}, we may consider alternative choices of latent positions (and thus matrices $\mLam_r$) for the GMSBM. We often use the second choice posited in Proposition \ref{latent_positions}, in which we consider the singular value decomposition $\mB = \mU\mSig\mV^\top$ with $\mU \in \mathrm{O}(K)$ and $\mV \in \mathrm{O}((K_1 + \ldots + K_k) \times K)$, and let the positions $\mX_i$ be chosen from the rows of $\mU\mSig^{1/2}$ and $\mY^{(r)}_j$ be chosen from the rows of $\mV_r\mSig$ (where we split $\mV$ into $k$ distinct blocks $\mV_r \in \mathbb{R}^{K_r \times K}$). In this case, each matrix $\mLam_r$ is equal to the identity matrix. If $k = 1$, this choice closely resembles the model presented in \cite{RPTC18}, with the signs of the eigenvalues of the matrix $\mB$ being absorbed into the latent positions $\mY_j$ (if we impose the additional condition that $\mY$ is equal to $\mX$ with probability one, then $\mLam = \mI_{p,q}$, exactly as in \cite{RPTC18}).

\subsection{Undirected graphs}
As a first demonstration of the UASE for the GMSBM, we consider the case in which the latent positions $\mY^{(r)} = \mX$, and the matrices $\mB^{(r)}$ and $\mA^{(r)}$ are symmetric (that is, the standard multilayer SBM). We begin with two examples. For the first, we assume that the adjacency matrices $\mA^{(1)}$ and $\mA^{(2)}$ are \textit{identically} distributed according to a multilayer SBM with parameters 
\begin{align}
	\mB = \left(\begin{array}{ccc}0.42&&0.42\\0.42&&0.5\end{array}\right),~\mpi = (0.6,0.4)
\end{align} 
(the Laplacian spectral embedding of a single graph generated with these parameters was studied in \cite{TP19}). Figure \ref{Fig_1_Identical_GMSBM} plots the estimated latent positions for the ASE of the matrix $\mA^{(1)}$ (first row) and the UASE of the matrix $\mA = [\mA^{(1)}|\mA^{(2)}]$ (second row) for $n = 1000$, 2000 and 4000. Also displayed are the $95\%$ level curves of the empirical distributions (dashed curves) and the theoretical distributions specified by Theorem \ref{CLT} (solid curves). 

\begin{figure}[ht!]
\centering
\includegraphics[scale=0.68]{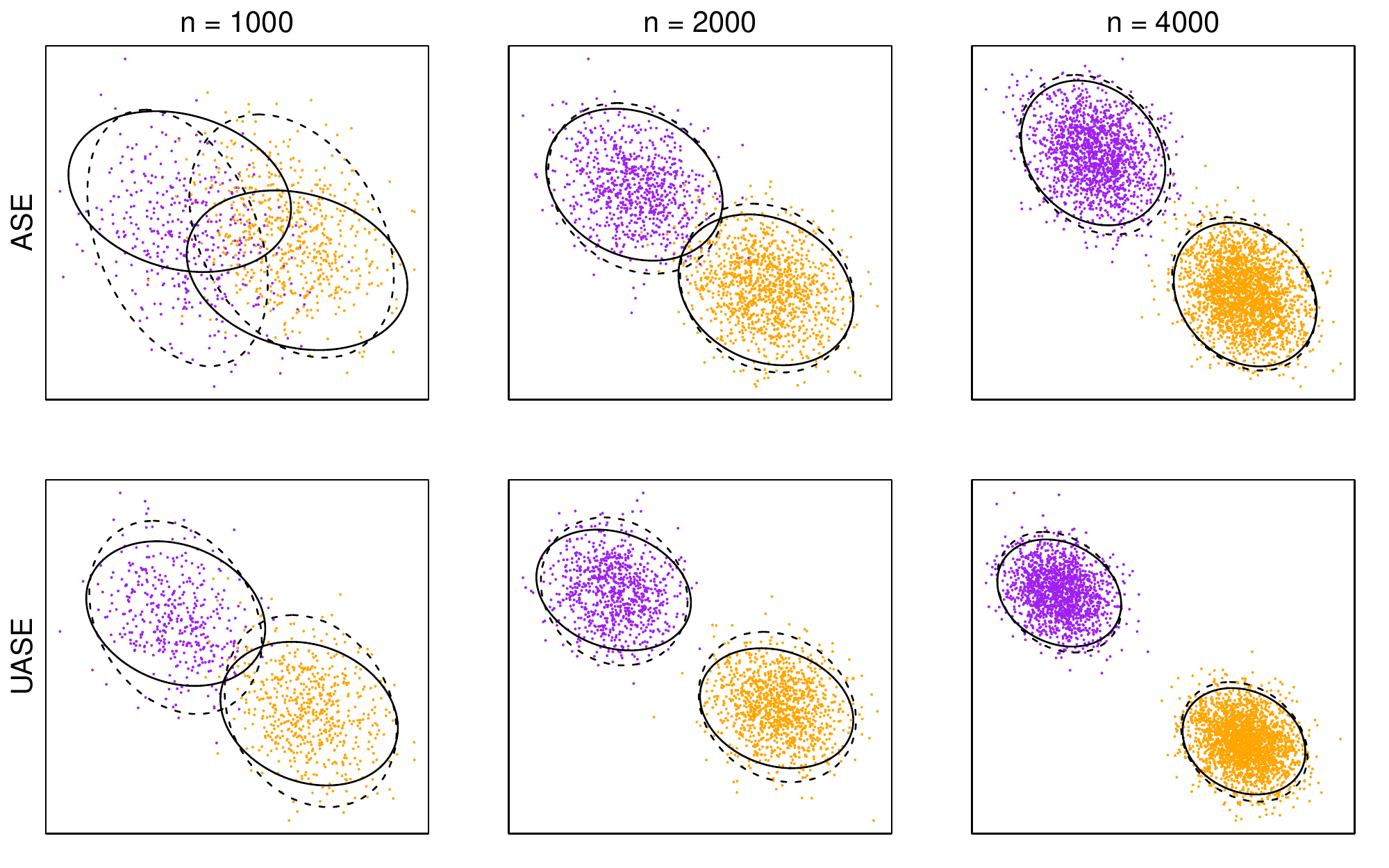}
\caption{Plots of the latent position estimates by the ASE and UASE of a pair of identically distributed graphs drawn from a $2$-community SBM on the same set of $n$ nodes. Points are coloured according to the community membership of the corresponding vertices. Ellipses give the $95\%$ level curves of the \textit{empirical} (dashed curves) and \textit{theoretical} (solid curves) distributions specified by Theorem \ref{CLT}.}
\label{Fig_1_Identical_GMSBM}
\end{figure}

For the second example, we take a pair of graphs with adjacency matrices $\mA^{(1)}$ and $\mA^{(2)}$ generated according to a multilayer SBM with parameters 
\begin{align}
	\mB^{(1)} = \left(\begin{array}{ccc}0.58&&0.58\\0.58&&0.5\end{array}\right),~\mB^{(2)} = \left(\begin{array}{ccc}0.42&&0.42\\0.42&&0.5\end{array}\right),~\mpi = (0.6,0.4).\end{align} 
Since the matrices $\mB^{(1)}$ and $\mB^{(2)}$ have signatures $(2,0)$ and $(1,1)$ respectively, they exhibit different assortativity behaviours. Figure \ref{Fig_2_Mixed_GMSBM} plots the estimated latent positions for the ASE of the matrix $\mA^{(1)}$ (first row) and the UASE (second row) for $n = 1000$, 2000 and 4000, with the $95\%$ level curves displayed as in Figure \ref{Fig_1_Identical_GMSBM}. 

\begin{figure}[ht!]
\centering
\includegraphics[scale=0.68]{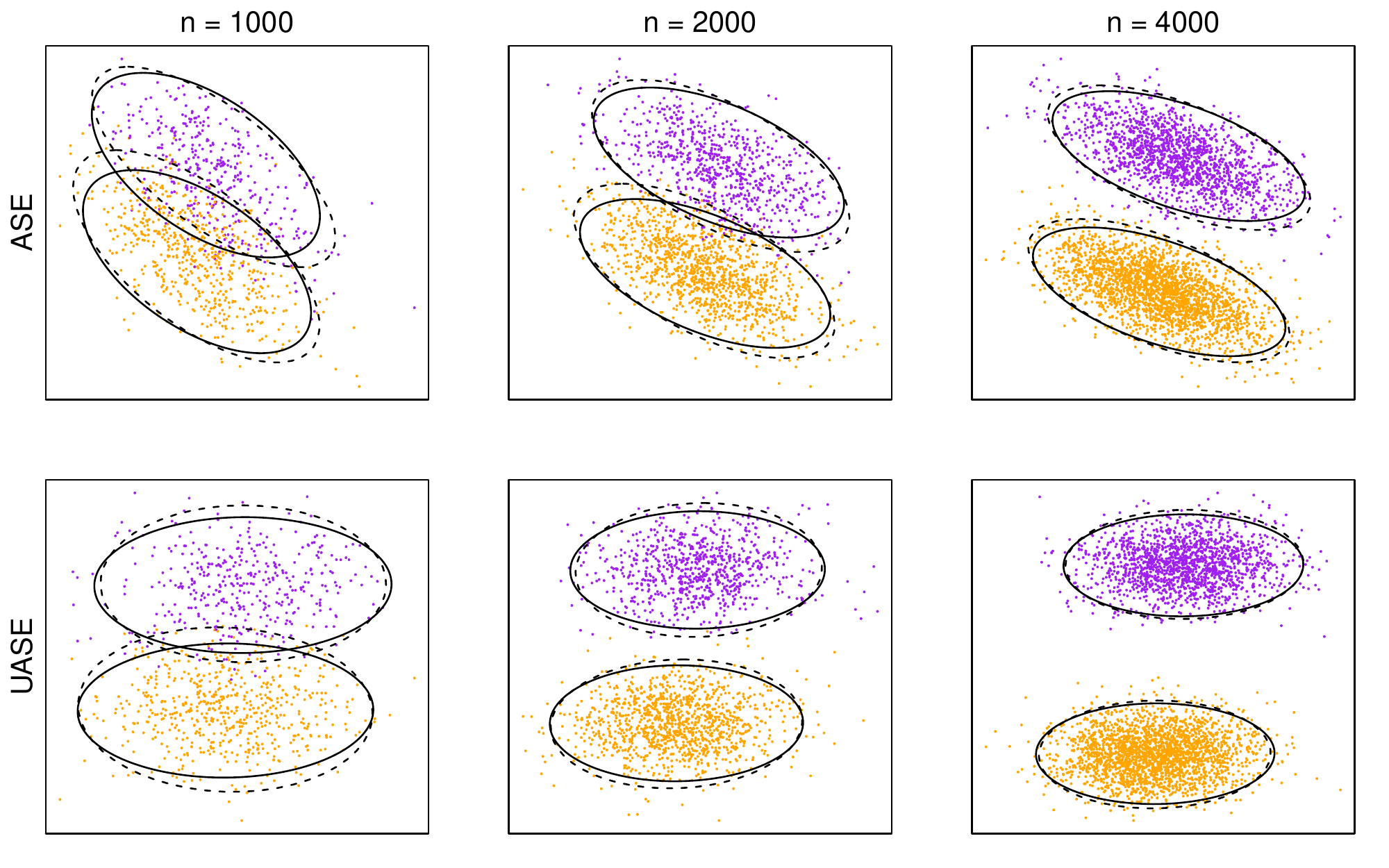}
\caption{Plots of the latent position estimates by the ASE and UASE of a pair of graphs drawn from a $2$-community SBM on the same set of $n$ nodes with differing distributions. Points are coloured according to the community membership of the corresponding vertices. Ellipses give the $95\%$ level curves of the \textit{empirical} (dashed curves) and \textit{theoretical} (solid curves) distributions specified by Theorem \ref{CLT}.}
\label{Fig_2_Mixed_GMSBM}
\end{figure}

In each example, the UASE demonstrates greater cluster separation over the ASE. To test empirically whether this behaviour holds in general, we performed the following experiment: for each value of $n \in \{50,100,250,500,750,1000,1500,2000\}$ we performed $500$ trials in which we generate two matrices $\mB^{(r)}$ with entries $\mB^{(r)}_{ij} \sim \mathrm{Uniform}[0,1]$, and probability vectors $\mpi \sim \mathrm{Dirichlet}(1,1)$ (to ensure that both clusters were of a reasonable size, we discard vectors $\mpi$ for which either of the $\mpi_i$ was less than $0.2$). Matrices $\mA^{(r)}$ were then generated according to the resulting multilayer SBM, the UASE calculated, and nodes were then assigned to the most likely cluster predicted by the Gaussian mixture model obtained via the MCLUST algorithm (see \cite{SFMR16}). These were then compared against the known cluster assignments given by the latent positions $\mX_i$, and the average classification error rate calculated across all samples of a given size $n$ (for the ASE, the average error rate of the two embeddings was used). For added differentiation, we performed this test for 3 separate cases: one in which the matrices $\mB^{(r)}$ were identical; one in which they had a common signature; and one in which they had different signatures. Figure \ref{Fig_3_GMSBM_community_detection} displays these error rates in the case of the identical and mixed parameter examples, as well as the average error rates for the \textit{mean embedding}---that is, the spectral embedding of the matrix $\bar{\mA} = \frac{1}{2}(\mA^{(1)} + \mA^{(2)})$, which we denote by ASE(mean)---and, in the identically distributed case, the \textit{omnibus embedding}---denoted by OMNI (see \cite{LATLP17}).

\begin{figure}[ht!]
\centering
\includegraphics[scale=0.68, trim = 18 0 12 0,clip]{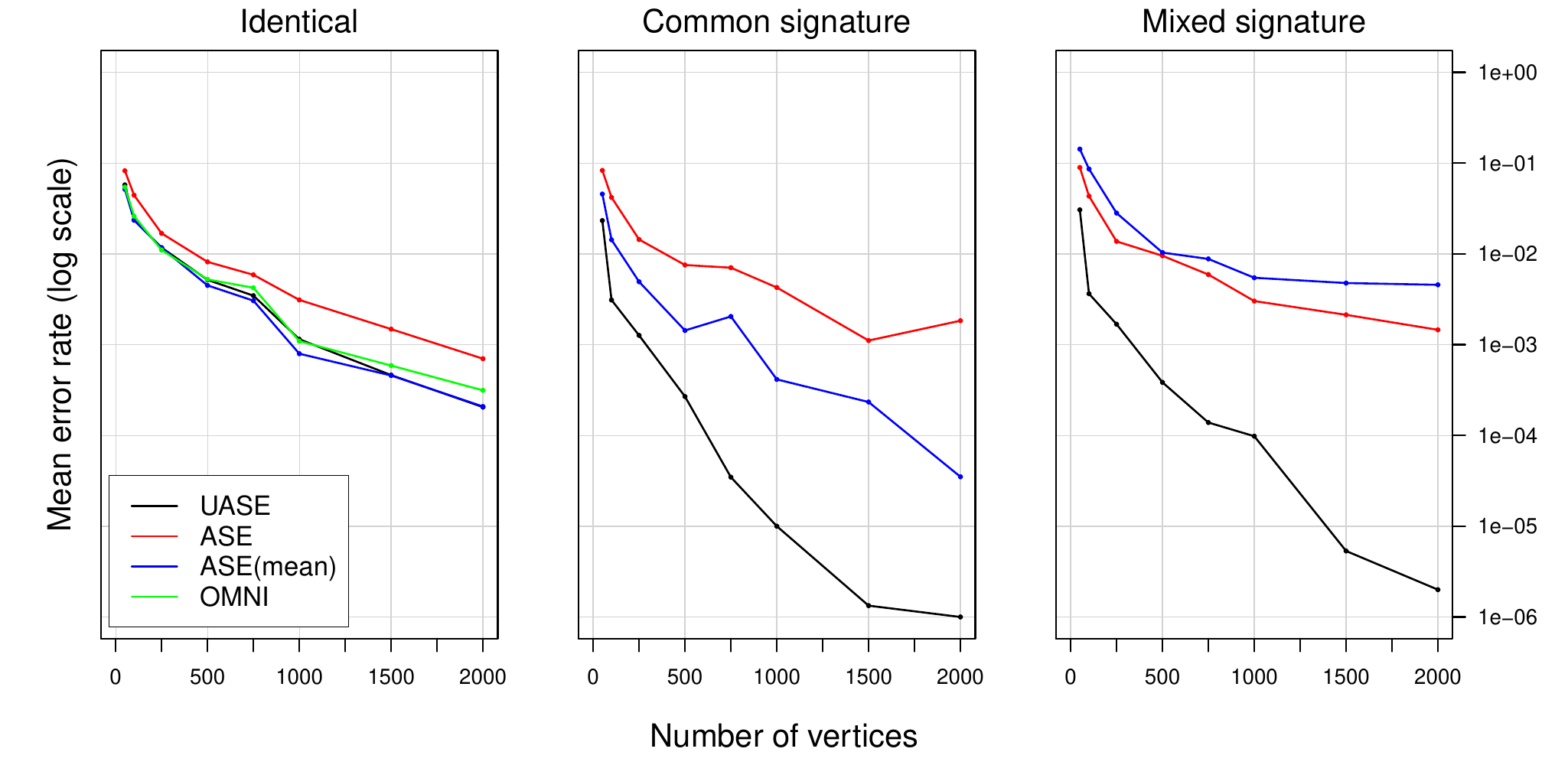}
\caption{Classification error rates for assignment of nodes to the most likely cluster predicted by the Gaussian mixture model obtained via the MCLUST	algorithm. Error rates are plotted on a logarithmic scale. See main text for details.}
\label{Fig_3_GMSBM_community_detection}
\end{figure}

As expected, the UASE (black line) clearly outperforms the ASE (red line) in all cases. When the two graphs are identically distributed, the UASE is comparable with the mean embedding (blue line) and slightly outperforms the omnibus embedding (green line). If the matrices $\mB^{(r)}$ differ then the UASE significantly outperforms the mean embedding; in particular, if the matrices $\mB^{(r)}$ have different signatures (resulting in different assortativity behaviours in the graphs $\mA^{(r)}$) then the mean embedding performs worse than even the ASE. This is not particularly surprising; if the adjacency matrices of different graphs have different signatures, then it is entirely possible for the matrix $\bar{\mP}$ to have non-maximal rank, causing some of the information in the system to be lost when we spectrally embed the matrix $\bar{\mA}$. Conversely, one finds that the embedding $\mX_\mA$ is the same as that of the positive-definite square root of the matrix $\sum_{r=1}^k (\mA^{(r)})^2$, which will \textit{always} be of maximal rank.

\subsection{Directed graphs}
 Unlike the standard multilayer SBM, the GMSBM does not require that the matrices $\mB^{(r)}$ or $\mA^{(r)}$ be symmetric, and so it allows us to consider \textit{directed} graphs. In particular, given a \textit{single} directed graph whose adjacency matrix $\mA \in \{0,1\}^{n \times n}$ follows a GMSBM, Theorems \ref{consistency} and \ref{CLT} show us that taking the standard ASE of $\mA$ will provide us with consistent estimates of the latent positions $\mX$. 

 However, the ability of the UASE to evaluate multiple adjacency matrices presents an alternative method for working with directed graphs. Let $(\mA,\mX,\mX) \sim \mathrm{GMSBM}(\mathcal{F},\mB)$ for some $\mathcal{F}$ and $\mB$, and let $\mathrm{Sym}(\mA)$ be the hollow symmetric matrix whose upper-triangular part is equal to that of $\mA$. Then we observe that $(\mathrm{Sym}(\mA),\mX,\mX) \sim \mathrm{GMSBM}(\mathcal{F},\mB)$, where we view $\mathrm{Sym}(\mA)$ as the adjacency matrix of an \textit{undirected} graph drawn on the same set of nodes as our original graph. Similarly, we see that $(\mathrm{Sym}(\mA^\top),\mX,\mX) \sim \mathrm{GMSBM}(\mathcal{F},\mB^\top)$, and thus $(\widetilde{\mA},\mX,\mX \oplus \mX) \sim \mathrm{GMSBM}(\widetilde{\mathcal{F}},\widetilde{\mB})$, where $\widetilde{\mA} = [\mathrm{Sym}(\mA)|\mathrm{Sym}(\mA^\top)]$, $\widetilde{\mB} = [\mB|\mB^\top]$ and $\widetilde{\mathcal{F}}$ is the natural extension of the distribution $\mathcal{F}$.

 We demonstrate the effectiveness of this proposed method through the following experiment: for each value of $n \in \{50,100,250,500,750,1000,1500,2000\}$, we performed $1000$ trials in which a directed graph was generated according to a $2$-community stochastic block model, where the (not necessarily symmetric) probability matrix $\mB$ and community probabilities $\mpi$ were randomly generated as before. The adjacency matrix $\mA$ was then constructed, and the ASE of $\mA$, the UASE of $\widetilde{\mA}$, and the mean embedding we calculated. Using each of these, we assign nodes to their most likely cluster using the MCLUST algorithm as before, and calculate the mean error rate across all the trials of a given sample size. The results (in the form of the average classification error rate) are plotted in Figure \ref{Fig_4_directed_GMSBM_community_detection}, and indicate that on average using the UASE offers significant improvement over the ASE, and a minor improvement over the mean embedding (we note however, that this does not guarantee that the UASE will \textit{always} offer the best performance). 

\begin{figure}[ht!]
\centering
\includegraphics[scale=0.72, trim = 0 20 30 50,clip]{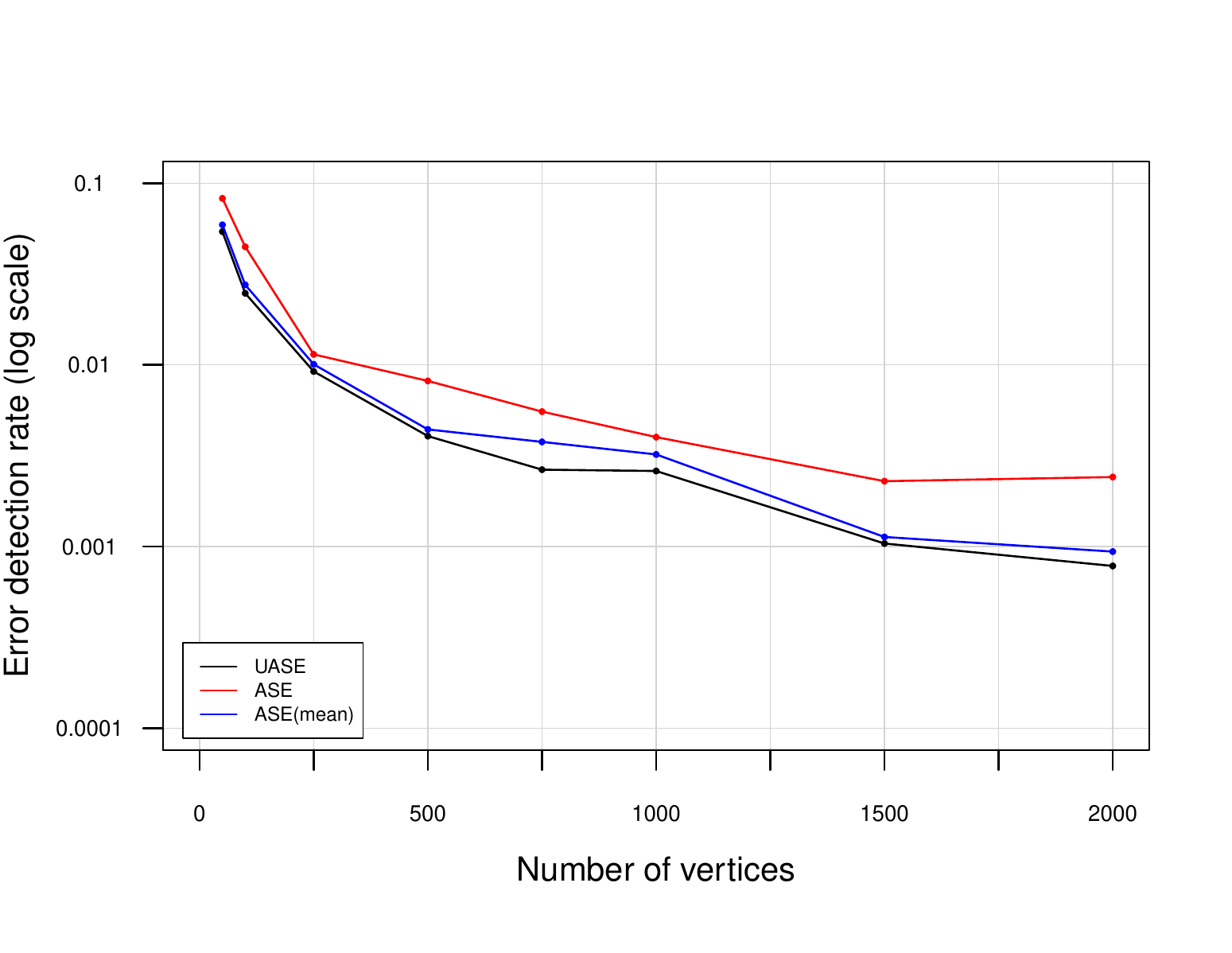}
\caption{Classification error rates for assignment of nodes from a directed graph to the most likely cluster predicted by the Gaussian mixture model obtained via the MCLUST	algorithm. Error rates are plotted on a logarithmic scale. See main text for details.}
\label{Fig_4_directed_GMSBM_community_detection}
\end{figure}w

\subsection{Bipartite multilayer SBMs}
 The flexibility of the GMSBM means that we do not have to restrict our attention to a standard multiple graph embedding; we may restrict our attention to submatrices of the adjacency matrices $\mA^{(r)}$ for each graph in order to focus only on the interactions involving a given set of nodes. As an illustrative example, consider the ``bipartite'' situation in which we have a GMSBM with two sets of latent positions $\mX_i \in \mathbb{R}^2$ and $\mY_j \in \mathbb{R}^3$, with probability matrices \begin{align}
	\mB^{(1)} = \left(\begin{array}{ccccc}0.45&&0.45&&0.52\\0.51&&0.51&&0.54\end{array}\right),~\mB^{(2)} = \left(\begin{array}{ccccc}0.46&&0.51&&0.43\\0.42&&0.47&&0.52\end{array}\right)
\end{align}
and community membership probabilities $\mpi_X = (1/2,1/2)$ and $\mpi_Y = (1/3,1/3,1/3)$. 

Figure \ref{Fig_5_bipartite_GMSBM} plots the left embedding $\mX_\mA$ and the right embeddings $\mY^{(r)}_\mA$ for the pairs $(n,n') = (1000,1500)$, $(2000,3000)$ and $(4000,6000)$, where we use the latent positions $\widetilde{\mX}, \widetilde{\mY}^{(r)} \in \mathbb{R}^2$ posited in Proposition \ref{latent_positions} This demonstrates an important point: while Theorems \ref{consistency} and \ref{CLT} guarantee that the embeddings $\mY^{(r)}_\mA$ provide consistent estimates of the latent positions $\widetilde{\mY}^{(r)}$, we cannot guarantee that they will distinguish between different communities, as the rank of $\widetilde{\mY}^{(r)}$ is equal to the rank of $\mB^{(r)}$, which may be less than the rank of $\mathrm{rank}(\mY)$. In our example, the embedding $\mY^{(1)}_\mA$ fails to distinguish between all three communities, while $\mY^{(2)}_\mA$ does distinguish between them. In general, if the columns of the matrix $\mB^{(r)}$ are distinct, then the latent positions corresponding to different communities will be distinct.

\begin{figure}[ht!]
\centering
\includegraphics[scale=0.64, keepaspectratio=TRUE, trim = 0 30 0 15,clip]{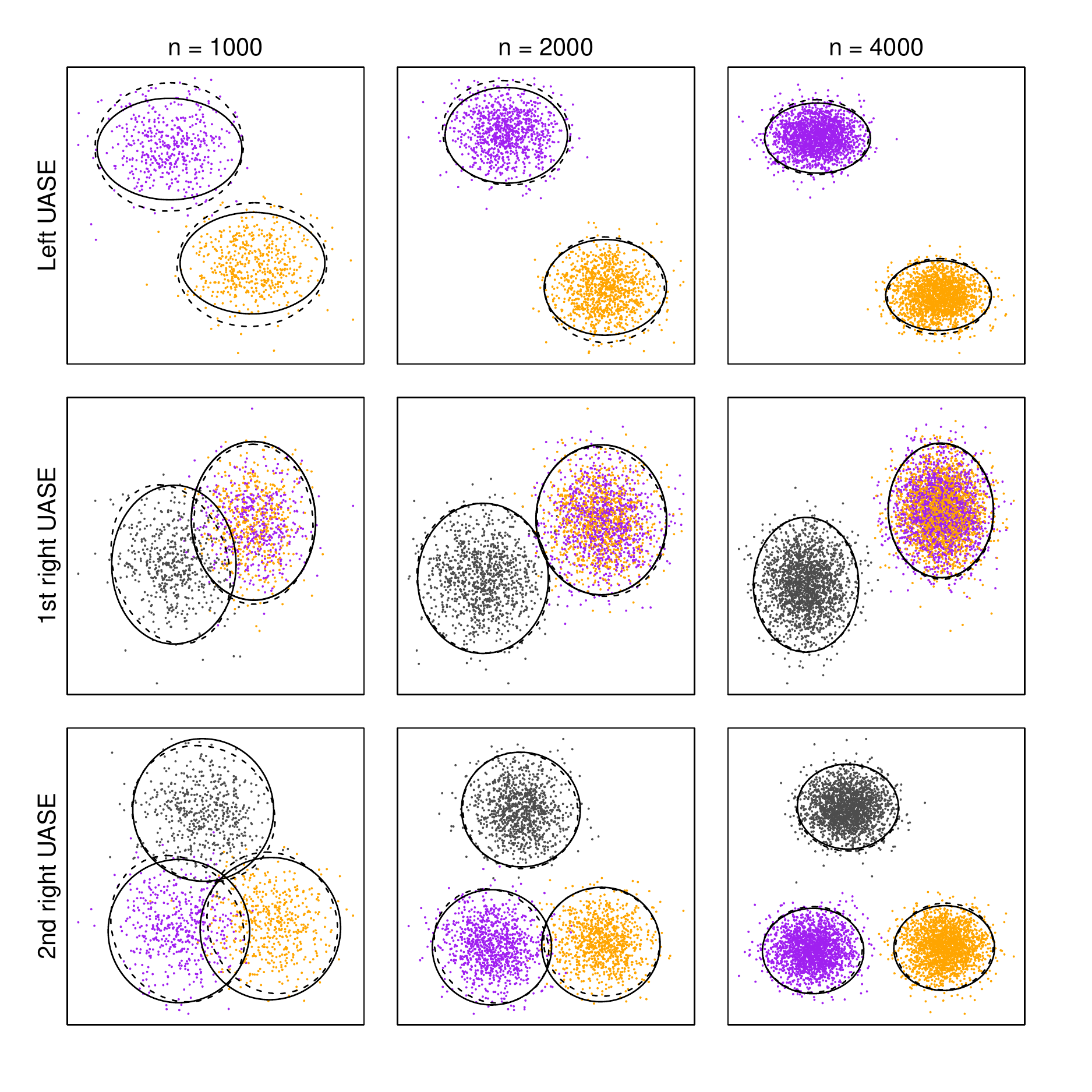}
\caption{\label{Fig_5_bipartite_GMSBM}Plots of the latent position estimates by the UASE of a pair of graphs drawn from a $(2,3)$-community GMSBM. Points are coloured according to the community membership of the corresponding vertices. Ellipses give the $95\%$ level curves of the \textit{empirical} (dashed curves) and \textit{theoretical} (solid curves) distributions specified by Theorem \ref{CLT}.}
\end{figure}

\subsection{Rank considerations} 
 One advantage of studying the MRDPG is that we do \textit{not} require the matrices $\mLam_r$ to have maximal rank; this can lead to situations in which information about latent positions is obscured in individual graphs, but becomes apparent when considering the joint embedding. As an example, consider a dynamic network which can be modeled as a two graph $3$-community multilayer SBM, in which the matrices $\mB^{(r)}$ of probabilities take the form 
\begin{align}
	\mB^{(1)} = \left(\begin{array}{ccccc}p_1&&p_1&&q_1\\p_1&&p_1&&q_1\\q_1&&q_1&&r_1\end{array}\right),~\mB^{(2)} = \left(\begin{array}{ccccc}p_2&&q_2&&q_2\\q_2&&r_2&&r_2\\q_2&&r_2&&r_2\end{array}\right)
\end{align} 
for values $p_i, q_i, r_i \in [0,1]$. We could view this as a simple model for time-dependent snapshots of the communication preferences between two departments in a company, in which a third team moves from the first department to the second in between snapshots, and inherits the communication preferences of the department to which they are assigned at the time. The matrices $\mB^{(r)}$ are both of non-maximal rank, but $\mB = [\mB_1|\mB_2]$ has maximal rank, provided that the $p_i, q_i$ and $r_i$ are distinct.

We demonstrate this with an example. Let $n = 4000$ and suppose that we have three communities, containing $1750$, $500$ and $1750$ nodes respectively.  Let the probability matrices be given by 
\begin{align}
	\mB^{(1)} = \left(\begin{array}{ccccc}0.47&&0.47&&0.39\\0.47&&0.47&&0.39\\0.39&&0.39&&0.56\end{array}\right),~\mB^{(2)} = \left(\begin{array}{ccccc}0.53&&0.61&&0.61\\0.61&&0.44&&0.44\\0.61&&0.44&&0.44\end{array}\right).
\end{align}

Figure \ref{Fig_6_non_maximal_rank_GMSBM} shows the embedded point clouds generated by the individual ASEs and the UASE in this situation. As one would expect, the individual ASEs display only the two communities that one observes at that given snapshot in time, while the UASE clearly displays all three communities..

\begin{figure}[ht]
\centering
\centering
\includegraphics[scale=0.7, trim = 15 0 10 0,clip]{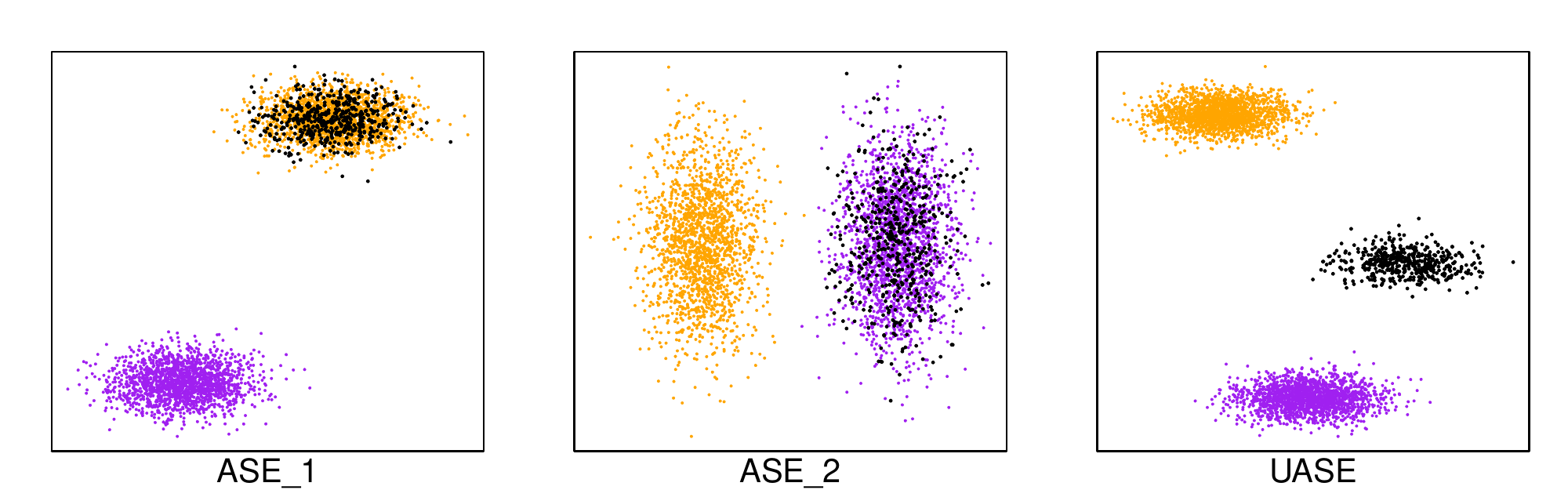}
\caption{\label{Fig_6_non_maximal_rank_GMSBM}Plots of the ASEs of the adjacency matrices $\mA_1$ and $\mA_2$ and the UASE, with nodes coloured according to community membership. Points are coloured according to the community to which their corresponding nodes belong, with points in black representing those nodes that switch between communities.}
\end{figure}

\section{Multiple graph inference: comparison with existing methods}\label{sec:experimental}

The performance of unfolded adjacency spectral embedding is now compared with alternative spectral approaches on the inference tasks of latent position recovery, subspace estimation, model estimation, and two-graph hypothesis testing.  We restrict our attention to the case in which the matrices $\mA^{(r)}$ are true adjacency matrices of graphs, and thus work under the assumption that the latent position matrices $\mY^{(r)}$ are equal to $\mX$ with probability one for all $r$.

\subsection{Recovery of latent positions}

An important estimation problem for the data of a random dot product graph is that of estimating the latent positions $\mX_i$, and so we shall investigate the performance of the MRDPG in the context of such an estimation problem. For comparison, we will consider the \textit{multiple adjacency spectral embedding} (MASE) \cite{AACCPV19}, which is an alternative method of jointly embedding adjacency matrices which follow a model that is essentially identical to the MRDPG, known as the \textit{common subspace independent edge graph model}. In \cite{AACCPV19}, the authors demonstrate that the MASE yields state-of-the-art performance on subsequent inference tasks, ahead of other competing models for studying multiple graph embeddings such as the multi-RDPG \cite{NW18} and MREG \cite{WAVP19} models, making it an ideal method to compare the UASE against.

\begin{definition} (Common Subspace Independent Edge graphs). 
\\Let $\mU = [\mU_1|\cdots|\mU_n]^\top \in \mathbb{R}^{n \times d}$ have orthonormal columns, and let $\mR^{(1)}, \ldots, \mR^{(k)} \in \mathbb{R}^{d \times d}$ be symmetric matrices such that, for each $r \in \{1,\ldots,k\}$, $\mU_i^\top \mR^{(r)}\mU_j \in [0,1]$ for all $i, j \in \{1,\ldots,n\}$.  The random adjacency matrices $\mA^{(1)}, \ldots, \mA^{(k)}$ are said to be jointly distributed according to the common subspace independent-edge graph model with bounded rank $d$ and parameters $\mU$ and $\mR^{(1)}, \ldots, \mR^{(k)}$ if for each $r \in \{1, \ldots, k\}$, conditional upon $\mU$ and $\mR^{(r)}$ we have $\mA^{(r)}_{ij} \sim \mathrm{Bern}\bigl(\mP^{(r)}_{ij}\bigr)$, where $\mP^{(r)} = \mU \mR^{(r)} \mU^\top$, in which case we write $(\mA^{(1)},\ldots,\mA^{(k)}) \sim \mathrm{COSIE}(\mU; \mR^{(1)}, \ldots, \mR^{(k)})$.\end{definition}

For all intents and purposes, the COSIE and MRDPG models are equivalent. Any COSIE model gives rise to a MRDPG by simply setting the latent positions $\mX_i$ to be equal to the rows $\mU_i$, and the matrices $\mLam_r = \mR^{(r)}$ for each $r$. Conversely, given a MRDPG such that the matrix $\mX$ of latent positions is of rank $d$, we can define $\mU = \mX(\mX^\top\mX)^{-1/2}$, where we have taken the positive-definite matrix square root of the matrix $\mX^\top\mX$. It is clear that the columns of $\mU$ are orthonormal, and we obtain a COSIE model by setting $\mR^{(r)} = (\mX^\top\mX)^{1/2}\mLam_r(\mX^\top\mX)^{1/2}$. In both cases the two definitions of the matrices $\mP^{(r)}$ coincide.

\begin{definition}(Multiple adjacency spectral embedding). 
\\Let $(\mA^{(1)},\ldots,\mA^{(m)}) \sim \mathrm{COSIE}(\mU; \mR^{(1)}, \ldots, \mR^{(k)})$. For each $r \in \{1,\ldots,k\}$ let $d_r$ denote the rank of $\mR^{(r)}$, let $\mX_{\mA^{(r)}} \in \mathbb{R}^{n \times d_r}$ be the adjacency spectral embedding of $\mA^{(r)}$ and define the matrix of concatenated spectral embeddings $\mM_\mA = [\mX_{\mA^{(1)}} | \cdots | \mX_{\mA^{(k)}}]$. The multiple adjacency spectral embedding of $\mA^{(1)}, \ldots, \mA^{(k)}$ is the matrix $\widehat{\mU}_\mA \in \mathbb{R}^{n \times d}$ containing the $d$ leading left singular vectors of $\mM_\mA$. \end{definition}

 We note that it is possible to use the MASE to produce estimates for the latent positions in a MRDPG in an analogous way to the method used for the UASE. Let $\widehat{\mX}_\mA = \widehat{\mU}_\mA \mSig_\mA$, where $\mSig_\mA \in \mathbb{R}^{d \times d}$ is the diagonal matrix of the leading $d$ singular vectors of $\mM_\mA$, and define $\mM_\mP$, $\widehat{\mU}_\mP$ and $\widehat{\mX}_\mP$ analogously for the matrices $\mP^{(r)}$. Firstly, we note that adding columns of zeros to any of the $\mX_{\mP^{(r)}}$ will not alter $\widehat{\mX}_\mP$, and so we may assume without loss of generality that the matrices $\mX_{\mP^{(r)}} \in \mathbb{R}^{n \times d}$, and thus that there exist matrices $\mL^{(r)} \in \mathbb{R}^{d \times d}$ of rank $d_r$ such that $\mX_{\mP^{(r)}} = \mX\mL^{(r)}$. One can prove the existence of a matrix $\mL \in \mathrm{GL}(d)$ such that $\widehat{\mX}_\mP = \mX\mL$, and so performing a Procrustes-style alignment between $\widehat{\mX}_\mA$ and $\widehat{\mX}_\mP$ and multiplying by $\mL^{-1}$ produces a set of points that in practice are a good approximation to the latent positions $\mX_i$. 

We first tested the performance of the two embeddings on graphs of different sizes by performing, for each value of $n \in \{10,25,50,75,100,250,500,750,1000\}$, $1000$ independent trials in which the latent positions $\mX_i$ are drawn i.i.d. from a Dirichlet distribution with parameter $(1,1,1)^\top \in \mathbb{R}^3$. In each trial, we generate two graphs $\mA^{(1)}, \mA^{(2)} \in \mathbb{R}^{n \times n}$, where $\mA^{(r)}_{ij} \sim \mathrm{Bern}(\mX_i^\top\mLam_r\mX_j)$ for $i < j$, where each $\mLam_r$ is a randomly chosen matrix. We then calculate estimates $\widehat{\mX}$ using the UASE and MASE as described previously, and compare the average mean squared error $\frac{1}{n}\sum_{i=1}^n|\widehat{\mX}_i-\mX_i|^2$ of the two embeddings across the $1000$ trials. 

 We then investigated the effect of changing the number of graphs to be embedded on the accuracy of each embedding type. Fixing $n = 750$, we again performed $1000$ independent trials as above for $m = 2, \ldots, 10$ embeddings, using the same procedure for generating the latent positions and adjacency matrices, and again compared the average mean squared error between the estimated and actual latent positions.

 Figure \ref{Fig_7_latent_position_recovery} plots the results of the two experiments. While the MASE outperforms the UASE for values of $n < 75$ (with the joint embedding performing significantly worse for $n < 50$) the UASE clearly demonstrates superior accuracy as the size of the graph grows, a trend which continues as we increase the number of graphs to be embedded.

\begin{figure}[ht]
\centering
\includegraphics[scale=0.68, trim = 10 0 30 10,clip]{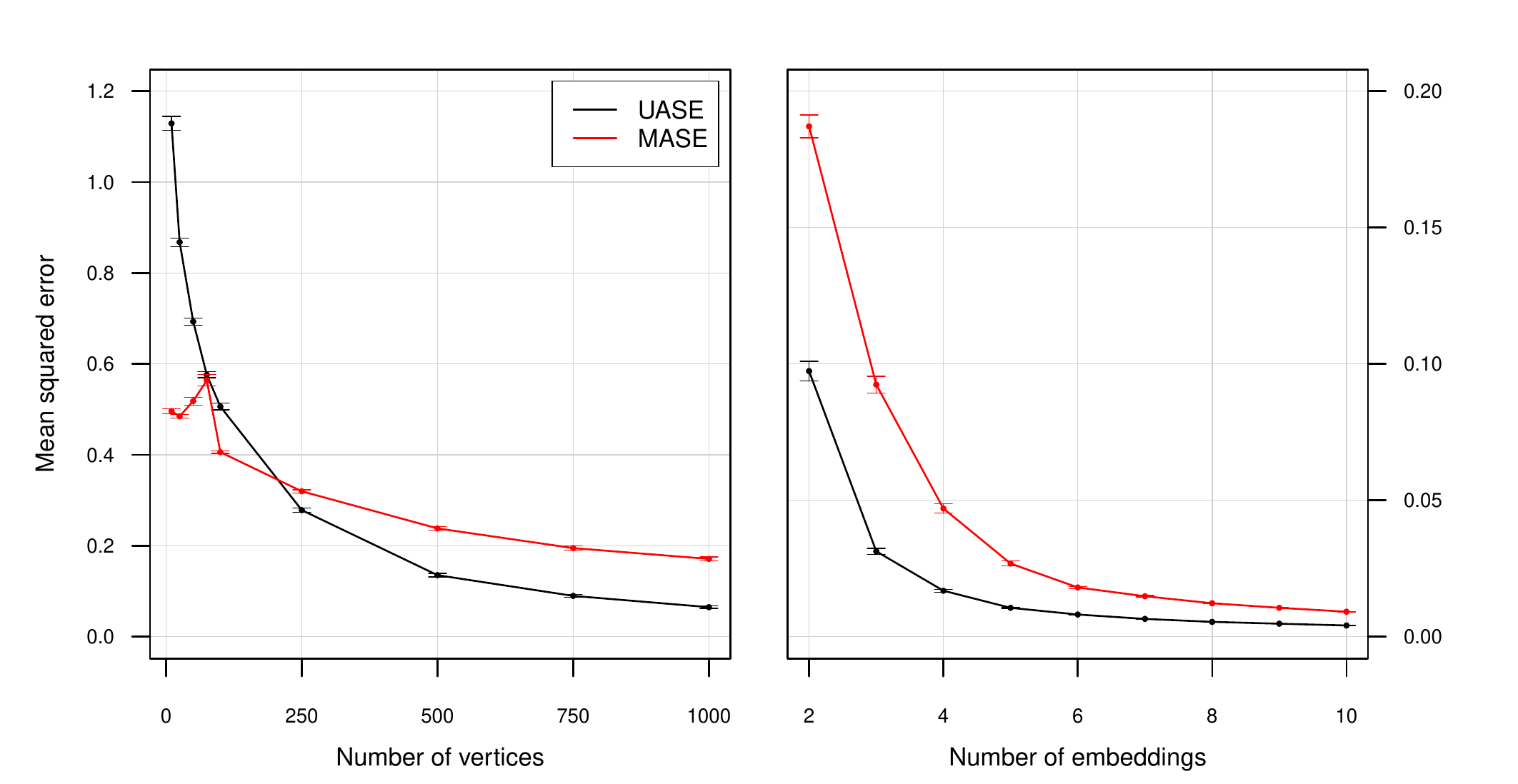}
\caption{\label{Fig_7_latent_position_recovery} Mean squared error in recovery of latent positions in a 2-graph MRDPG model as a function of the number of vertices (left-hand graph) and the number of embeddings (right-hand graph).}
\end{figure}

\subsection{Estimation of invariant subspaces}

 We next investigate the performance of UASE at estimating the invariant subspace $\mU$ in the COSIE model. We do this by setting the matrix $\mX$ of latent positions in the MRDPG to be equal to $\mU$, and considering the \textit{unscaled} UASE, $\mU_\mA$. Unlike the scaled embedding, which approximates the latent positions only up to linear transformation, the unscaled UASE approximates the invariant subspace $\mU$ up to \textit{orthogonal} transformation. Indeed, from our results for the scaled embedding the matrix $\mU_\mP = \mU \mQ_\mX$ for some $\mQ_\mX \in \mathrm{GL}(d)$, whence the requirement that both $\mU$ and $\mU_\mP$ have orthonormal columns forces $\mQ_\mX$ to in fact belong to $\mathbb{O}(d)$, while the transformation applied in the Procrustes alignment between $\mU_\mA$ and $\mU_\mP$ is by definition orthogonal.

We can measure the distance between the estimate $\mU_\mA$ and the true invariant subspace $\mU$ using the spectral norm of the difference between the projections $\|\mU_\mA\mU_\mA^\top - \mU\mU^\top\|$ (and similarly for the estimate $\widehat{\mU}$ produced by the MASE). This distance is zero only when there exists an orthogonal matrix $\mW \in \mathbb{O}(d)$ such that $\mU_\mA = \mU\mW$ (respectively $\widehat{\mU} = \mU\mW$). 

 As in the previous example, we investigated the effect of changing both the size of the graphs and the number of graphs to be embedded on the performance of the UASE and MASE. Again, we began by performing $1000$ independent trials for each value of $n \in \{10,25,50,75,100,250,500,750,1000\}$, but this time the adjacency matrices $\mA^{(1)}$ and $\mA^{(2)}$ were distributed according to a $3-$community multilayer stochastic block model, where the matrices $\mB^{(r)}$ were randomly chosen, and vertices assigned to a community uniformly at random, discarding any trials for which the matrix $\mX$ of community assignments was not of full rank. We then calculated and compared the average of the subspace distances $\|\mU_\mA\mU_\mA^\top - \mU\mU^\top\|$ and $\|\widehat{\mU}\widehat{\mU}^\top - \mU\mU^\top\|$ across each of the $1000$ trials. For the second experiment, we again fixed $n = 750$, performed $1000$ independent trials as above for $k = 2, \ldots, 10$ embeddings, and compared the subspace distance between the estimated and actual invariant subspaces.

\begin{figure}[ht]
\centering
\includegraphics[scale=0.68, trim = 10 0 30 10,clip]{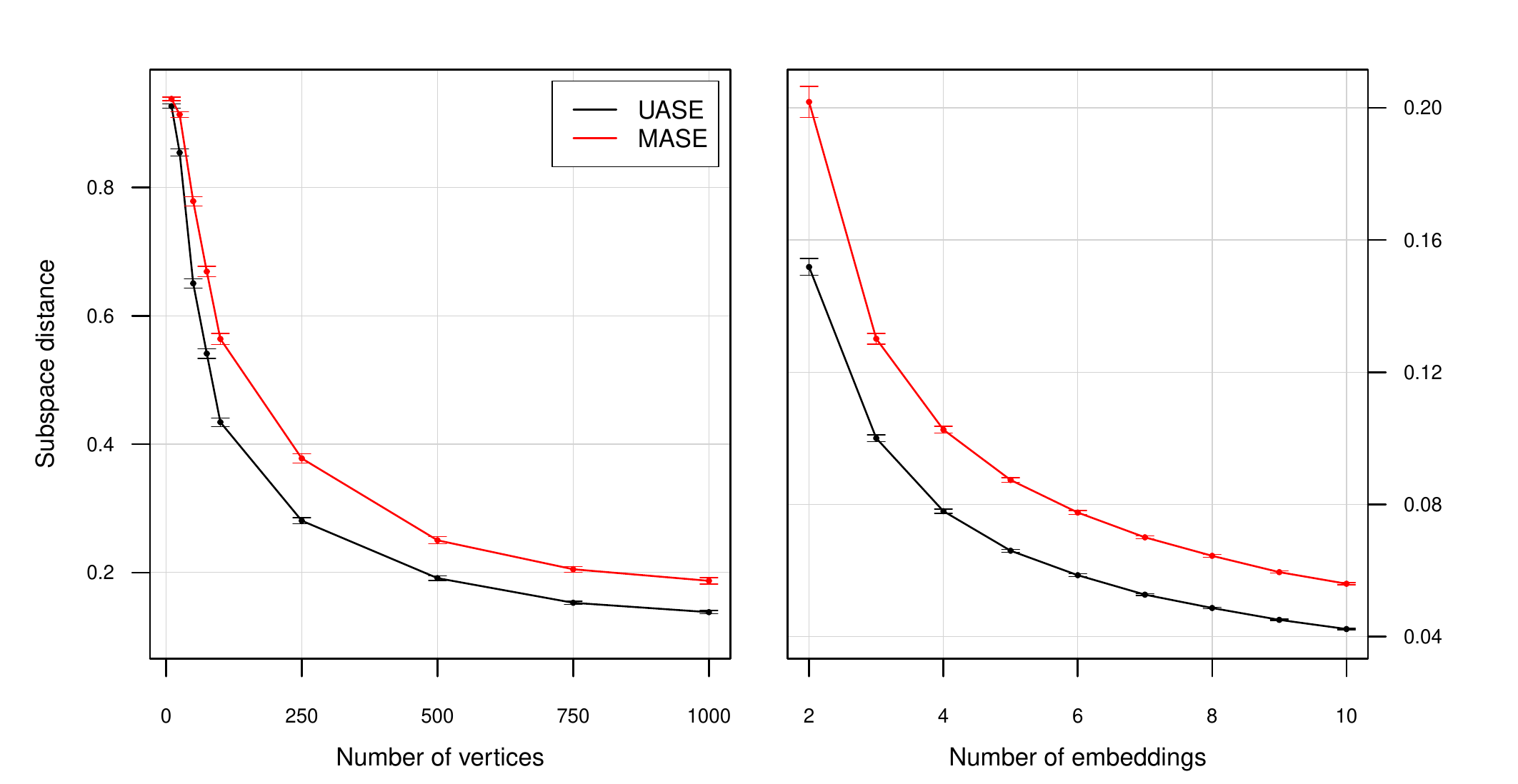}
\caption{\label{Fig_8_subspace_estimation} Average distance between the estimated and actual invariant subspaces in a $3$-community multilayer stochastic block model as a function of the number of vertices (left-hand graph) and the number of embeddings (right-hand graph).}
\end{figure}

Figure \ref{Fig_8_subspace_estimation} plots the results of the two experiments. For this task, although the performance of the two embedding types is almost indistinguishable for very small graphs, as the number of vertices grows the UASE consistently outstrips the MASE. As in the previous example, increasing the number of embedded graphs results in greater accuracy for both methods, where again the UASE offers the best performance of the two.

\subsection{Model estimation}\label{model_estimation_section}

As a final comparison of the UASE and MASE methods, we investigate the efficiency of both at the task of estimating the underlying matrices $\mP^{(r)}$ in the MRDPG and COSIE models, which is of particular practical interest for link prediction tasks. To establish an appropriate estimate, we first consider the case of the standard GRDPG (that is, when $k = 1$). In this case, an estimate $\widehat{\mP}$ for the matrix $\mP$ can be obtained by setting $\widehat{\mP} = \mX_\mA \mI_{p,q} \mX_\mA^\top$. Note that due to orthogonality of the singular vectors, the matrix $\mX_\mA \in \mathbb{R}^{n \times d}$ of the leading $d$ left singular vectors of $\mA$ is the projection of the \textit{full} matrix of left singular vectors onto the $d$-dimensional subspace spanned by $\mU_\mA$. Since this projection corresponds to left multiplication by the matrix $\mU_\mA\mU_\mA^\top$, we have the alternative description $\widehat{\mP} = \mU_\mA\mU_\mA^\top\mA\mU_\mA\mU_\mA^\top$. 

Returning to the general case, we obtain an estimate $\widehat{\mP}^{(r)} = \mU_\mA\mU_\mA^\top \mA^{(r)} \mU_\mA\mU_\mA^\top$ for the matrix $\mP^{(r)}$ for each $r \in \{1,\ldots,k\}$ using the unscaled UASE. For the MASE, we use the matrix $\widehat{\mU}\widehat{\mU}^\top \mA^{(r)}\widehat{\mU}\widehat{\mU}^\top$ as our estimate. For each of the trials in the previous example, we calculated these estimates, and measured the model estimation error in each case using the normalised mean squared error 
\begin{align}
	\frac{\|\widehat{\mP}^{(r)} - \mP^{(r)}\|_F}{\|\mP^{(r)}\|_F}.
\end{align}

Figure \ref{Fig_9_model_estimation} plots the results of the two experiments, in which we see that once again the UASE consistently demonstrates greater accuracy than the MASE for all but the smallest of graphs, and for all numbers of embedded graphs.

\begin{figure}[ht]
\centering
\includegraphics[scale=0.68, trim = 10 0 30 10,clip]{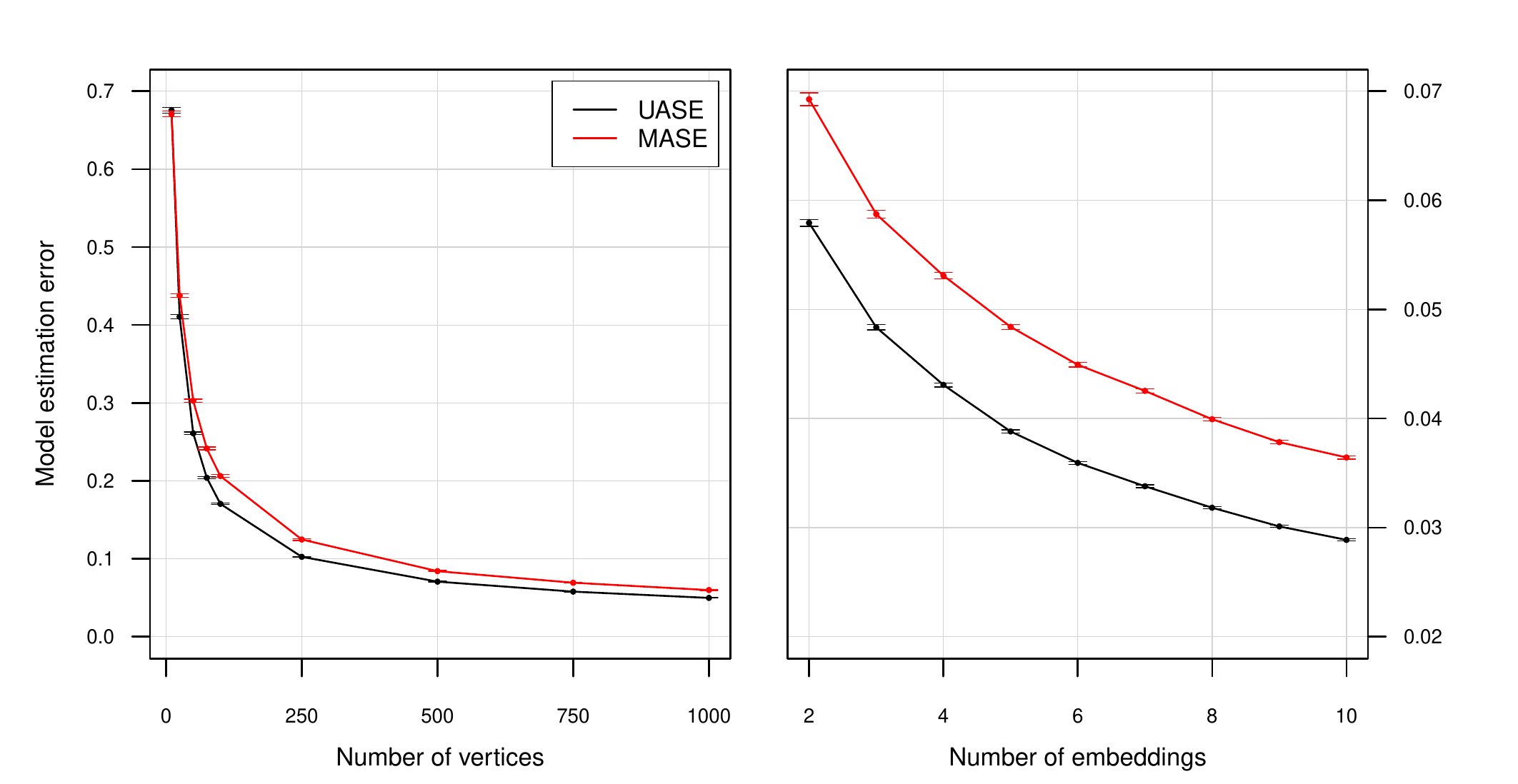}
\caption{\label{Fig_9_model_estimation} Model estimation error for the UASE and MASE in a $3$-community multilayer stochastic block model as a function of the number of vertices (left-hand graph) and the number of embeddings (right-hand graph).}
\end{figure}

\subsection{Two-graph hypothesis testing}\label{two-graph-testing}
When $\mA^{(1)}, \ldots, \mA^{(k)}$ are identically distributed, the right embeddings $\mY^{(r)}_\mA$ are identically distributed too and each subject to the \textit{same} unidentifiable linear transformation (Corollaries~\ref{identical_MRDPG_consistency} and~\ref{identical_MRDPG_CLT}). It is therefore natural to consider the effectiveness of the UASE at testing the semiparametric hypothesis that two observed graphs are drawn from the same underlying latent positions. This problem was considered for the omnibus embedding in \cite{LATLP17}, and we shall use the framework established there to test the UASE. Suppose, then, that we have points $\mX_1, \ldots, \mX_n, \mY_1, \ldots, \mY_n \in \mathbb{R}^d$, and that we have two graphs $G_1$ and $G_2$ whose adjacency matrices $\mA^{(1)}$ and $\mA^{(2)}$ satisfy $\mA^{(1)}_{ij} \sim \mathrm{Bern}(\mX_i^\top\mI_{p,q}\mX_j)$ and $\mA^{(2)}_{ij} \sim \mathrm{Bern}(\mY_i^\top\mI_{p,q}\mY_j)$. The UASE allows us to test the hypothesis: \begin{align}
	H_0: \mX_i = \mY_i \quad\forall i \in \{1,\ldots,n\}
\end{align}
by comparing the right embeddings $\mY^{(1)}_\mA$ and $\mY^{(2)}_\mA$. If $H_0$ holds, then the rows of the $\mY^{(r)}_\mA$ are identically (although not independently) distributed, whereas if $H_0$ fails to hold then for some $k$ the $k$th row of $\mY^{(1)}_\mA$ and $\mY^{(2)}_\mA$ should be distributionally distinct.

The framework used in \cite{LATLP17} to test this hypothesis, which we shall repeat here, is as follows: we begin by drawing $\mX_1,\ldots,\mX_n, \mZ_1, \ldots, \mZ_n \in \mathbb{R}^3$ identically according to a Dirichlet distribution with parameter $\alpha = (1,1,1)^\top$, select a subset $I$ of some fixed size uniformly at random among all such subsets of $\{1,\ldots,n\}$, and define 
\begin{align}
\mY_i = \left\{\begin{array}{cccc}\mZ_i&&&\mathrm{if}~ i \in I\\\mX_i&&&\mathrm{otherwise}\end{array}\right.
\end{align}

We generate two graphs $\mathcal{G}_1$ and $\mathcal{G}_2$ with adjacency matrices $\mA^{(1)}$ and $\mA^{(2)}$ satisfying $\mA^{(1)}_{ij} \sim \mathrm{Bern}(\mX_i^\top\mX_j)$ and $\mA^{(2)}_{ij} \sim \mathrm{Bern}(\mY_i^\top\mY_j)$, and estimate the latent positions $\widehat{\mX}$ and $\widehat{\mY}$ of the two graphs by using the right embeddings as described above, and (in the case of the omnibus embedding) the first and last $n$ rows of the spectral embedding of the matrix 
\begin{align}
	\mM = \left(\begin{array}{cccc}\mA^{(1)}&&&\tfrac{1}{2}(\mA^{(1)} + \mA^{(2)})\\\tfrac{1}{2}(\mA^{(1)} + \mA^{(2)})&&&\mA^{(2)}\end{array}\right).
\end{align} 
We note that this is only possible due to our prior knowledge of the matrix $\mP$, which allows us to construct the required transformations.

In both cases we use the test statistic $T = \sum_{i=1}^n \|\widehat{\mX}_i - \widehat{\mY}_i\|^2$; and accept or reject based on an estimate of the critical value of $T$ under the null hypothesis obtained by using $2000$ Monte Carlo iterates to estimate the distribution of $T$.

\begin{figure}[ht]
\centering
\includegraphics[width=\textwidth,keepaspectratio=TRUE,trim = 10 0 8 0,clip]{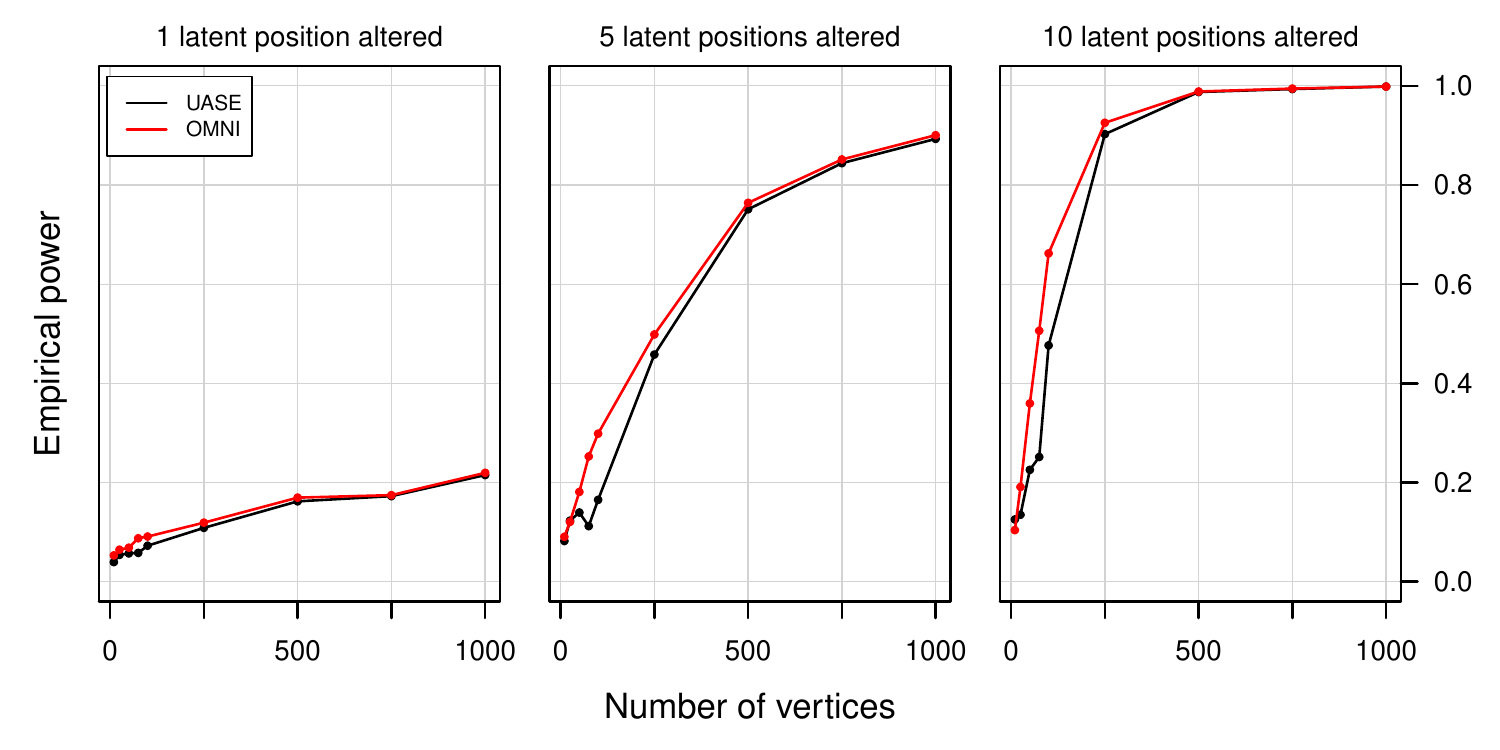}
\caption{\label{Fig_10_hypothesis_test}Empirical power of the UASE (black) and omnibus (red) tests to detect when the two graphs
being tested differ in the specified number of their latent positions. Each point is the proportion of 2000 trials for which the given technique correctly rejected the null hypothesis.}
\end{figure}

Figure \ref{Fig_10_hypothesis_test} shows the power of each method for testing the null hypothesis for different sized graphs and for different numbers of altered latent positions, by calculating the proportion (out of 2000 trials conducted for each sized graph) of trials for which we correctly reject the null hypothesis. For smaller graphs, the omnibus embedding provides the most effective method (although there is not much difference between the two where only one latent position is altered - however in this case the empirical power of both methods does not exceed 0.25). For larger graphs, particularly those with more than 500 vertices, the two methods are almost indistinguishable, and in such cases the UASE might be preferred based on size considerations, due to only requiring an $n \times kn$ rather than an $kn \times kn$ matrix.

\section{Real data: Link prediction on a computer network}\label{sec:real_data}

\subsection{Dynamic link prediction}\label{subsec:dynamic_link_prediction}

The Los Alamos National Laboratory computer network \cite{TKH18} was studied in \cite{RPTC18}, in which it was demonstrated empirically that the disassortative connectivity behaviour inherent in the network leads to the GRDPG offering a marked modelling improvement over the RDPG in the task of out-of-sample link prediction between computers in the network. For a large-scale dynamic network such as this, the MRDPG offers the possibility of further refinement by allowing us to consider multiple ``snapshots'' of communication behaviour at different points in time simultaneously.

As an example, we extract a ten minute sample at random from the ``Network Event Data'' dataset, which we divide into two separate five minute samples. From the first sample we generate five graphs, each one describing the communication behaviour of the computers in the network over a period of one minute, by assigning each IP address to a node (with this assignation being kept consistent across all graphs), and recording an edge between two nodes if the corresponding edges are observed to communicate at least once within this period, and then construct the corresponding adjacency matrices $\mA^{(r)}$. Setting our embedding dimension $d = 10$ (an admittedly arbitrary choice) we then generate estimates $\widehat{\mP}^{(r)}$ for the probability matrices as described in Section \ref{model_estimation_section}. We then use the average of these matrices to give us estimates of the probabilities of a link being generated between any given pair of computers.

In a similar manner, we generate estimates of the link probabilities for the mean adjacency matrix $\bar{\mA}$. We also construct adjacency matrices $\mA_{[1]}$ and $\mA_{[5]}$ from the connectivity graphs for the first minute and the full five minute period respectively (both of which follow a standard GRDPG) and generate link probability estimates accordingly. 

\begin{figure}[ht!]
\centering
\includegraphics[scale=0.68, trim = 0 0 0 30, clip]{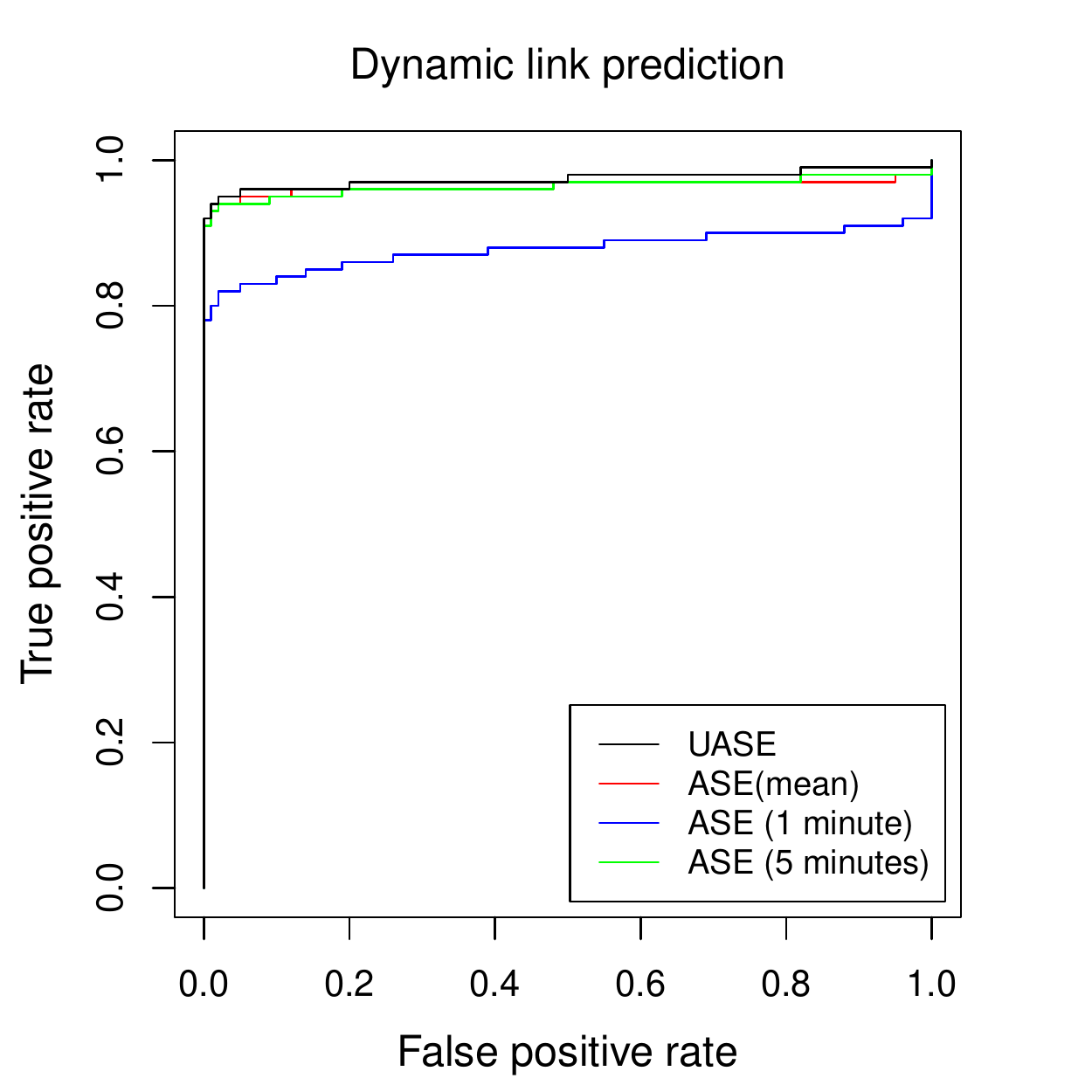}
\caption{\label{Fig_11_dynamic_link_prediction}Receiver Operating Characteristic curves for the UASE (black), mean embedding (red) and ASE (blue and green) methods for out-of-sample link prediction on the Los Alamos National Laboratory computer network. See main text for details.}
\end{figure}

Using these estimates, we attempt to predict which \textit{new} edges will occur within the second five minute window, disregarding those involving new nodes. Figure \ref{Fig_11_dynamic_link_prediction} shows the receiver operating characteristic (ROC) curves for each model and for each port, where we treat the prediction task as a binary classification problem whose outcomes are either the presence or absence of an edge between nodes, which we predict by thresholding the estimated link probabilities. As one would expect, using the ASE of the graph corresponding to only a single minute of communication (the blue curve) produces the least accurate predictions, but the UASE (black), mean embedding (red) and the ASE for a five minute sample (green) all produce similar results, with the UASE slightly outperforming the other two methods. This can be confirmed numerically by calculating the area under each ROC curve (AUC) which is equivalent to the probability that a given classifier will rank a randomly chosen positive instance higher than a randomly chosen negative instance \cite{F06}. We calculate that the UASE has an AUC of 0.9734, which is a small improvement over the mean embedding and the $5$-minute ASE, which have AUC values of 0.9627 and 0.9635 respectively (the $1$-minute ASE, by contrast, yields an AUC of 0.8767).

\subsection{Port-specific link prediction}

 The LANL network data presents the opportunity to demonstrate another significant improvement offered by the MRDPG over the GRDPG, namely its ability to integrate data from sources which do not necessarily behave similarly. Each communication within the LANL network passes through a source and destination port, the latter of which indicates the type of service being used, and it is natural to expect that different services may exhibit different communication behaviours.

 We consider the first five minute sample from the previous section. During the first minute alone, there are a total of 121,737 (not necessarily unique) communications between 10,762 computers, with 4,379 different destination ports being used. Of these, the 8 most commonly-used ports account for over 75\% of communications, and Table \ref{Table_1_netflow_port_data} lists these, together with the purpose to which each port is assigned.

\begin{table}[ht!]
\centering\footnotesize\caption{\label{Table_1_netflow_port_data}Purpose and proportion of traffic utilizing the 8 most frequently used ports during 1 minute of activity on the Los Alamos National Laboratory computer network.}
\begin{tabular}{ccc}
\hline
Port & Purpose & Proportion of traffic\\
\hline
53& DNS &27.7\%\\
443& HTTPS&14.8\%\\
80& HTTP&11.9\%\\
514& Syslog&7.2\%\\
389& Lightweight Directory Access Protocol&4.3\%\\
427& Service Location Protocol&4.1\%\\
88& Kerberos authentication system&3.4\%\\
445& Microsoft--DS Active Directory&1.8\%\\\hline

\end{tabular}\end{table}

For each of these 8 ports, we generate a graph of the communications made between computers within the network through this specific port over the first minute of our sample as in the previous example. Figure \ref{Fig_12_adjacency_matrix_plots} visualizes the adjacency matrix of each of these graphs as a 2-dimensional plot, together with the adjacency matrix of the full network graph. 

\begin{figure}[ht!]
\centering
\subfloat{\includegraphics[page=1,scale=0.8,trim=10 0 10 0,clip]{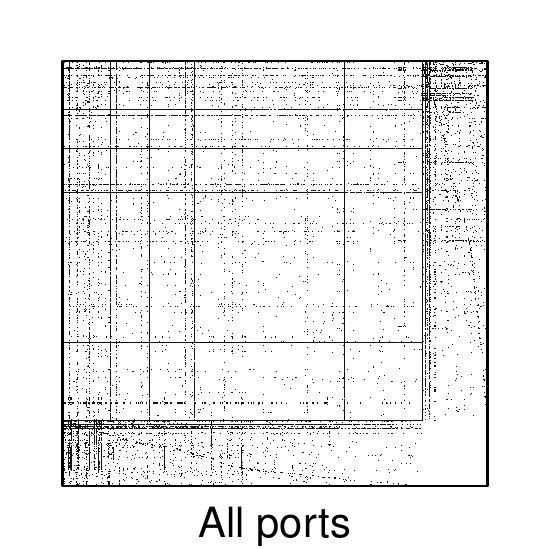}}
\subfloat{\includegraphics[page=2,scale=0.8,trim=10 0 10 0,clip]{Fig_12_adjacency_matrix_plots.pdf}}
\subfloat{\includegraphics[page=3,scale=0.8,trim=10 0 10 0,clip]{Fig_12_adjacency_matrix_plots.pdf}}\\\vspace{-1.5\baselineskip}
\subfloat{\includegraphics[page=4,scale=0.8,trim=10 0 10 0,clip]{Fig_12_adjacency_matrix_plots.pdf}}
\subfloat{\includegraphics[page=5,scale=0.8,trim=10 0 10 0,clip]{Fig_12_adjacency_matrix_plots.pdf}}
\subfloat{\includegraphics[page=6,scale=0.8,trim=10 0 10 0,clip]{Fig_12_adjacency_matrix_plots.pdf}}\\\vspace{-1.5\baselineskip}
\subfloat{\includegraphics[page=7,scale=0.8,trim=10 0 10 0,clip]{Fig_12_adjacency_matrix_plots.pdf}}
\subfloat{\includegraphics[page=8,scale=0.8,trim=10 0 10 0,clip]{Fig_12_adjacency_matrix_plots.pdf}}
\subfloat{\includegraphics[page=9,scale=0.8,trim=10 0 10 0,clip]{Fig_12_adjacency_matrix_plots.pdf}}
\caption{\label{Fig_12_adjacency_matrix_plots}Visualization of adjacency matrices showing connections between computers on the Los Alamos National Laboratory computer network during 1 minute of activity. The top-left image shows all connections during this time, while the remaining images show only connections via the specified port.}
\end{figure}

As before, we calculate estimates of the link probabilities for each port using the UASE, but now rather than averaging them we consider each port individually. For comparison, we also estimate link probabilities for each individual port using the corresponding GRDPG, and then use both estimates to attempt to predict which \textit{new} edges will occur within the remainder of our five-minute window. Figure \ref{Fig_13_port_specific_link_prediction} shows the ROC curves for each model and for each port, while Table \ref{Table_2_netflow_port_AUC} gives the AUC values for each curve. We note that for Port 53 (the busiest port) the standard ASE actually outperforms the UASE, but for every other port the UASE is the superior method, offering a significant improvement over the ASE for the less active ports.

\begin{table}[ht!]
\centering\footnotesize
\caption{\label{Table_2_netflow_port_AUC}AUC values for the ROC curves in Figure \ref{Fig_13_port_specific_link_prediction}.}
\begin{tabular}{ccccccccc}
\hline
Embedding&Port 53&Port 443&Port 80&Port 514&Port 389&Port 427&Port 88&Port 445\\
\hline
UASE&0.8391&0.7850&0.9010&0.8931&0.8534&0.8580&0.8949&0.9836\\
ASE&0.8568&0.6370&0.7185&0.7806&0.5532&0.6566&0.5668&0.5720\\\hline
\end{tabular}\end{table}

\begin{figure}[ht!]
\centering
\centering
\includegraphics[scale=0.68]{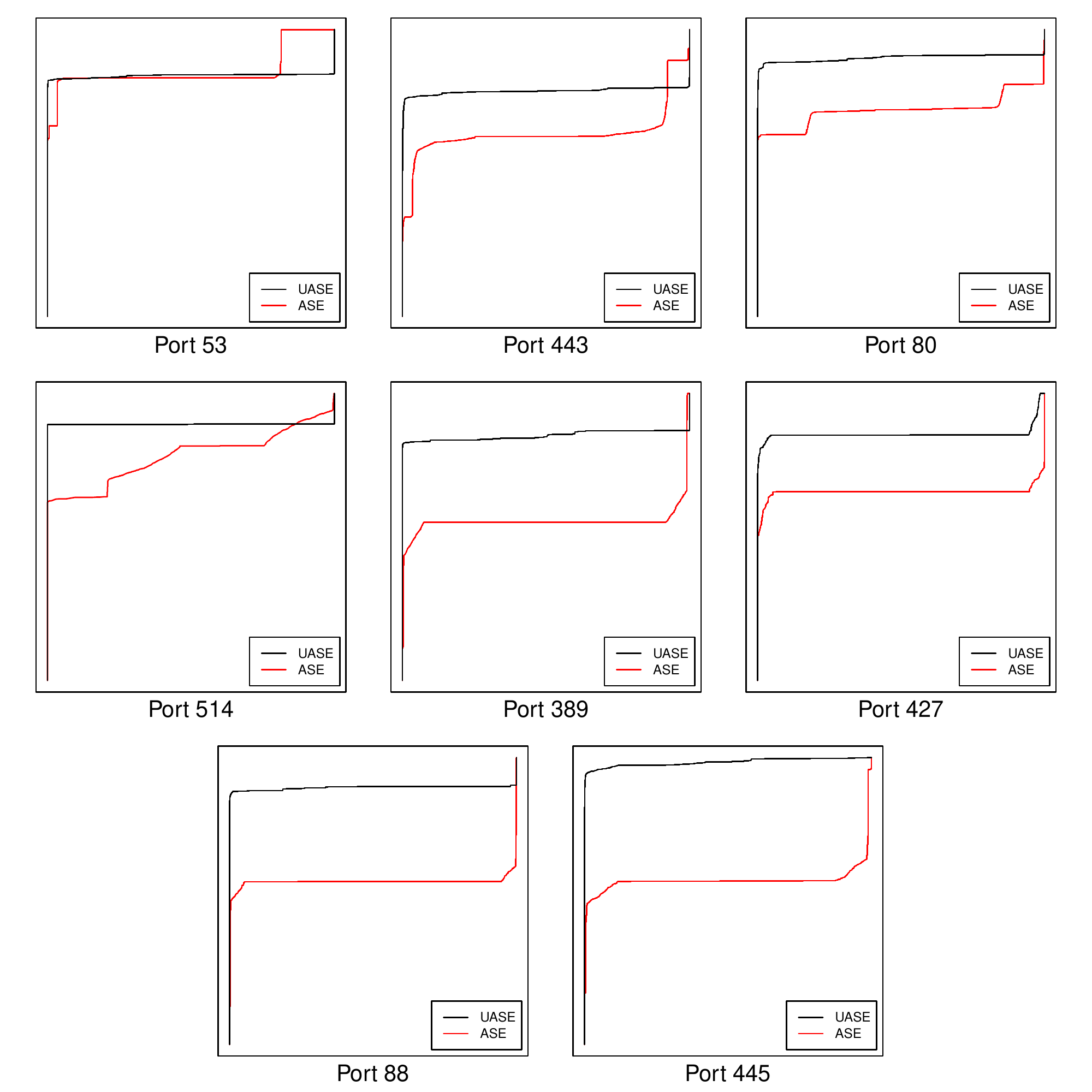}
\caption{\label{Fig_13_port_specific_link_prediction}Receiver Operating Characteristic curves for the UASE (black) and ASE (red) for out-of-sample link prediction via the specified port on the Los Alamos National Laboratory computer network. See main text for details.}
\end{figure}

\subsection{Link prediction using mixed data sources}

 Our final example demonstrates how combining data from graphs with \textit{different} nodes can inform our knowledge of a common subset. We begin by extracting a five minute sample at random from the ``Network Event Data'' dataset similarly to Section \ref{subsec:dynamic_link_prediction}, and construct the ASE of the adjacency matrix $\mA^{(1)}$ obtained by restricting our attention to the first minute of computer-to-computer communication. For the UASE, we augment this data by constructing the submatrix $\mA^{(2)}$ of computer-to-port communications, in which $\mA^{(2)}_{ij} = 1$ if there is a connection involving the $i$th computer during this first minute for which the $j$th port is the destination port. We then generate estimates of the link probabilities for the ASE of $\mA^{(1)}$ and UASE of $\mA = [\mA^{(1)}|\mA^{(2)}]$, and generate link probability estimates accordingly. Figure \ref{Fig_14_mixed_data_link_prediction} plots the resulting ROC curves, in which we observe that augmenting the computer-to-computer connectivity data with computer-to-port connectivity data yields an improvement in prediction power (similarly, we note the AUC values of $0.949$ for the UASE compared to $0.905$ for the ASE).

\begin{figure}[ht!]
\centering
\centering
\includegraphics[scale=0.68]{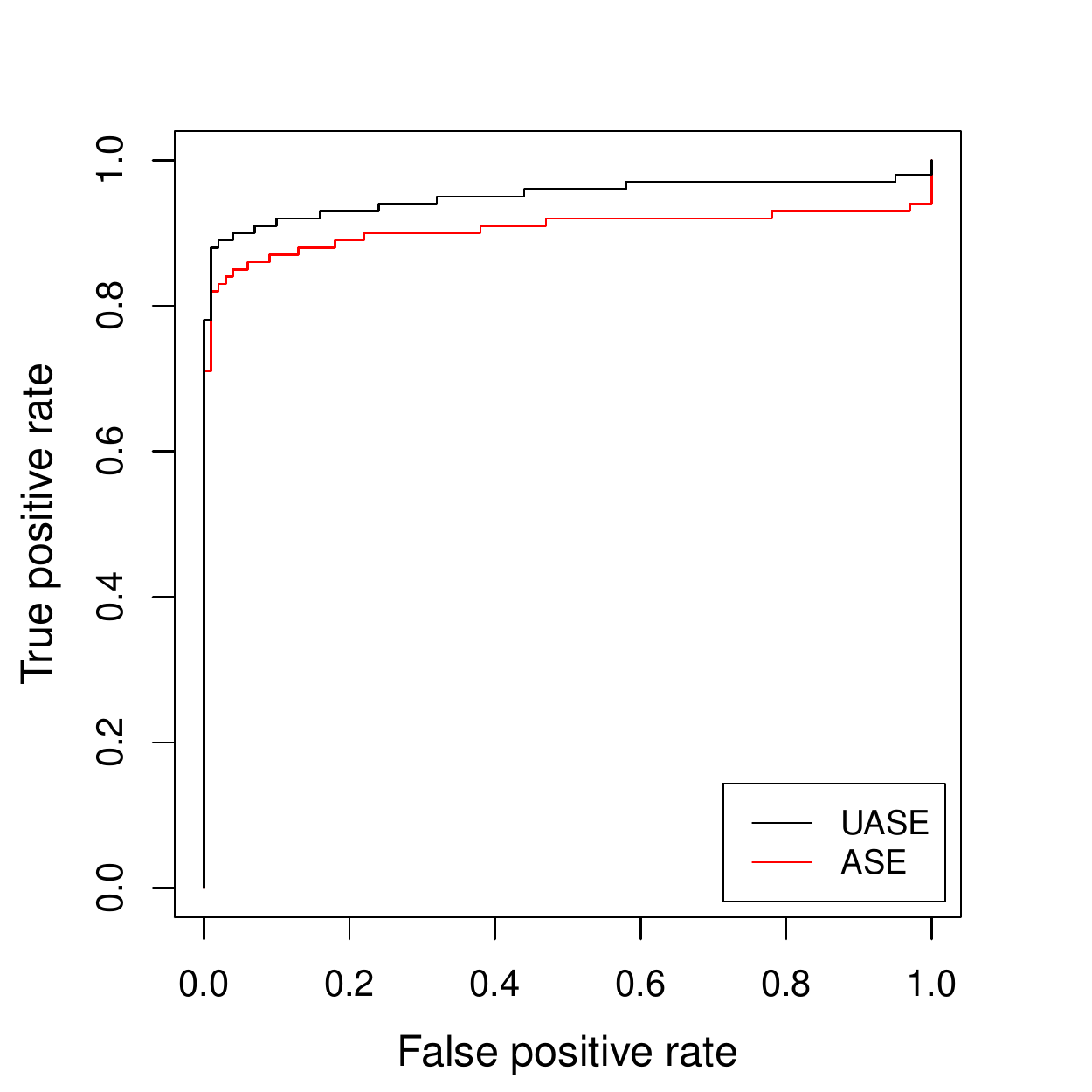}
\caption{\label{Fig_14_mixed_data_link_prediction}Receiver Operating Characteristic curves for the UASE (black) and ASE (red) methods for out-of-sample link prediction on the Los Alamos National Laboratory computer network, in which the UASE is augmented with computer-to-port connectivity data. See main text for details.}
\end{figure}

\section{Chernoff information and the GMSBM}\label{sec:Chernoff}

 Given a collection of matrices $\mA^{(r)}$ that are distributed according to a GMSBM, it is reasonable to ask whether there is any tangible benefit to studying the UASE as opposed to the ASEs of the individual matrices, and how one might quantify this. Tang and Priebe \cite{TP19}, in the context of comparing the performance of the spectral embeddings of the Laplacian and adjacency matrix in recovering block assignments from a stochastic block model graph, proposed using the \textit{Chernoff information} \cite{C52} of the limiting Gaussian distributions obtained from the Central Limit Theorem associated to each embedding as a means of doing so. In a two-cluster problem, the Chernoff information is the exponential rate at which the Bayes error (from the decision rule which assigns each data point to its most likely cluster \textit{a posteriori}) decreases asymptotically. The Chernoff information is an example of a \textit{f}-divergence \cite{AS66}, \cite{C67} and therefore possesses the desirable attribute of being invariant under invertible linear transformations \cite{LV06}.

 If $F_1$ and $F_2$ are two absolutely continuous multivariate distributions supported on $\Omega \subset \mathbb{R}^d$, with density functions $f_1$ and $f_2$ respectively, the Chernoff information between $F_1$ and $F_2$ is defined by $C(F_1,F_2) = \displaystyle\sup_{t \in (0,1)} C_t(F_1,F_2)$ \cite{C52} where the \textit{Chernoff divergence} 
\begin{align}
	C_t(F_1,F_2) = -\log\bigl(\int_\Omega f_1(\mx)^t f_2(\mx)^{1-t} d\mx\bigr).
\end{align} 

For a $K$-cluster problem, in which we have distributions $F_1, \ldots, F_K$ with corresponding density functions $f_1, \ldots, f_K$, we consider the Chernoff information of the critical pair $\displaystyle\min_{i \neq j} C(F_i,F_j)$.

If the $F_i$ are multivariate normal distributions, it is known (see \cite{P05}) that the Chernoff information can be expressed as $C(F_i,F_j) = \displaystyle\sup_{t \in (0,1)} C_t(F_i,F_j)$, where 
\begin{align}
	C_t(F_i,F_j) = \Bigl(\frac{t(1-t)}{2}(\mx_i - \mx_j)^\top \mSig_t^{-1}(\mx_i - \mx_j) + \tfrac{1}{2}\log\Bigl(\frac{|\mSig_t|}{|\mSig_i|^t|\mSig_j|^{1-t}}\Bigr)\Bigr),
\end{align} 
where $F_i \sim \mathcal{N}(\mx_i,\mSig_i)$ and $\mSig_t = t\mSig_1 + (1-t)\mSig_2$.

We now apply this to obtain an expression for the Chernoff information of the left UASE of a GMSBM. Let $(\mA,\mX,\mY) \sim \mathrm{GMSBM}(\mathcal{F},\mB)$, and define $C_\mA = \displaystyle\min_{i \neq j} \displaystyle\sup_{t \in (0,1)} C_t\left(F_i,F_j\right)$, where 
\begin{align}
	F_i \sim \mathcal{N}\bigl(\me_i,\tfrac{1}{n}\Delta_{\mB,Y}^{-1}\mB\mSig_Y(\me_i)\mB^\top\Delta_{\mB,Y}^{-1}\bigr)
\end{align} 
and $\Delta_{\mB,Y}$ and $\mSig_Y(\me_i)$ are as defined in Theorem \ref{CLT} (for simplicity, we will assume that our graphs have the same sparsity factor). 

Some standard algebraic manipulation shows that this quantity is invariant under the transformations of the latent positions listed in Proposition \ref{latent_positions} (and thus the Chernoff information of the underlying MRDPG is well-defined), and similarly that it is invariant under invertible linear transformations of the UASE $\mX_\mA$. Thus we may study $\mX_\mA$ rather than the estimate $\mX_\mA\mL$ of the latent positions, which requires knowledge - that we will not typically possess - of the underlying matrices of probabilities $\mP^{(r)}$. We note that we can similarly define the Chernoff information for the right UASEs (where they are defined) and that this too would be invariant under invertible transformations of the latent positions.

Given a collection $\mA^{(1)}, \ldots, \mA^{(k)}$ of matrices and a subset $\mathcal{K}$ of $\{1,\ldots,k\}$, let $\mA_\mathcal{K}$ denote the matrix obtained by concatenating the matrices $\mA^{(r)}$ for $r \in \mathcal{K}$. If we have a collection of matrices $\mA^{(r)}$ which are \textit{identically} distributed as a GMSBM, the following result shows that it is \textit{always} preferable to embed as many matrices as possible:

\begin{proposition}\label{Chernoff_identical_embeddings} Let $(\mA,\mX,\mY) \stackrel{\mathrm{id}}\sim \mathrm{GMSBM}(\mathcal{F},\mB)$ be identically distributed as a GMSBM. Then $C_\mA \geq C_{\mA_\mathcal{K}}$ for any subset $\mathcal{K}$ of $\{1,\ldots,k\}$.\end{proposition}

\begin{proof} 
Let $\mSig_i = \mB^{-\top}\Delta_Y^{-1}\mSig_Y(\me_i)\Delta_Y^{-1}\mB^{-1}$ denote the covariance matrix present in the Central Limit Theorem for the single graph ASE given in Corollary \ref{identical_MRDPG_CLT} for our particular case. Then 
\begin{align}
	C_\mA = \min_{i \neq j} \sup_{t \in (0,1)} \Bigl(\frac{kt(1-t)}{2c}(\me_i - \me_j)^\top \mSig_{i,j,t}^{-1}(\me_i - \me_j) + \tfrac{1}{2}\log\Bigl(\frac{|\mSig_{i,j,t}|}{|\mSig_i|^t|\mSig_j|^{1-t}}\Bigr)\Bigr) 
\end{align} 
where $\mSig_{i,j,t} = t\mSig_i + (1-t)\mSig_j$, while for a subset $\mathcal{K}$ of size $r$ we have 
\begin{align}C_{\mA_\mathcal{K}} = \min_{i \neq j} \sup_{t \in (0,1)} \Bigl(\frac{rt(1-t)}{2c}(\me_i - \me_j)^\top \mSig_{i,j,t}^{-1}(\me_i - \me_j) + \tfrac{1}{2}\log\Bigl(\frac{|\mSig_{i,j,t}|}{|\mSig_i|^t|\mSig_j|^{1-t}}\Bigr)\Bigr). 
\end{align}

Since the matrix $\mSig_{i,j,t}$ is positive semi-definite for any $i$ and $j$, the first term in each pair of brackets is positive, and thus the term $C_\mA$ clearly dominates. 
\end{proof}

If the adjacency matrices $\mA^{(r)}$ are \textit{not} identically distributed, however, the situation is not so clear-cut, as it is entirely possible to encounter situations for which $C_{\mA_\mathcal{K}} > C_\mA$ for some subset $\mathcal{K}$. For example, if we consider the two-graph multilayer stochastic block model with matrices 
\begin{align}
	\mB^{(1)} = \left(\begin{array}{ccc}0.67&&0.46\\$0.4$6&&0.36\end{array}\right),~\mB^{(2)} = \left(\begin{array}{ccc}0.98&&0.49\\0.49&&0.10\end{array}\right),
\end{align} 
then while the ratio $C_{\mA}/C_{\mA^{(1)}}$ tends to $11.98$, the ratio $C_{\mA}/C_{\mA^{(2)}}$ tends to $0.96$.

Before proceeding further, we note that any analysis of the Chernoff information for large-scale GMSBMs can be simplified by observing that the logarithmic term in the definition is independent of the number of vertices $n$, and so becomes insignificant as $n \to \infty$ if we impose the simplifying assumption that the covariance matrices be non-singular. To this end, we consider instead the truncated terms $\rho_\mA$, in which we simply omit the logarithmic term from the definition of $C_\mA$. Considering the function $\rho_\mA$ gives an accurate means of comparison between the large-scale behaviour of two multilayer stochastic block models, as the ratio $C_\mA/C_{\mA_\mathcal{K}}$ tends to $\rho_\mA/\rho_{\mA_\mathcal{K}}$ as $n$ increases for any subset $\mathcal{K}$.

To observe the effect of embedding additional graphs on the Chernoff information of a GMSBM, we conducted the following experiment: for each $k \in \{2,\ldots,10\}$, we performed 1000 trials in which a $(2, 2, \ldots, 2)$-community GMSBM was generated by choosing $k$ matrices $\mB^{(r)} \in [0,1]^{2 \times 2}$ and a probability vector $\mpi$ at random (with the entries $\mB^{(r)}_{ij} \sim \mathrm{Uniform}[0,1]$, while $\pi \sim \mathrm{Dirichlet}(1,1)$. We then calculated the ratio $\rho_\mA/\rho_{\mA_\mathcal{K}}$ for each subset $\mathcal{K}$ of $\{1,\ldots,k\}$ under the assumption that the vectors $\mY^{(r)}$ were identically distributed. We then repeated the experiment for a $(2,2,\ldots,2)$-community GMSBM in which we allowed different probability vectors $\mpi_r$ for each embedding, and finally repeated both experiments for a $(2,3,\ldots,3)$-community GMSBM. 

The results of these simulations are presented in Figure \ref{Fig_15_chernoff_information}. Entries above the diagonal (in blue) correspond to trials in which the parameters $\mpi_r$ are equal, while entries below the diagonal (in red) correspond to trials in which the parameters $\mpi_r$ are allowed to differ; the $(r,k)$th and $(k,r)$th coordinates in the respective cases indicate the proportion of embeddings of $k$ matrices for which $\rho_\mA/\rho_{\mA_\mathcal{K}} > 1$ for \textit{every} $r$-element subset $\mathcal{K}$. 

\begin{figure}[ht!]
\centering
\includegraphics[scale=0.62,trim=10 10 0 10,clip]{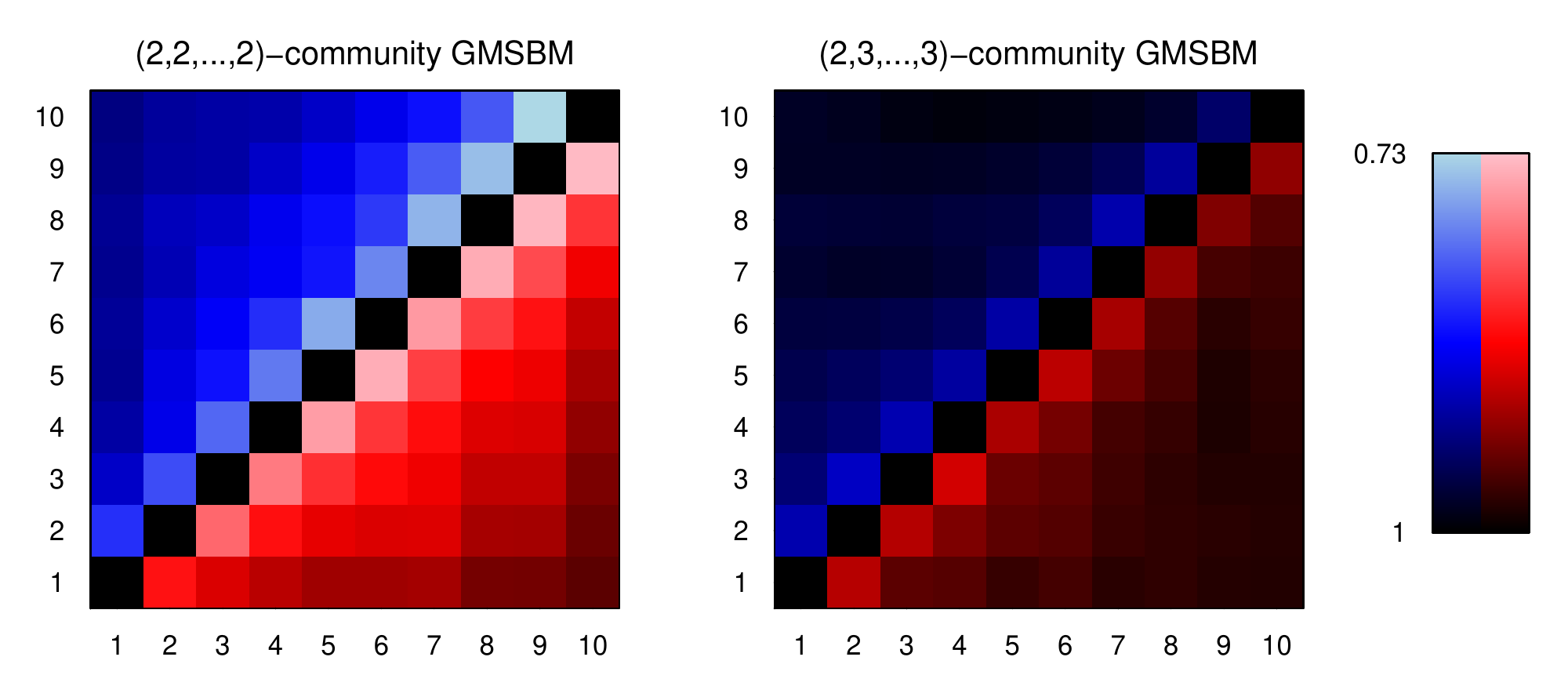}
\caption{\label{Fig_15_chernoff_information}Proportion of GMSBMs for which $\rho_\mA/\rho_{\mA_\mathcal{K}} > 1$ for all subsets $\mathcal{K}$ of a given size. See main text for details.}
\end{figure}

We make the following observations:
\begin{itemize}
	\item On average, embedding more matrices seems to improve performance, with $\rho_\mA/\rho_{\mA_\mathcal{K}} > 1$ for at least $73\%$ of the trials we conducted, and this improvement in performance seems to increase as we allow the latent positions $\mY^{(r)}$ to be taken from more communities.
	\item In general, $\rho_\mA/\rho_{\mA_\mathcal{K}} > 1$ more often the greater the difference in size between the subset $\mathcal{K}$ and the total number of embeddings $k$.
	\item We detected little distinguishable difference between allowing the parameters $\mpi_r$ to differ as opposed to keeping them all the same.
	\item In each trial, we noted that there is always at least \textit{one} subset $\mathcal{K}$ of any given size for which $\rho_\mA/\rho_{\mA_\mathcal{K}} > 1$, and we conjecture that this should always hold. If true, this would mean in particular that the UASE is always better than the worst of our embeddings (as opposed to, say, the mean embedding, which we have seen can lead to degeneracy when the matrices $\mP^{(r)}$ have differing signatures).
\end{itemize}

\section{Conclusion}\label{sec:conclusion}

The multilayer random dot product graph is a vast yet natural extension of the random dot product graph, granting us insight into the behaviour of a common subset of nodes across a series of graphs---both undirected and directed---in which we allow a mixture of assortativity behaviours. Its simplicity and flexibility make it an ideal model for a variety of situations, and it can be seen to perform as well as (and in many cases better than) existing models at multiple graph inference tasks such as community detection and graph-to-graph comparison, while allowing inference in a much wider range of situations than current methods permit. These experimental results are supported by theoretical results showing that the node representations obtained by the left- and right-sided spectral embeddings converge uniformly in the Euclidean norm to the latent positions with Gaussian error, in particular providing us with the first known examples of such results for bipartite graphs. Finally, we demonstrate the practical effectiveness of our model by applying it to the task of link prediction within a computer network, indicating its usefulness to the field of cyber-security.

\bibliographystyle{plainnat} 
\bibliography{MRDPG_v3_bib.bib}

\section*{Appendix}\label{appn}
\addcontentsline{toc}{section}{\nameref{appn}}

Before proving Theorems \ref{consistency} and \ref{CLT}, we first require some control over the asymptotic behaviour of the singular values of the matrices $\mP$, $\mA$ and $\mA-\mP$, which we establish using a series of results. Throughout this section, we will assume that $(\mA,\mX, \mY) \sim \mathrm{MRDPG}(\mathcal{F}_\rho,\mLam_\epsilon)$ for a distribution $\mathcal{F}$ and sparsity factors $\rho$ and $\epsilon_r$ satisfying the criteria stated in Section \ref{subsec:asymptotics}.

\begin{proposition}\label{P_bounds} The non-zero singular values $\sigma_i(\mP)$ for $i \in \{1,\ldots,d\}$ satisfy 
\begin{align}
	\frac{\sigma_i(\mP)}{\rho n} \longrightarrow \sqrt{\lambda_i\bigl(\Delta_X\mLam_*\Delta_Y\mLam_*^\top\bigr)}
\end{align} 
almost surely, where $\Delta_Y = c_1\Delta_{Y,1} \oplus \cdots \oplus c_\kappa\Delta_{Y,\kappa}$, and consequently $\sigma_i(\mP) = \mathrm{O}(\epsilon\rho n)$ and $\sigma_i(\mP) = \Omega(\rho n)$ almost surely.\end{proposition}

\begin{proof} Since $\mP = \mX\mLam_\epsilon\mY^\top$, we see that 
\begin{align}
	\sigma_i(\mP) = \sqrt{\lambda_i\bigl(\mX\mLam_\epsilon\mY^\top\mY\mLam_\epsilon^\top\mX^\top\bigr)} = \sqrt{\lambda_i\bigl(\mX^\top\mX\mLam_\epsilon\mY^\top\mY\mLam_\epsilon^\top\bigr)},\end{align} 
as the non-zero eigenvalues of a product of matrices are invariant under cyclic permutations of its factors. 

Our assumptions regarding the distribution $\mathcal{F}$ ensure that the spectral norms $\|\mX^\top\mX - \rho n\Delta_X\|$ and $\|\mY^\top\mY - \rho n\Delta_Y\|$ are of order $\mathrm{O}\bigl(\rho n^{1/2}\log^{1/2}(n)\bigr)$ mutually almost surely, and we note that consequently 
\begin{align}
	\|\mY^\top\mY\| \leq \rho \|\Delta_Y\| + \|\mY^\top\mY - \rho \Delta_Y\| = \mathrm{O}\left(\rho n\right)
\end{align}
almost surely by a standard application of the triangle inequality. 

Next, note that 
\begin{align}
\begin{aligned}
	\mX^\top\mX\mLam_\epsilon\mY^\top\mY\mLam_\epsilon^\top - \rho^2 n\Delta_X\mLam_\epsilon\Delta_Y\mLam_\epsilon^\top &= (\mX^\top\mX-\rho n\Delta_X)\mLam_\epsilon\mY^\top\mY\mLam_\epsilon^\top \\
	&\phantom{=} + \rho n\Delta_X \mLam_\epsilon (\mY^\top\mY - \rho \Delta_Y)\mLam_\epsilon^\top
\end{aligned}
\end{align}
and so 
\begin{align}
\begin{aligned}
	\|\mX^\top\mX\mLam_\epsilon\mY^\top\mY\mLam_\epsilon^\top - \rho^2 n\Delta_X\mLam_\epsilon\Delta_Y\mLam_\epsilon^\top\| &\leq \|\mLam_\epsilon\|^2\bigl(\|\mX^\top\mX-\rho n\Delta_X\|\|\mY^\top\mY\| \\
	&\phantom{=}+ \rho n\|\Delta_X \| \|\mY^\top\mY - \rho \Delta_Y\|\bigr)
\end{aligned}
\end{align}
meaning that
\begin{align}
	\|\mX^\top\mX\mLam_\epsilon\mY^\top\mY\mLam_\epsilon^\top - \rho^2 n\Delta_X\mLam_\epsilon\Delta_Y\mLam_\epsilon^\top\| = \mathrm{O}\bigl(\epsilon^2 \rho^2 n^{3/2}\log^{1/2}(n)\bigr) 
\end{align}
almost surely.

As a result, we see that $\frac{1}{\rho^2 n^2}\mX^\top\mX\mLam_\epsilon\mY^\top\mY\mLam_\epsilon^\top$ converges to $\Delta_X\mLam_\epsilon\Delta_Y\mLam_\epsilon^\top$ in the spectral norm, and thus also in the Frobenius norm, implying in particular that the entries of the two matrices converge in absolute value. Now, the eigenvalues of any matrix are the roots of its characteristic polynomial, the coefficients of which are polynomial functions of the entries of the matrix. In particular, by continuity of the roots of polynomials, the eigenvalues of $\frac{1}{\rho^2 n}\mX^\top\mX\mLam_\epsilon\mY^\top\mY\mLam_\epsilon^\top$ and $\Delta_X\mLam_*\Delta_Y\mLam_*^\top$ must converge, giving the first result.

For the second, note that $\lambda_i\bigl(\Delta_X\mLam_*\Delta_Y\mLam_*^\top\bigr) = \lambda_i\bigl(\Delta_X^{1/2}\mLam_*\Delta_Y\mLam_*^\top\Delta_X^{1/2}\bigr)$ (where $\Delta_X^{1/2}$ is the unique positive definite matrix square root of $\Delta_X$) and that the latter matrix is a sum of symmetric matrices, allowing us to apply Weyl's inequalities to obtain the upper and lower bounds. \end{proof}

\begin{proposition}\label{spectral_bound} $\|\mA-\mP\| = \mathrm{O}\bigl(\epsilon^{1/2}\rho^{1/2} k^{1/4}n^{1/2}\log^{1/2}(n)\bigr)$ almost surely.\end{proposition}

\begin{proof} Condition on some choice of latent positions. We will make use of a matrix analogue of the Bernstein inequality (see \cite{T15}, Theorem 1.6.2):

\begin{theorem}[Matrix Bernstein] Let $\mM_1, \ldots, \mM_n$ be independent random matrices with common dimensions $m_1 \times m_2$, satisfying $\mathbb{E}[\mM_k] = 0$ and $\|\mM_k\| \leq L$ for each $1 \leq k \leq n$, for some fixed value $L$. 

Let $\mM = \sum \mM_k$ and let $v(\mM)= \max\{\|\mathbb{E}[\mM\mM^\top]\|,\|\mathbb{E}[\mM^\top\mM]\|\}$ denote the matrix variance statistic of $\mM$. Then for all $t \geq 0$, we have 
\begin{align}
	\mathbb{P}\bigl(\|\mM\| \geq t\bigr) \leq (m_1 + m_2)\exp\left(\tfrac{-t^2/2}{v(\mM) + Lt/3}\right).
\end{align}
\end{theorem}

We apply this as follows: given $r \in \{1,\ldots,k\}$, define a matrix $\mT_{ij}^{(r)} \in \mathbb{R}^{n \times (n_1 + \ldots + n_k)}$ for each $i \in \{1,\ldots,n\}$ and $j \in \{1,\ldots,n_r\}$ whose $(i,n_1 + \ldots + n_{r-1}+j)$th entry is equal to $\mA_{ij}^{(r)} - \mP_{ij}^{(r)}$, with all other entries equal to $0$ (in other words, if we divide $\mT_{ij}^{(r)}$ into distinct $n \times n_t$ blocks, then the $r$th block is the only non-zero one, and within this block only the $(i,j)$th entry is non-zero).

We then define matrices $\mM_{ij}^{(r)}$ for each $r \in \{1,\ldots,k\}$ as follows:
\begin{itemize}
	\item If $\mA^{(r)}$ is bipartite, we define $\mM_{ij}^{(r)} = \mT_{ij}^{(r)}$ for all $i \in \{1,\ldots,n\}$ and $j \in \{1,\ldots,n_r\}$;
	\item If $\mA^{(r)}$ is directed, we define $\mM_{ij}^{(r)} = \mT_{ij}^{(r)}$ for all $i,j \in \{1,\ldots,n_r\}$ with $i \neq j$;
	\item If $\mA^{(r)}$ is undirected, we define $\mM_{ij}^{(r)} = \mT_{ij}^{(r)} + \mT_{ji}^{(r)}$ for all $i,j \in \{1,\ldots,n_r\}$ with $i < j$.
\end{itemize}

Observe that $\|\mM_{ij}^{(r)}\| = \left|\mA_{ij}-\mP_{ij}\right| < 1$, and by definition $\mathbb{E}[\mM_{ij}^{(r)}] = 0$, so the matrix sum $\mM = \sum_{i,j,r}\mM_{ij}^{(r)}$ satisfies the criteria for Bernstein's Theorem (where we sum over the variables $i$, $j$ and $r$ according to the cases above). To bound the matrix variance statistic $v(\mM)$, let $\mM_r = \sum_{i,j} \mM^{(r)}_{ij}$, and note that 
\begin{align}
	\mM_r\mM_r^\top = \left\{\begin{array}{cccc}\displaystyle\sum_{l = 1}^{n_r} \bigl(\mA_{il}^{(r)} - \mP_{il}^{(r)}\bigr)\bigl(\mA_{jl}^{(r)} - \mP_{jl}^{(r)}\bigr)&&&\mathrm{if~}\mA^{(r)}\mathrm{~is~bipartite}\\\\\displaystyle\sum_{l \neq i,j} \bigl(\mA_{il}^{(r)} - \mP_{il}^{(r)}\bigr)\bigl(\mA_{jl}^{(r)} - \mP_{jl}^{(r)}\bigr)&&&\mathrm{otherwise},\end{array}\right.
\end{align} 
and therefore that 
\begin{align}
	\mathbb{E}[\mM_r\mM_r^\top]_{ij} = \left\{\begin{array}{cccc}\displaystyle\sum_{l = 1}^{n_r} \mP_{il}^{(r)}\bigl(1 - \mP_{il}^{(r)}\bigr)&&&\mathrm{if~}\mA^{(r)}\mathrm{~is~bipartite~and~}i = j\\\\\displaystyle\sum_{l \neq i} \mP_{il}^{(r)}\bigl(1 - \mP_{il}^{(r)}\bigr)&&&\mathrm{if~}\mA^{(r)}\mathrm{~is~not~bipartite~and~}i = j\\\\0&&&\mathrm{if~}i \neq j.\end{array}\right.
\end{align} 
By definition, $\mP_{il}^{(r)}(1-\mP_{il}^{(r)}) \leq \epsilon_r\rho$ for all $i$ and $l$, and so since $\mathbb{E}[\mM_r\mM_r^\top]$ is diagonal, we see that $\|\mathbb{E}[\mM_r\mM_r^\top]\| \leq \epsilon_r\rho n_r$. Since $\mM\mM^\top = \sum \mM_r\mM_r^\top$ and the $\mM_r$ are independent, it follows that $\|\mathbb{E}[\mM\mM^\top]\| \leq (\epsilon_1n_1 + \ldots + \epsilon_kn_k)\rho \leq \epsilon\rho k^{1/2}n_{\max}$ by the Cauchy--Schwarz inequality. 

Similarly,
\begin{align}
	[\mM_r^\top\mM_r]_{ij} = \left\{\begin{array}{cccc}\displaystyle\sum_{l = 1}^{n} \bigl(\mA_{li}^{(r)} - \mP_{li}^{(r)}\bigr)\bigl(\mA_{lj}^{(r)} - \mP_{lj}^{(r)}\bigr)&&&\mathrm{if~}\mA^{(r)}\mathrm{~is~bipartite}\\\\\displaystyle\sum_{l \neq i,j} \bigl(\mA_{li}^{(r)} - \mP_{li}^{(r)}\bigr)\bigl(\mA_{lj}^{(r)} - \mP_{lj}^{(r)}\bigr)&&&\mathrm{otherwise},\end{array}\right.
\end{align} 
and so 
\begin{align}
	\mathbb{E}[\mM_r^\top\mM_r]_{ij} = \left\{\begin{array}{cccc}\displaystyle\sum_{l = 1}^{n} \mP_{li}^{(r)}\bigl(1 - \mP_{li}^{(r)}\bigr)&&&\mathrm{if~}\mA^{(r)}\mathrm{~is~bipartite~and~}i = j\\\\\displaystyle\sum_{l \neq i} \mP_{li}^{(r)}\bigl(1 - \mP_{li}^{(r)}\bigr)&&&\mathrm{if~}\mA^{(r)}\mathrm{~is~not~bipartite~and~}i = j\\\\0&&&\mathrm{if~}i \neq j,\end{array}\right.
\end{align} 
and as before we see that $\|\mathbb{E}[\mM_r^\top\mM_r]\| \leq \epsilon_r\rho n$, while $\|\mathbb{E}[\mM_s^\top\mM_t]\| = 0$ if $s \neq t$ since the matrices $\mM_r$ are independent. Thus the matrix $\mathbb{E}[\mM^\top\mM]$ is block diagonal, and so it follows that $\|\mathbb{E}[\mM^\top\mM]\| \leq \rho n$, and so certainly $v(\mM) = \mathrm{O}(\epsilon\rho k^{1/2}n)$.

Substituting these into Bernstein's Theorem and rearranging, we find that, for any $t \geq 0$, 
\begin{align}
	\mathbb{P}\bigl(\|\mM\| \geq t\bigr) \leq (n+n_1 + \ldots + n_k) \exp\left(\frac{-3t^2}{6\rho k^{1/2}n + 2t}\right).
\end{align} 
The numerator of the exponential term dominates for $n$ sufficiently large if \\$t = \mathrm{O}\bigl(\epsilon^{1/2}\rho^{1/2} k^{1/4}n^{1/2}\log^{1/2}(n)\bigr)$, and so $\|\mM\|=\mathrm{O}\bigl(\epsilon^{1/2}\rho^{1/2} k^{1/4}n^{1/2}\log^{1/2}(n)\bigr)$ almost surely. 

Finally, we note that $\mM = \mA - \mP + \mP_0$, where $\mP_0$ comprises $k$ distinct blocks of sizes $n \times n_r$, with the $r$th such block either identically zero or diagonal, with entries equal to the diagonal entries of $\mP^{(r)}$, depending on whether or not $\mA^{(r)}$ is bipartite, and satisfies $\|\mP_0\| \leq \epsilon \rho$. The result then follows from subadditivity of the spectral norm and integrating over all possible choices of latent positions. \end{proof}

\begin{corollary}\label{A_bounds} The leading $d$ singular values of $\mA$ satisfy 
\begin{align}
	\frac{\sigma_i(\mA)}{\rho n} \longrightarrow \sqrt{\lambda_i\bigl(\Delta_X\mLam_*\Delta_Y\mLam_*^\top\bigr)}
\end{align} 
almost surely, where $\Delta_Y = c_1\Delta_{Y,1} \oplus \cdots \oplus c_\kappa\Delta_{Y,\kappa}$, and consequently $\sigma_i(\mA) = \mathrm{O}\left(\epsilon\rho n\right)$ and $\sigma_i(\mA) = \Omega\left(\rho n\right)$ almost surely.\end{corollary}

\begin{proof} A corollary of Weyl's inequalities (see, for example, \cite{HJ12}, Corollary 7.3.5) states that $|\sigma_i(\mA) - \sigma_i(\mP)| \leq \|\mA-\mP\|$, and so in particular (applying the reverse triangle inequality where necessary) 
\begin{align}
	\sigma_i(\mP) - \|\mA-\mP\| \leq \sigma_i(\mA) \leq \sigma_i(\mP) - \|\mA-\mP\|.
\end{align} 

The result then follows from Propositions \ref{P_bounds} and \ref{spectral_bound}. \end{proof}

The next few results provide bounds on the asymptotic growth of a number of residual terms in the proofs of our main theorems. While the proofs are similar in nature to a number of results in \cite{LTAPP17}, there are some minor differences to account for the fact that the matrices $\mA$ and $\mP$ are not symmetric, and so we reproduce them in full. We begin an analogue of a bound appearing in Lemma 17 of \cite{LTAPP17}:

\begin{proposition}\label{orth_spectral_bound} $\|\mU_\mP^\top(\mA-\mP)\mV_\mP\|_F = \mathrm{O}\bigl(\log^{1/2}(n)\bigr)$ almost surely.
\end{proposition}

\begin{proof} Condition on some choice of latent positions. For any $i, j \in \{1,\ldots,d\}$ and $r \in \{1,\ldots,k\}$, let 
\begin{align}
	\mE_{ij}^{(r)} = \left\{\begin{array}{cccc}\displaystyle\sum_{p=1}^n\sum_{q=1}^{n_r} u_pv^{(r)}_q\bigl(\mA^{(r)}_{pq} - \mP^{(r)}_{pq}\bigr)&&& \mathrm{if~}\mA^{(r)} \mathrm{~is~bipartite}\\&&&\\\displaystyle\sum_{p \neq q} u_pv^{(r)}_q\bigl(\mA^{(r)}_{pq} - \mP^{(r)}_{pq}\bigr)&&& \mathrm{if~}\mA^{(r)} \mathrm{~is~directed}\\&&&\\
\displaystyle\sum_{p < q} (u_pv^{(r)}_q + u_qv^{(r)}_p)\bigl(\mA^{(r)}_{pq} - \mP^{(r)}_{pq}\bigr)&&& \mathrm{if~}\mA^{(r)} \mathrm{~is~undirected}\end{array}\right.
\end{align} 
and 
\begin{align}
	\mF_{ij}^{(r)} = \left\{\begin{array}{cccc}\mathbf{0}&&& \mathrm{if~}\mA^{(r)} \mathrm{~is~bipartite}\\&&&\\
\displaystyle\sum_{p=1}^n u_pv^{(r)}_p\mP^{(r)}_{pp}&&& \mathrm{otherwise},\end{array}\right.
\end{align} 
where $u$ and $v^{(r)}$ denote the $i$th and $j$th columns of $\mU_\mP$ and $\mV^{(r)}_\mP$ respectively, so that 
\begin{align}
	\bigl(\mU_\mP^\top(\mA-\mP)\mV_\mP\bigr)_{ij} = \sum_{r=1}^k \mE_{ij}^{(r)} - \sum_{r=1}^k \mF_{ij}^{(r)}.
\end{align}

We can bound the latter term by applying the Cauchy--Schwarz inequality to find that 
\begin{align}
	\Bigl|\sum_{r=1}^k \mF_{ij}^{(r)}\Bigr| \leq \Bigl(\sum_{r=1}^k\sum_{p=1}^n|u_p\mP^{(r)}_{pp}|^2\Bigr)^{1/2} \Bigl(\sum_{r=1}^k\sum_{p=1}^n|v^{(r)}_p|^2\Bigr)^{1/2} = \mathrm{O}(\epsilon\rho),
\end{align} 
and thus can be discounted in our asymptotic analysis.

Each of the $\mE^{(r)}_{ij}$ is a sum of independent zero-mean random variables, with each of the individual terms bounded in absolute value by $|u_pv^{(r)}_q|$ in the bipartite and directed cases and $|u_pv^{(r)}_q+u_qv^{(r)}_p|$ in the undirected case. Applying Hoeffding's inequality, we thus observe that 
\begin{align}
	\mathbb{P}\Biggl(\Bigl|\sum_{r=1}^k\mE^{(r)}_{ij}\Bigr| > t\Biggr) \leq 2\mathrm{exp}\left(\frac{-2t^2}{4\Bigl(\sum\limits_{r_1}\sum\limits_{p, q} |u_pv^{(r_1)}_q|^2 + \sum\limits_{r_3}\sum\limits_{p<q} |u_pv^{(r_2)}_q+u_qv^{(r_2)}_p|^2\Bigr)}\right),
\end{align}
where $r_1$ sums over the bipartite and directed cases and $r_2$ sums over the undirected cases.

Note that $|u_pv^{(r)}_q+u_qv^{(r)}_p|^2 \leq |u_pv^{(r)}_q|^2 + |u_qv^{(r)}_p|^2 + 2|u_pu_q v^{(r)}_qv^{(r)}_p|$, and so 
\begin{align}
	\mathbb{P}\Biggl(\Bigl|\sum_{r=1}^k\mE^{(r)}_{ij}\Bigr| > t\Biggr) \leq 2\mathrm{exp}\left(\frac{-2t^2}{4\Bigl(\sum\limits_{r=1}^k\sum\limits_{p, q} |u_pv^{(r)}_q|^2 +2\sum\limits_{r_2}\sum\limits_{p<q} |u_pu_q v^{(r_2)}_qv^{(r_2)}_p|\Bigr)}\right).
\end{align}

Both summands are at most $1$; the first is clear, while for the second we apply the Cauchy--Schwarz inequality and note that 
\begin{align}
	\sum\limits_{r_2}\sum\limits_{p<q} |u_pu_q v^{(r_2)}_qv^{(r_2)}_p| &\leq \Bigl(\sum\limits_{r_3}\sum\limits_{p<q} |u_pv^{(r_2)}_q|^2\Bigr)^{1/2}\Bigl(\sum\limits_{r_2}\sum\limits_{p<q} |u_qv^{(r_2)}_p|^2\Bigr)^{1/2} \\
&\leq \Bigl(\sum\limits_{r_2}\sum\limits_{p,q} |u_p|^2|v^{(r_2)}_q|^2\Bigr) = 1.
\end{align}

Thus $\sum_{r=1}^k\mE^{(r)}_{ij} = \mathrm{O}\bigl(\log^{1/2}(n)\bigr)$ almost surely, and the result follows after integrating over all possible choices of latent positions. \end{proof}

Before establishing the next set of bounds (which relate to the left and right singular vectors of the matrices $\mA$ and $\mP$), we state the following variation of the Davis--Kahan theorem (see \cite{YWS15}, Theorem 4):

\begin{theorem}[Variant of Davis--Kahan] Let $\mM_1, \mM_2 \in \mathbb{R}^{m \times n}$ have singular value decompositions 
\begin{align}
	\mM_i = \mU_i\mSig_i\mV_i^\top + \mU_{i,\perp}\mSig_{i,\perp}\mV_{i,\perp}^\top,
\end{align}
where $\mU_i \in \mathbb{O}(m \times d)$ has orthonormal columns corresponding to the $d$ greatest singular values of $\mM_i$, for some $1 \leq d \leq n$. Then, if $|\sigma_d(\mM_1)^2 - \sigma_{d+1}(\mM_1)^2| > 0$, we have 
\begin{align}\label{davis_kahan_identity}
	\|\sin \Theta(\mU_2, \mU_1)\| \leq \frac{2\sqrt{d}\left(2\sigma_1(\mM_1)+\|\mM_2-\mM_1\|\right)\|\mM_2-\mM_1\|}{\sigma_{d}(\mM_1)^2-\sigma_{d+1}(\mM_1)^2}.
\end{align} 
where we take $\sigma_{n+1}(\mM_1) = -\infty$. \end{theorem}

We note that the same inequality holds for $\|\sin \Theta(\mV_2, \mV_1)\|$ since the right-hand side of (\ref{davis_kahan_identity}) is invariant under matrix transposition. Using this result, we can prove the following: 

\begin{proposition}\label{alignment_bounds} The following bounds hold almost surely:
\begin{enumerate}[label=\normalfont(\roman*)]
\item $\|\mU_\mA\mU_\mA^\top - \mU_\mP\mU_\mP^\top\|$, $\|\mV_\mA\mV_\mA^\top - \mV_\mP\mV_\mP^\top\| = \mathrm{O}\Bigl(\tfrac{\epsilon^{3/2}k^{1/4}\log^{1/2}(n)}{\rho^{1/2} n^{1/2}}\Bigr)$;
\item $\|\mU_\mA-\mU_\mP\mU_\mP^\top\mU_\mA\|_F$, $\|\mV_\mA-\mV_\mP\mV_\mP^\top\mV_\mA\|_F = \mathrm{O}\Bigl(\tfrac{\epsilon^{3/2}k^{1/4}\log^{1/2}(n)}{\rho^{1/2} n^{1/2}}\Bigr)$;
\item $\|\mU_\mP^\top\mU_\mA\mSig_\mA-\mSig_\mP\mV_\mP^\top\mV_\mA\|_F$, $\|\mSig_\mP\mU_\mP^\top\mU_\mA - \mV_\mP^\top\mV_\mA\mSig_\mA\|_F = \mathrm{O}\bigl(\epsilon^2 k^{1/2}\log(n)\bigr)$; 
\item $\|\mU_\mP^\top\mU_\mA - \mV_\mP^\top\mV_\mA\|_F = \mathrm{O}\Bigl(\tfrac{\epsilon^2 k^{1/2}\log(n)}{\rho n}\Bigr)$
\end{enumerate}
\end{proposition}

\begin{proof}
\leavevmode
\begin{enumerate}[label=\normalfont(\roman*)]
\item Let $\sigma_1, \ldots, \sigma_d$ denote the singular values of $\mU_\mP^\top\mU_\mA$, and let $\theta_i = \cos^{-1}(\sigma_i)$ be the principal angles. It is a standard result that the non-zero eigenvalues of the matrix $\mU_\mA\mU_\mA^\top - \mU_\mP\mU_\mP^\top$ are precisely the $\sin(\theta_i)$ (each occurring twice) and so by Davis--Kahan we have 
\begin{align}
	\|\mU_\mA\mU_\mA^\top - \mU_\mP\mU_\mP^\top\| = \max_{i \in \{1,\ldots,d\}}|\sin(\theta_i)| \leq \frac{2\sqrt{d}\left(2\sigma_1(\mP)+\|\mA-\mP\|\right)\|\mA-\mP\|}{\sigma_d(\mP)^2}
\end{align}
 for $n$ sufficiently large.

Applying the bounds from Propositions \ref{P_bounds} and \ref{spectral_bound} then shows that 
\begin{align}
	\|\mU_\mA\mU_\mA^\top - \mU_\mP\mU_\mP^\top\| = \mathrm{O}\Bigl(\tfrac{\epsilon^{3/2}k^{1/4}\log^{1/2}(n)}{\rho^{1/2} n^{1/2}}\Bigr).
\end{align} 
An identical argument gives the result for $\|\mV_\mA\mV_\mA^\top - \mV_\mP\mV_\mP^\top\|$.

\item Using the bound from part (i), we find that 
\begin{align}
	\|\mU_\mA-\mU_\mP\mU_\mP^\top\mU_\mA\|_F = \|(\mU_\mA\mU_\mA^\top-\mU_\mP\mU_\mP^\top)\mU_\mA\|_F = \mathrm{O}\Bigl(\tfrac{\epsilon^{3/2}k^{1/4}\log^{1/2}(n)}{\rho^{1/2} n^{1/2}}\Bigr).
\end{align} 
An identical argument bounds the term $\|\mV_\mA-\mV_\mP\mV_\mP^\top\mV_\mA\|_F$.

\item Observe that 
\begin{align}
	\mU_\mP^\top\mU_\mA\mSig_\mA-\mSig_\mP\mV_\mP^\top\mV_\mA = \mU_\mP^\top(\mA-\mP)\mV_\mA
\end{align}
and that we may rewrite the right-hand term to find that
\begin{align}
\begin{aligned}
	\mU_\mP^\top\mU_\mA\mSig_\mA-\mSig_\mP\mV_\mP^\top\mV_\mA &= \mU_\mP^\top(\mA-\mP)(\mV_\mA - \mV_\mP\mV_\mP^\top\mV_\mA) \\
	&\phantom{=}+ \mU_\mP^\top(\mA-\mP)\mV_\mP\mV_\mP^\top\mV_\mA.
\end{aligned}
\end{align} 

These terms satisfy 
\begin{align}
	\|\mU_\mP^\top(\mA-\mP)(\mV_\mA - \mV_\mP\mV_\mP^\top\mV_\mA)\|_F = \mathrm{O}\bigl(\epsilon^2 k^{1/2}\log(n)\bigr)
\end{align} 
and 
\begin{align}
	\mU_\mP^\top(\mA-\mP)\mV_\mP\mV_\mP^\top\mV_\mA = \mathrm{O}\bigl(\log^{1/2}(n)\bigr)
\end{align}
by Propositions \ref{spectral_bound}, \ref{orth_spectral_bound} and the result from part (ii), and thus 
\begin{align}
	\|\mU_\mP^\top\mU_\mA\mSig_\mA-\mSig_\mP\mV_\mP^\top\mV_\mA\|_F = \mathrm{O}\bigl(\epsilon^2 k^{1/2}\log(n)\bigr).
\end{align}

An identical argument bounds the term $\|\mSig_\mP\mU_\mP^\top\mU_\mA - \mV_\mP^\top\mV_\mA\mSig_\mA\|_F$.

\item Note that 
\begin{align}
\begin{aligned}
	\mU_\mP^\top\mU_\mA - \mV_\mP^\top\mV_\mA &= \bigl((\mU_\mP^\top\mU_\mA\mSig_\mA-\mSig_\mP\mV_\mP^\top\mV_\mA) + (\mSig_\mP\mU_\mP^\top\mU_\mA \\
	&\phantom{=}- \mV_\mP^\top\mV_\mA\mSig_\mA)\bigr)\mSig_\mA^{-1} - \mSig_\mP(\mU_\mP^\top\mU_\mA- \mV_\mP^\top\mV_\mA)\mSig_\mA^{-1}.
\end{aligned}	
\end{align}

 For any $i, j$ we find (after rearranging and bounding the absolute value of the right-hand terms by the Frobenius norm):

\begin{align}
\begin{aligned}
	\bigl|(\mU_\mP^\top\mU_\mA - \mV_\mP^\top\mV_\mA)_{ij}\bigr|\Bigl(1 + \tfrac{\sigma_i(\mP)}{\sigma_j(\mA)}\Bigr) &\leq \bigl(\|\mU_\mP^\top\mU_\mA\mSig_\mA-\mSig_\mP\mV_\mP^\top\mV_\mA\|_F \\
	&\phantom{=}+ \|\mSig_\mP\mU_\mP^\top\mU_\mA - \mV_\mP^\top\mV_\mA\mSig_\mA\|_F\bigr)\|\mSig_\mA^{-1}\|_F
\end{aligned}
\end{align} 
where we have used the result from part (iii) and Corollary \ref{A_bounds}.

Consequently, we find that 
\begin{align}
	\bigl|(\mU_\mP^\top\mU_\mA - \mV_\mP^\top\mV_\mA)_{ij}\bigr| = \mathrm{O}\Bigl(\tfrac{\epsilon^2 k^{1/2}\log(n)}{\rho n}\Bigr)
\end{align}
by noting that $\Bigl(1 + \tfrac{\sigma_i(\mP)}{\sigma_j(\mA)}\Bigr) \geq 1$. 
\end{enumerate}
\end{proof}

The following result (an analogue of \cite{LTAPP17}, Proposition 16) relates to orthogonal matrix $\mW$ used to perform a simultaneous Procrustes alignment of $\mX_\mA$ with $\mX_\mP$ and $\mY_\mA$ with $\mY_\mP$.

\begin{proposition}\label{procrustes_bound} Let $\mU_\mP^\top\mU_\mA + \mV_\mP^\top\mV_\mA$ admit the singular value decomposition 
\begin{align}
	\mU_\mP^\top\mU_\mA + \mV_\mP^\top\mV_\mA = \mW_1\mSig\mW_2^\top,
\end{align} 
and let $\mW = \mW_1\mW_2^\top$. Then 
\begin{align}
	\max\bigl\{\|\mU_\mP^\top\mU_\mA - \mW\|_F, ~\|\mV_\mP^\top\mV_\mA - \mW\|_F\bigr\} = \mathrm{O}\Bigl(\tfrac{\epsilon^3 k^{1/2}\log(n)}{\rho n}\Bigr)
\end{align}
almost surely.\end{proposition}

\begin{proof} A standard argument shows that $\mW$ minimises the term $\|\mU_\mP^\top\mU_\mA - \mQ\|_F^2 + \|\mV_\mP^\top\mV_\mA - \mQ\|_F^2$ among all $\mQ \in \mathbb{O}(d)$. Let $\mU_\mP^\top\mU_\mA = \mW_{\mU,1}\mSig_\mU\mW_{\mU,2}^\top$ be the singular value decomposition of $\mU_\mP^\top\mU_\mA$, and define $\mW_\mU \in \mathbb{O}(d)$ by $\mW_\mU = \mW_{\mU,1}\mW_{\mU,2}^\top$. Then 
\begin{align}
\begin{aligned}
	\|\mU_\mP^\top \mU_\mA - \mW_\mU\|_F &= \|\mSig - \mI\|_F = \Bigl(\sum_{i=1}^{d} (1-\sigma_i)^2\Bigr)^{1/2} \leq \sum_{i=1}^{d} (1-\sigma_i) \\
	&\leq \sum_{i=1}^{d}(1-\sigma_i^2) = \sum_{i=1}^{d} \sin^2(\theta_i) \leq d\|\mU_\mA\mU_\mA^\top - \mU_\mP\mU_\mP^\top\|^2 \\
\end{aligned}
\end{align}
and so 
\begin{align}
	\|\mU_\mP^\top \mU_\mA - \mW_\mU\|_F = \mathrm{O}\Bigl(\tfrac{\epsilon^3 k^{1/2}\log(n)}{\rho n}\Bigr).
\end{align}

Also, 
\begin{align}
	\|\mV_\mP^\top \mV_\mA - \mW_\mU\|_F \leq \|\mV_\mP^\top \mV_\mA - \mU_\mP^\top\mU_\mA\|_F + \|\mU_\mP^\top\mU_\mA - \mW_\mU\|_F
\end{align}
	and so 
\begin{align}
	\|\mV_\mP^\top \mV_\mA - \mW_\mU\|_F = \mathrm{O}\Bigl(\tfrac{\epsilon^3k^{1/2}\log(n)}{\rho n}\Bigr)
\end{align} 
by Proposition \ref{alignment_bounds}.

Combining these shows that 
\begin{align}
	\|\mU_\mP^\top\mU_\mA - \mW\|_F^2 + \|\mV_\mP^\top\mV_\mA - \mW\|_F^2 \leq \|\mU_\mP^\top\mU_\mA - \mW_\mU\|_F^2 + \|\mV_\mP^\top\mV_\mA - \mW_\mU\|_F^2 
\end{align}
and thus
\begin{align}
	\max\bigl\{\|\mU_\mP^\top\mU_\mA - \mW\|_F, ~\|\mV_\mP^\top\mV_\mA - \mW\|_F\bigr\} = \mathrm{O}\Bigl(\tfrac{\epsilon^3 k^{1/2}\log(n)}{\rho n}\Bigr)
\end{align} as required. 
\end{proof}

The following bounds are a straightforward adaptation of \cite{LTAPP17}, Lemma 17:

\begin{proposition}\label{W_bounds} The following bounds hold almost surely:
\begin{enumerate}[label=\normalfont(\roman*)]
	\item $\|\mW\mSig_\mA - \mSig_\mP\mW\|_F = \mathrm{O}\bigl(\epsilon^4 k^{1/2}\log(n)\bigr)$;
	\item $\|\mW\mSig_\mA^{1/2} - \mSig_\mP^{1/2}\mW\|_F = \mathrm{O}\Bigl(\tfrac{\epsilon^4 k^{1/2}\log(n)}{\rho^{1/2} n^{1/2}}\Bigr)$;
	\item $\|\mW\mSig_\mA^{-1/2} - \mSig_\mP^{-1/2}\mW\|_F = \mathrm{O}\Bigl(\tfrac{\epsilon^4 k^{1/2}\log(n)}{\rho^{3/2} n^{3/2}}\Bigr)$.
\end{enumerate}
\end{proposition}

\begin{proof}
\leavevmode
\begin{enumerate}[label=\normalfont(\roman*)]
	\item Observe that
\begin{align}
	\mW\mSig_\mA - \mSig_\mP\mW = (\mW-\mU_\mP^\top\mU_\mA)\mSig_\mA + \mU_\mP^\top\mU_\mA\mSig_\mA - \mSig_\mP\mW
\end{align}
and that the right-hand expression may be rewritten as
\begin{align}
	(\mW-\mU_\mP^\top\mU_\mA)\mSig_\mA + (\mU_\mP^\top\mU_\mA\mSig_\mA - \mSig_\mP\mV_\mP^\top\mV_\mA) + \mSig_\mP(\mV_\mP^\top\mV_\mA - \mW).
\end{align}

The terms $\|(\mW-\mU_\mP^\top\mU_\mA)\mSig_\mA\|_F$ and $\|\mSig_\mP(\mV_\mP^\top\mV_\mA-\mW)\|_F$ are both $\mathrm{O}\bigl(\epsilon^4 k^{1/2}\log(n)\bigr)$ (as shown by Propositions \ref{P_bounds}, \ref{procrustes_bound} and Corollary \ref{A_bounds}), while the term $\|\mU_\mP^\top\mU_\mA\mSig_\mA - \mSig_\mP\mV_\mP^\top\mV_\mA\|_F$ is $\mathrm{O}\bigl(\epsilon^2 k^{1/2}\log(n)\bigr)$ by Proposition \ref{alignment_bounds}, and so $\|\mW\mSig_\mA - \mSig_\mP\mW\|_F = \mathrm{O}\bigl(\epsilon^4 k^{1/2}\log(n)\bigr)$ as required.

\item Note that 
\begingroup
\addtolength{\jot}{0.6em}
\begin{align}
	\bigl(\mW\mSig_\mA^{1/2}-\mSig_\mP^{1/2} \mW\bigr)_{ij} &= \mW_{ij}\bigl(\sigma_j(\mA)^{1/2} - \sigma_i(\mP)^{1/2}\bigr) \\
	&= \frac{\mW_{ij}(\sigma_j(\mA) - \sigma_i(\mP))}{\sigma_j(\mA)^{1/2} + \sigma_i(\mP)^{1/2}} \\
	&= \frac{\bigl(\mW\mSig_\mA - \mSig_\mP\mW\bigr)_{ij}}{\sigma_j(\mA)^{1/2} + \sigma_i(\mP)^{1/2}},
\end{align}
\endgroup
and so we find that $\|\mW\mSig_\mA^{1/2} - \mSig_\mP^{1/2}\mW\|_F = \mathrm{O}\Bigl(\tfrac{\epsilon^4 k^{1/2}\log(n)}{\rho^{1/2} n^{1/2}}\Bigr)$ by applying part (i) and summing over all $i, j \in \{1,\ldots,d\}$.

\item Note that 
\begingroup
\addtolength{\jot}{0.6em}
\begin{align}
	\bigl(\mW\mSig_\mA^{-1/2}-\mSig_\mP^{-1/2} \mW\bigr)_{ij} &= \frac{\mW_{ij}(\sigma_i(\mP)^{1/2} - \sigma_j(\mA)^{1/2})}{\sigma_i(\mP)^{1/2}\sigma_j(\mA)^{1/2}} \\
	&= \frac{\bigl(\mW\mSig_\mA^{1/2}-\mSig_\mP^{1/2} \mW\bigr)_{ij}}{\sigma_i(\mP)^{1/2}\sigma_j(\mA)^{1/2}}
\end{align}
\endgroup
and so we find that $\|\mW\mSig_\mA^{-1/2} - \mSig_\mP^{-1/2}\mW\|_F = \mathrm{O}\Bigl(\tfrac{\epsilon^4 k^{1/2}\log(n)}{\rho^{3/2} n^{3/2}}\Bigr)$ by applying part (ii) and summing over all $i, j \in \{1, \ldots, d\}$. 
\end{enumerate} 
\end{proof}

The next results establish the existence of, and some properties relating to, the matrices $\widetilde{\mL}$ and $\widetilde{\mR}_r$ which map the latent position matrices $\mX$ and $\mY^{(r)}$ to the embeddings $\mX_\mP$ and $\mY^{(r)}_\mP$ respectively. In particular, where they exist, we bound the growth of the spectral norm of the inverses (or pseudo-inverses where appropriate) of these matrices, which is necessary in order to be able to recover the latent positions from the UASE.

\begin{proposition}\label{L_matrices} If $\mX$ is of rank $d$ then there exists a matrix $\widetilde{\mL} \in \mathrm{GL}(d)$ such that $\mX_\mP = \mX\widetilde{\mL}$. In addition, if $\mY^{(r)}$ is of rank $d_r$ then $\mY^{(r)}_\mP = \mY^{(r)}\widetilde{\mR}_r$, where the matrix $\widetilde{\mR}_r \in \mathbb{R}^{d_r \times d}$ satisfies $\widetilde{\mL}\widetilde{\mR}_r^\top = \mLam_{\epsilon,r}$.  In particular, $\mathrm{rank}(\widetilde{\mR}_r)= \mathrm{rank}(\mLam_r)$.\end{proposition}

\begin{proof} Let $\mPi_\mX, \mPi_\mY \in \mathrm{GL}(d)$ satisfy $\mPi_\mX = (\mX^\top\mX)^{1/2}$ and $\mPi_\mY = (\mLam_\epsilon\mY^\top\mY\mLam_\epsilon^\top)^{1/2}$, where we take the unique positive-definite matrix square root in each case.

Observe that 
\begin{align}
	\bigl(\mX_\mP\mSig_\mP^{1/2}\bigr)\bigl(\mX_\mP\mSig_\mP^{1/2}\bigr)^\top = \mU_\mP\mSig_\mP^2\mU_\mP^\top = \mP\mP^\top = (\mX\mPi_\mY)(\mX\mPi_\mY)^\top
\end{align} 
and similarly (noting that $\mU_\mP\mSig_\mP{\mV^{(r)}_\mP}^\top = \mP^{(r)}$) that
\begin{align}
	\bigl(\mY^{(r)}_\mP\mSig_\mP^{1/2}\bigr)\bigl(\mY^{(r)}_\mP\mSig_\mP^{1/2}\bigr)^\top = {\mP^{(r)}}^\top\mP^{(r)} = \bigl(\mY^{(r)}\mLam_{\epsilon,r}^\top\mPi_\mX\bigr)\bigl(\mY^{(r)}\mLam_{\epsilon,r}^\top\mPi_\mX\bigr)^\top.
\end{align}

Thus there exist orthogonal matrices $\mQ_\mP, \mQ^{(r)}_\mP \in \mathbb{O}(d)$ such that 
\begin{align}
	\mX_\mP\mSig_\mP^{1/2} = \mX\mPi_\mY\mQ_\mP , \quad \mY^{(r)}_\mP\mSig_\mP^{1/2} = \mY^{(r)}\mLam_{\epsilon,r}^\top\mPi_\mX\mQ^{(r)}_\mP
\end{align} 
and so 
\begin{align}
	\widetilde{\mL} = \mPi_\mY\mQ_\mP\mSig_\mP^{-1/2} \in \mathrm{GL}(d) , \quad \widetilde{\mR}_r = \mLam_{\epsilon,r}^\top\mPi_\mX\mQ^{(r)}_\mP\mSig_\mP^{-1/2} \in \mathbb{R}^{d_r \times d}
\end{align}
are our desired matrices. 

For the final statement, observe that 
\begin{align}
	\mX\widetilde{\mL}\widetilde{\mR}_r^\top{\mY^{(r)}}^\top = \mX_\mP{\mY^{(r)}_\mP}^\top = \mP^{(r)} = \mX\mLam_{\epsilon,r}{\mY^{(r)}}^\top,
\end{align} 
and so the result follows after multiplying by $(\mX^\top\mX)^{-1}\mX^\top$ and $\mY^{(r)}({\mY^{(r)}}^\top\mY^{(r)})^{-1}$ on the left and right respectively. 
\end{proof}

\begin{corollary}\label{L_matrix_norm} The matrices $\widetilde{\mL}$ and $\widetilde{\mR}_r$ satisfy $\|\widetilde{\mL}\| = \mathrm{O}(\epsilon)$, $\|\widetilde{\mL}^{-1}\| = \mathrm{O}(\epsilon^{1/2})$ and $\|\widetilde{\mR}_r\| = \mathrm{O}(\epsilon_r)$ almost surely. 

 Moreover, if the matrix $\mLam_r$ is of rank $d_r$, then $\|\widetilde{\mR}_r^+\| = \mathrm{O}\bigl(\tfrac{1}{\epsilon_r}\bigr)$ almost surely, where $\widetilde{\mR}_r^+ = \widetilde{\mR}_r^\top(\widetilde{\mR}_r\widetilde{\mR}_r^\top)^{-1}$ is the Moore-Penrose inverse of $\widetilde{\mR}_r$.
\end{corollary}

\begin{proof} Proposition \ref{P_bounds} shows us that $\|\mSig_\mP\| = \mathrm{O}(\epsilon\rho n)$ and $\|\mSig_\mP^{-1}\| = \mathrm{O}\bigl(\tfrac{1}{\rho n}\bigr)$, and an identical line of reasoning shows that $\|\Pi_{\mLam_\epsilon,\mY}\| = \mathrm{O}(\epsilon\rho^{1/2} n^{1/2})$ and $\|\Pi_{\mLam_\epsilon,\mY}^{-1}\| = \mathrm{O}\bigl(\tfrac{1}{\rho^{1/2} n^{1/2}}\bigr)$, and similarly that $\|\Pi_\mX\| = \mathrm{O}(\rho^{1/2} n^{1/2})$ and $\|\Pi_\mX^{-1}\| = \mathrm{O}\bigl(\tfrac{1}{\rho^{1/2}n^{1/2}}\bigr)$.

 The first three bounds then follow from submultiplicativity and unitary invariance of the spectral norm. For the final result, note that $\widetilde{\mR}_r\widetilde{\mR}_r^\top = \mLam_{\epsilon,r}^\top\mPi_\mX\mQ^{(r)}_\mP\mSig_\mP^{-1}{\mQ^{(r)}_\mP}^\top\mPi_\mX^\top\mLam_{\epsilon,r}$, and so we see that 
\begin{align}
	\sigma_{d_r}\bigl(\widetilde{\mR}_r\widetilde{\mR}_r^\top\bigr) &= \lambda_{d_r}\bigl({\mQ^{(r)}_\mP}^\top\mPi_\mX^\top\mLam_{\epsilon,r}\mLam_{\epsilon,r}^\top\mPi_\mX\mQ^{(r)}_\mP\mSig_\mP^{-1}\bigr) \\
	&\geq \sigma_1(\mP)^{-1}\lambda_{d_r}\bigl({\mQ^{(r)}_\mP}^\top\mPi_\mX^\top\mLam_{\epsilon,r}\mLam_{\epsilon,r}^\top\mPi_\mX\mQ^{(r)}_\mP\bigr)
\end{align} 
by a standard application of the min-max theorem for eigenvalues of Hermitian matrices.

Now, 
\begin{align}
	\lambda_{d_r}\bigl({\mQ^{(r)}_\mP}^\top\mPi_\mX^\top\mLam_{\epsilon,r}\mLam_{\epsilon,r}^\top\mPi_\mX\mQ^{(r)}_\mP\bigr) = \lambda_{d_r}\bigl(\mLam_{\epsilon,r}^\top\mPi_\mX\mPi_\mX^\top\mLam_{\epsilon,r}\bigr) = \lambda_{d_r}\bigl(\mLam_{\epsilon,r}^\top\mX^\top\mX\mLam_{\epsilon,r}\bigr),
\end{align} 
and a similar line of reasoning to that in Proposition \ref{P_bounds} shows that $\sigma_{d_r}\bigl(\widetilde{\mR}_r\widetilde{\mR}_r^\top\bigr) = \Omega(\epsilon_r^2)$ almost surely. The result then follows from submultiplicativity and unitary invariance of the spectral norm. 
\end{proof}

\begin{proposition}\label{L_matrix_identities} If $\mX$ is of rank $d$ and each $\mY^{(r)}$ is of rank $d_r$ then 
\begin{align}
	(\widetilde{\mR}_r\mSig_\mP^{-1}\widetilde{\mL}^{-1})^\top = (\mLam_\epsilon\mY^\top\mY\mLam_\epsilon^\top)^{-1}\mLam_{\epsilon,r}.
\end{align} 
Moreover, if $d_r = d$ and the matrix $\mLam_r$ is invertible, then 
\begin{align}
	(\widetilde{\mL}\mSig_\mP^{-1}\widetilde{\mR}_r^{-1})^\top = \mLam_{\epsilon,r}^{-1}(\mX^\top\mX)^{-1}.
\end{align}
\end{proposition}

\begin{proof} Recall from Proposition \ref{L_matrices} that $\widetilde{\mL}\widetilde{\mR}_r^\top = \mLam_{\epsilon,r}$. Similarly, since 
\begin{align}
	\mX\widetilde{\mL}\mSig_\mP\widetilde{\mL}^\top\mX^\top = \mX_\mP\mSig_\mP\mX_\mP^\top = \mP\mP^\top = \mX\mLam_\epsilon\mY^\top\mY\mLam_\epsilon^\top\mX^\top,
\end{align} 
we find that $\widetilde{\mL}\mSig_\mP\widetilde{\mL}^\top = \mLam_\epsilon\mY^\top\mY\mLam_\epsilon^\top$. Thus 
\begin{align}
	(\widetilde{\mR}_r\mSig_\mP^{-1}\widetilde{\mL}^{-1})^\top &= \widetilde{\mL}^{-\top}\mSig_\mP^{-1}\widetilde{\mL}^{-1}\widetilde{\mL}\widetilde{\mR}_r^\top \\
	& = (\mLam_\epsilon\mY^\top\mY\mLam_\epsilon^\top)^{-1}\mLam_{\epsilon,r}
\end{align} 
and 
\begin{align}
	(\widetilde{\mL}\mSig_\mP^{-1}\widetilde{\mR}_r^{-1})^\top &= \widetilde{\mR}_r^{-\top}\mSig_\mP^{-1}\mX_\mP^\top\mX(\mX^\top\mX)^{-1} \\
	&= \widetilde{\mR}_r^{-\top}\mSig_\mP^{-1}\mX_\mP^\top\mX_\mP\widetilde{\mL}^{-1}(\mX^\top\mX)^{-1} \\
	&= \widetilde{\mR}_r^{-\top}\mSig_\mP^{-1}\mSig_\mP\widetilde{\mL}^{-1}(\mX^\top\mX)^{-1} \\
	&= \widetilde{\mR}_r^{-\top}\widetilde{\mL}^{-1}(\mX^\top\mX)^{-1} \\
	&= \mLam_{\epsilon,r}^{-1}(\mX^\top\mX)^{-1},
\end{align} 
where we have used the identities $\widetilde{\mL} = (\mX^\top\mX)^{-1}\mX^\top\mX_\mP$ and $\mX_\mP^\top\mX_\mP = \mSig_\mP$. 
\end{proof}

The final result before proving our main theorems is an adaptation of Lemma 4 in \cite{AFTPPVLLQ17} to the MRDPG, and utilises the previous results to establish upper bounds for the two-to-infinity norms of a number of residual terms that will appear in the proofs of Theorems \ref{consistency} and \ref{CLT}: 

\begin{proposition}\label{residual_bounds} Let 
\begin{gather}
	\mR_{1,1} = \mU_\mP(\mU_\mP^\top\mU_\mA\mSig_\mA^{1/2} - \mSig_\mP^{1/2} \mW) \\
	\mR_{1,2} = (\mI - \mU_\mP\mU_\mP^\top)(\mA-\mP)(\mV_\mA - \mV_\mP\mW)\mSig_\mA^{-1/2} \\
	\mR_{1,3} = -\mU_\mP\mU_\mP^\top(\mA-\mP)\mV_\mP\mW\mSig_\mA^{-1/2} \\
	\mR_{1,4} = (\mA-\mP)\mV_\mP(\mW\mSig_\mA^{-1/2} - \mSig_\mP^{-1/2}\mW)
\end{gather}
and 
\begin{gather}
	\mR_{2,1} = \mV_\mP(\mV_\mP^\top\mV_\mA\mSig_\mA^{1/2} - \mSig_\mP^{1/2}\mW) \\
	\mR_{2,2} = (\mI - \mV_\mP\mV_\mP^\top)(\mA-\mP)^\top(\mU_\mA - \mU_\mP\mW)\mSig_\mA^{-1/2} \\
	\mR_{2,3} = -\mV_\mP\mV_\mP^\top(\mA-\mP)^\top\mU_\mP\mW\mSig_\mA^{-1/2} \\
	\mR_{2,4} = (\mA-\mP)^\top\mU_\mP(\mW\mSig_\mA^{-1/2} - \mSig_\mP^{-1/2}\mW)
\end{gather}
Then the following bounds hold almost surely:
\begin{enumerate}[label=\normalfont(\roman*)]
	\item $\|\mR_{1,1}\|_{2 \to \infty} = \mathrm{O}\left(\tfrac{\epsilon^5 k^{1/2}\log(n)}{\rho^{1/2}n}\right)$ and $\|\mR_{2,1}\|_{2 \to \infty} = \mathrm{O}\left(\tfrac{\epsilon^4 k^{1/2}\log(n)}{\rho^{1/2}n}\right)$;
	\item $\|\mR_{1,2}\|_{2 \to \infty}, \|\mR_{2,2}\|_{2 \to \infty} = \mathrm{O}\left(\tfrac{\epsilon^2 k^{1/2}\log(n)}{\rho^{1/2} n^{3/4}}\right)$;
	\item $\|\mR_{1,3}\|_{2 \to \infty} = \mathrm{O}\left(\tfrac{\epsilon \log^{1/2}(n)}{\rho^{1/2}n}\right)$ and $\|\mR_{2,3}\|_{2 \to \infty} = \mathrm{O}\left(\tfrac{ \log^{1/2}(n)}{\rho^{1/2}n}\right)$;
	\item $\|\mR_{1,4}\|_{2 \to \infty}, \|\mR_{2,4}\|_{2 \to \infty} = \mathrm{O}\left(\tfrac{\epsilon^{9/2}k^{3/4}\log^{3/2}(n)}{n}\right)$.
\end{enumerate}
\end{proposition}

\begin{proof} We give full proofs of the bounds only for the terms $\mR_{1,i}$, noting any differences for the proofs for the terms $\mR_{2,i}$. 
\begin{enumerate}[label=\normalfont(\roman*)]
	\item Recall that $\mU_\mP\mSig_\mP^{1/2} = \mX\widetilde{\mL}$, where the matrix $\widetilde{\mL} \in \mathrm{GL}(d)$ satisfies $\|\widetilde{\mL}\| = \mathrm{O}(\epsilon)$ by Corollary \ref{L_matrix_norm}. Using the relation $\|\mA\mB\|_{2 \to \infty} \leq \|\mA\|_{2 \to \infty} \|\mB\|$ (see, for example, \cite{CTP17}, Proposition 6.5) we find that $\|\mU_\mP\|_{2 \to \infty} \leq \|\mX\|_{2 \to \infty}\|\widetilde{\mL}\|\|\mSig_\mP^{-1/2}\|$, and thus $\|\mU_\mP\|_{2 \to \infty} = \mathrm{O}\bigl(\tfrac{\epsilon}{n^{1/2}}\bigr)$ as the rows of $\mX$ are by definition of order $\mathrm{O}(\rho^{1/2})$ (similarly, we find that $\|\mV_\mP\|_{2 \to \infty} = \mathrm{O}\bigl(\tfrac{1}{n^{1/2}}\bigr)$ by splitting $\mV_\mP\mSig_\mP^{1/2}$ into the separate terms $\mV^{(r)}_\mP\mSig_\mP^{1/2}$ and evaluating each separately, and noting that $\epsilon_r \leq 1$ for all $r$).

Thus 
\begin{align}
	\|\mR_{1,1}\|_{2 \to \infty} &\leq \|\mU_\mP\|_{2 \to \infty}\|\mU_\mP^\top\mU_\mA\mSig_\mA^{1/2} - \mSig_\mP^{1/2} \mW\| \\
	& \leq \|\mU_\mP\|_{2\to\infty}\bigl(\|(\mU_\mP^\top\mU_\mA - \mW)\mSig_\mA^{1/2}\|_F + \|\mW\mSig_\mA^{1/2}-\mSig_\mP^{1/2} \mW\|_F\bigr)
\end{align}
 The first summand is $\mathrm{O}\Bigl(\tfrac{\epsilon^{7/2}k^{1/2}\log(n)}{\rho^{1/2} n^{1/2}}\Bigr)$ by Proposition \ref{procrustes_bound} and Corollary \ref{A_bounds}, while Proposition \ref{W_bounds} shows that the second is $\mathrm{O}\Bigl(\tfrac{\epsilon^4 k^{1/2}\log(n)}{\rho^{1/2} n^{1/2}}\Bigr)$, and so 
\begin{align}
	\|\mR_{1,1}\|_{2 \to \infty} = \mathrm{O}\Bigl(\tfrac{\epsilon^5 k^{1/2}\log(n)}{\rho^{1/2}n}\Bigr).
\end{align}

\item We begin by splitting the term $\mR_{1,2} = \mM_1 + \mM_2$, where 
\begin{align}
	&\mM_1 = \mU_\mP\mU_\mP^\top(\mA-\mP)(\mV_\mA-\mV_\mP\mW)\mSig_\mA^{-1/2} \\
	&\mM_2 = (\mA-\mP)(\mV_\mA-\mV_\mP\mW)\mSig_\mA^{-1/2}
\end{align}

Now,
\begin{align}
	\|\mM_1\|_{2 \to \infty} &\leq \|\mU_\mP\|_{2\to \infty}\|\mA-\mP\|\|\mV_\mA-\mV_\mP\mW\|\|\mSig_\mA^{-1/2}\|,
\end{align} 
where the term 
\begin{align}
	\|\mU_\mP\|_{2\to \infty}\|\mA-\mP\|\|\mSig_\mA^{-1/2}\| = \mathrm{O}\Bigl(\tfrac{\epsilon^{3/2}k^{1/4}\log^{1/2}(n)}{n^{1/2}}\Bigr)
\end{align}
by Proposition \ref{spectral_bound} and Corollary \ref{A_bounds}, while 
\begin{align}
	\|\mV_\mA-\mV_\mP\mW\| &\leq \|\mV_\mA-\mV_\mP\mV_\mP^\top\mV_\mA\| + \|\mV_\mP(\mV_\mP^\top\mV_\mA-\mW)\| \\
	&= \mathrm{O}\Bigl(\tfrac{\epsilon^{3/2}k^{1/4}\log^{1/2}(n)}{\rho^{1/2} n^{1/2}} \Bigr)
\end{align} 
by Propositions \ref{alignment_bounds} and \ref{procrustes_bound} and the asymptotic growth conditions imposed on $\rho$. Thus 
\begin{align}
	\|\mM_1\|_{2 \to \infty} = \mathrm{O}\Bigl(\tfrac{\epsilon^3 k^{1/2}\log(n)}{\rho^{1/2}n}\Bigr).
\end{align}

Next, note that 
\begin{align}
	\mM_2 = (\mA-\mP)(\mI-\mV_\mP\mV_\mP^\top)\mV_\mA\mSig_\mA^{-1/2} + (\mA-\mP)\mV_\mP(\mV_\mP^\top\mV_\mA-\mW)\mSig_\mA^{-1/2}
\end{align} where 
\begin{align}
	\|(\mA-\mP)\mV_\mP(\mV_\mP^\top\mV_\mA-\mW)\mSig_\mA^{-1/2}\|_{2 \to \infty} &\leq \|\mA-\mP\|\|\mV_\mP^\top\mV_\mA-\mW\|\|\mSig_\mA^{-1/2}\| \\
	& = \mathrm{O}\Bigl(\tfrac{\epsilon^{7/2}k^{3/4}\log^{3/2}(n)}{\rho^{3/2}n}\Bigr)
\end{align} 
by Propositions \ref{spectral_bound}, \ref{procrustes_bound} and Corollary \ref{A_bounds}.

To bound the remaining term, let $\mM = (\mA-\mP)(\mI-\mV_\mP\mV_\mP^\top)\mV_\mA\mV_\mA^\top$, so that	
\begin{align}
	(\mA-\mP)(\mI-\mV_\mP\mV_\mP^\top)\mV_\mA\mSig_\mA^{-1/2} = \mM\mV_\mA\mSig_\mA^{-1/2}
\end{align}
and so
\begin{align}
	\|(\mA-\mP)(\mI-\mV_\mP\mV_\mP^\top)\mV_\mA\mSig_\mA^{-1/2}\|_{2 \to \infty} & \leq \|\mM\|_{2 \to \infty}\|\mV_\mA\mSig_\mA^{-1/2}\|.
\end{align}

 The term $\|\mV_\mA\mSig_\mA^{-1/2}\|$ is $\mathrm{O}\bigl(\tfrac{1}{\rho^{1/2}n^{1/2}}\bigr)$ by Corollary \ref{A_bounds}, so it suffices to bound $\|\mM\|_{2 \to \infty}$. To do this, we claim that the Frobenius norms of the rows of the matrix $\mM$ are exchangeable, and thus have the same expectation, which implies that $\mathbb{E}\bigl(\|\mM\|_F^2\bigr) = n\mathbb{E}\bigl(\|\mM_i\|^2\bigr)$ for any $i \in \{1,\ldots,n\}$. Applying Markov's inequality, we therefore see that 
\begin{align}
	\mathbb{P}\bigl(\|\mM_i\| > t\bigr) \leq \frac{\mathbb{E}\bigl(\|\mM_i\|^2\bigr)}{t^2} = \frac{\mathbb{E}\bigl(\|\mM\|_F^2\bigr)}{nt^2}.
\end{align}

Now, 
\begin{align}
	\|\mM\|_F &\leq \|\mA-\mP\|\|\mV_\mA-\mV_\mP\mV_\mP^\top\mV_\mA\|_F\|\mV_\mA^\top\|_F \\
	&= \mathrm{O}\bigl(\epsilon^2 k^{1/2}\log(n)\bigr)
\end{align} 
by Propositions \ref{spectral_bound} and \ref{alignment_bounds}. It follows that 
\begin{align}
	\mathbb{P}\Bigl(\|\mM_i\| > \tfrac{\epsilon^2 k^{1/2}\log(n)}{n^{1/4}}\Bigr) = \mathrm{O}\bigl(\tfrac{1}{n^{1/2}}\bigr)
\end{align} 
and thus 
\begin{align}
	\|\mM\|_{2 \to \infty} = \mathrm{O}\Bigl(\tfrac{\epsilon^2 k^{1/2}\log(n)}{n^{1/4}}\Bigr)
\end{align} 
and 
\begin{align}
	\|(\mA-\mP)(\mI-\mV_\mP\mV_\mP^\top)\mV_\mA\mSig_\mA^{-1/2}\|_{2 \to \infty} = \mathrm{O}\Bigl(\tfrac{\epsilon^2 k^{1/2}\log(n)}{\rho^{1/2} n^{3/4}}\Bigr)
\end{align} 
almost surely.

We must therefore show that the Frobenius norms of the rows of $\mM$ are exchangeable. Let $\mQ_L \in \mathbb{O}(n)$ and, for each $r \in \{1,\ldots,k\}$, $\mQ_{R,r} \in \mathbb{O}\left(n_r\right)$ be permutation matrices (where we require that $\mQ_{R,r} = \mQ_L$ if $\mY^{(r)} = \mX\mG_r$ with probability one for some matrix $\mG_r \in \mathbb{R}^{d \times d_r}$, and similarly that $\mQ_{R,r} = \mQ_{R,s}$ if $\mY^{(s)} = \mY^{(r)}\mG_{r,s}$ with probability one for some matrix $\mG_{r,s} \in \mathbb{R}^{d_r \times d_s}$), and let $\mQ_R = \mQ_{R,1} \oplus \cdots \oplus \mQ_{R,k}$. For any matrix $\mG$, let $\mathcal{R}_d(\mG)$ denote the projection onto the subspace spanned by the right singular vectors corresponding to the leading $d$ singular values of $\mG$, and let $\mathcal{R}_d^\perp(\mG)$ denote the projections onto the orthogonal complements of this subspace.

Note that 
\begin{align}
	\mathcal{R}_d(\mP) = \mV_\mP\mV_\mP^\top \quad\mathrm{and}\quad \mathcal{R}_d(\mA) = \mV_\mA\mV_\mA^\top,
\end{align} while for any permutation matrices $\mQ_L \in \mathbb{O}(n)$ and $\mQ_{R,r} \in \mathbb{O}(n_r)$ we have 
\begin{align}
\mathcal{R}_d\bigl(\mQ_L\mP\mQ_R^\top\bigr) = \mQ_R\mV_\mP\mV_\mP^\top\mQ_R^\top \quad\mathrm{and}\quad \mathcal{R}_d\bigl(\mQ_L\mA\mQ_R^\top\bigr) = \mQ_R\mV_\mA\mV_\mA^\top\mQ_R^\top.
\end{align}

For any commensurate pair of matrices $\mG$ and $\mH$, define an operator 
\begin{align}
	\mathcal{P}_{\mathcal{R},d}(\mG,\mH) = (\mG-\mH)\mathcal{R}_d^\perp(\mH)\mathcal{R}_d(\mG)
\end{align} and note that $\mathcal{P}_{\mathcal{R},d}(\mA,\mP) = \mM$, while 
\begin{align}
	\mathcal{P}_{\mathcal{R},d}\bigl(\mQ_L\mA\mQ_R^\top,\mQ_L\mP\mQ_R^\top\bigr) &= \mQ_L(\mA-\mP)\mQ_R^\top\mQ_R(\mI-\mV_\mP\mV_\mP^\top)\mQ_R^\top\mQ_R\mV_\mA\mV_\mA^\top\mQ_R^\top \\
	&= \mQ_L\mM\mQ_R^\top.
\end{align}

Our assumptions regarding the distribution of the latent positions ensure that the entries of the pair $(\mA,\mP)$ have the same joint distribution as those of the pair $\bigl(\mQ_L\mA\mQ_R^\top,\mQ_L\mP\mQ_R^\top\bigr)$, since $\mQ_R^\top$ permutes the columns of each $\mA^{(r)}$ and $\mP^{(r)}$ separately. Therefore, the entries of the matrix $\mathcal{P}_{\mathcal{R},d}(\mA,\mP)$ have the same joint distribution as those of the matrix $\mathcal{P}_{\mathcal{R},d}\bigl(\mQ_L\mA\mQ_R^\top,\mQ_L\mP\mQ_R^\top\bigr)$, which implies that $\mM$ has the same distribution as $\mQ_L\mM\mQ_R^\top$, and consequently the Frobenius norms of the rows of $\mM$ have the same distribution as those of $\mQ_L\mM$, which proves our claim.

Combining these results, we see that 
\begin{align}
	\|\mR_{1,2}\|_{2 \to \infty} = \mathrm{O}\Bigl(\tfrac{\epsilon^2 k^{1/2}\log\left(n\right)}{\rho^{1/2} n^{3/4}}\Bigr)
\end{align} 
almost surely, as required.

The proof of the bound for the term $\mR_{2,2}$ follows similarly, and culminates in showing that the term 
\begin{align}
	\mN = (\mA-\mP)^\top\bigl(\mI-\mU_\mP\mU_\mP^\top\bigr)\mU_\mA\mU_\mA^\top
\end{align}
satisfies 
\begin{align}
	\|\mN\|_{2 \to \infty} = \mathrm{O}\Bigl(\tfrac{\epsilon^2 k^{1/2}\log(n)}{n^{1/4}}\Bigr)
\end{align} 
almost surely. The matrix $\mN \in \mathbb{R}^{(n_1 + \ldots + n_k) \times n}$ splits into $k$ distinct matrices $\mN^{(r)} \in \mathbb{R}^{n_r \times n}$. This time, we must show that the Frobenius norms of the rows of each $\mN^{(r)}$ are interchangeable, from which a similar argument to our previous one will allow us to derive our desired bound for $\|\mN\|_{2 \to \infty}$.

Note that for any $r \in \{1,\ldots,k\}$ 
\begin{align}
	\mathcal{R}_d\bigl({\mP^{(r)}}^\top\bigr) = \mU_\mP\mU_\mP^\top \quad\mathrm{and}\quad \mathcal{R}_d\bigl({\mA^{(r)}}^\top\bigr) = \mU_\mA\mU_\mA^\top,
\end{align} 
while for any permutation matrices $\mQ_L \in \mathbb{O}(n_r)$ and $\mQ_R \in \mathbb{O}(n)$ we have 
\begin{align}
	\mathcal{R}_d\bigl(\mQ_L{\mP^{(r)}}^\top\mQ_R^\top\bigr) = \mQ_R\mU_\mP\mU_\mP^\top\mQ_R^\top \quad\mathrm{and}\quad \mathcal{R}_d\bigl(\mQ_L{\mA^{(r)}}^\top\mQ_R^\top\bigr) = \mQ_R\mV_\mA\mV_\mA^\top\mQ_R^\top.
\end{align}

Also, $\mathcal{P}_{\mathcal{R},d}\bigl({\mA^{(r)}}^\top,{\mP^{(r)}}^\top\bigr) = \mN^{(r)}$, while we find that
\begin{align}
	\mathcal{P}_{\mathcal{R},d}\bigl(\mQ_L{\mA^{(r)}}^\top\mQ_R^\top,\mQ_L{\mP^{(r)}}^\top\mQ_R^\top\bigr) = \mQ_L\mN^{(r)}\mQ_R^\top,
\end{align} 
and so it follows from a similar argument to before that the rows of $\mN^{(r)}$ are interchangeable, and thus we obtain our desired bound.

\item Similarly to part (i), we see that
\begin{align}
	\|\mR_{1,3}\|_{2 \to \infty} &\leq \|\mU_\mP\|_{2 \to \infty}\|\mU_\mP^\top(\mA-\mP)\mV_\mP\mW\mSig_\mA^{-1/2}\| \\
	& \leq \|\mU_\mP\|_{2 \to \infty}\|\mU_\mP^\top(\mA-\mP)\mV_\mP\|_F\|\mW\mSig_\mA^{-1/2}\|_F \\
	&= \mathrm{O}\Bigl(\tfrac{\epsilon \log^{1/2}(n)}{\rho^{1/2}n}\Bigr)
\end{align} 
by Proposition \ref{orth_spectral_bound} and Corollary \ref{A_bounds}.

\item Observe that 
\begin{align}
	\|\mR_{1,4}\|_{2 \to \infty} &\leq \|\mR_{1,4}\|_F \\
	& \leq \|\mA-\mP\|\|\mV_\mP\|_F\|\mW\mSig_\mA^{-1/2} - \mSig_\mP^{-1/2}\mW\|_F \\
	& = \mathrm{O}\Bigl(\tfrac{\epsilon^{9/2}k^{3/4}\log^{3/2}(n)}{n}\Bigr)
\end{align} 
by Propositions \ref{spectral_bound} and \ref{W_bounds}.
\end{enumerate}
\end{proof}

\subsection*{Proof of Theorem \ref{consistency}}\label{consistency_proof}
\addcontentsline{toc}{subsection}{\nameref{consistency_proof}}

\begin{proof} We first consider the left embedding $\mX_\mA$. Observe that 
\begin{align}
	\mX_\mA - \mX_\mP\mW &= \mU_\mA \mSig_\mA^{1/2} - \mU_\mP \mSig_\mP^{1/2} \mW \\
	&= \mU_\mA \mSig_\mA^{1/2} - \mU_\mP\mU_\mP^\top\mU_\mA\mSig_\mA^{1/2} + \mU_\mP(\mU_\mP^\top\mU_\mA\mSig_\mA^{1/2}- \mSig_\mP^{1/2} \mW) \\
	&= \mU_\mA \mSig_\mA^{1/2} - \mU_\mP\mU_\mP^\top\mU_\mA\mSig_\mA^{1/2} + \mR_1.
\end{align} 

Noting that $\mU_\mA\mSig_\mA^{1/2} = \mA\mV_\mA\mSig_\mA^{-1/2}$ and $\mU_\mP\mU_\mP^\top\mP = \mP$, we see that 
\begin{align}
	\mX_\mA - \mX_\mP\mW &= \mA\mV_\mA\mSig_\mA^{-1/2} - \mU_\mP\mU_\mP^\top\mA\mV_\mA\mSig_\mA^{-1/2} + \mR_{1,1} \\
	&= (\mA-\mP)\mV_\mA\mSig_\mA^{-1/2} - (\mU_\mP\mU_\mP^\top\mA - \mP)\mV_\mA\mSig_\mA^{-1/2} + \mR_{1,1} \\
	&= (\mA-\mP)\mV_\mA\mSig_\mA^{-1/2} - \mU_\mP\mU_\mP^\top(\mA-\mP)\mV_\mA\mSig_\mA^{-1/2} + \mR_{1,1} \\
	&= (\mI - \mU_\mP\mU_\mP^\top)(\mA-\mP)\mV_\mA\mSig_\mA^{-1/2} + \mR_{1,1} \\
	&= (\mI-\mU_\mP\mU_\mP^\top)(\mA-\mP)\bigl(\mV_\mP\mW + (\mV_\mA - \mV_\mP\mW)\bigr)\mSig_\mA^{-1/2} + \mR_{1,1} \\
	&= (\mA-\mP)\mV_\mP\mW\mSig_\mA^{-1/2} + \mR_{1,3} + \mR_{1,2} + \mR_{1,1} \\
	&= (\mA-\mP)\mV_\mP\mSig_\mP^{-1/2}\mW + \mR_{1,4} + \mR_{1,3} + \mR_{1,2} + \mR_{1,1}.
\end{align}

Applying Proposition \ref{residual_bounds}, we find that 
\begin{align}
	\|\mX_\mA - \mX_\mP\mW\|_{2 \to \infty} &= \|(\mA-\mP)\mV_\mP\mSig_\mP^{-1/2}\|_{2 \to \infty} + \mathrm{O}\Bigl(\tfrac{\epsilon^2 k^{1/2}\log(n)}{\rho^{1/2} n^{3/4}}\Bigr) \\
	& \leq \sigma_{d}(\mP)^{-1/2}\|(\mA-\mP)\mV_\mP\|_{2 \to \infty} + \mathrm{O}\Bigl(\tfrac{\epsilon^2 k^{1/2}\log(n)}{\rho^{1/2} n^{3/4}}\Bigr).
\end{align}
almost surely.

We condition on some set of latent positions. For any $i \in \{1,\ldots,n\}$, $j \in \{1,\ldots,d\}$ and $r \in \{1,\ldots,k\}$, let 
\begin{align}
	\mE_{ij}^{(r)} = \left\{\begin{array}{cc}\displaystyle\sum_{l=1}^{n_r} \bigl(\mA^{(r)}_{il} - \mP^{(r)}_{il}\bigr)v^{(r)}_l& \mathrm{if~}\mA^{(r)} \mathrm{~is~bipartite}\\&\\\displaystyle\sum_{l \neq i} \bigl(\mA^{(r)}_{il} - \mP^{(r)}_{il}\bigr)v^{(r)}_l& \mathrm{otherwise}\end{array}\right.
\end{align} 
and 
\begin{align}
	\mF_{ij}^{(r)} = \left\{\begin{array}{cc}\mathbf{0}& \mathrm{if~}\mA^{(r)} \mathrm{~is~bipartite}\\&\\\mP^{(r)}_{ii}v^{(r)}_i& \mathrm{otherwise},\end{array}\right.\end{align} 
where $v^{(r)}$ denotes the $j$th column of $\mV^{(r)}_\mP$, so that 
\begin{align}
	\bigl((\mA-\mP)\mV_\mP\bigr)_{ij} = \sum_{r=1}^k \mE_{ij}^{(r)} - \sum_{r=1}^k \mF_{ij}^{(r)}.
\end{align}

 The latter term is of order $\mathrm{O}\bigl(\epsilon\rho k^{1/2}\bigr)$ (as can be seen by applying the Cauchy--Schwarz inequality) and thus can be discounted from our asymptotic analysis. The former is a sum of independent, zero-mean random variables, with each individual term bounded in absolute value by $|v^{(r)}_l|$, and thus we can apply Hoeffding's inequality to see that 

\begin{align}
	\mathbb{P}\Biggl(\Bigl|\sum_{r=1}^k\mE_{ij}^{(r)}\Bigr| > t\Biggr) \leq 2 \exp\Biggl(\frac{-2t^2}{4\sum_{r=1}^k\sum_{l=1}^{n_r}|v^{(r)}_l|^2}\Biggr) = 2\exp\Bigl(\frac{-t^2}{2}\Bigr).
\end{align}

 Thus $\bigl((\mA-\mP)\mV_\mP\bigr)_{ij} = \mathrm{O}\bigl(\log^{1/2}(n)\bigr)$ almost surely, and hence $\bigl|\bigl((\mA-\mP)\mV_\mP\bigr)_i\bigr| = \mathrm{O}\bigl(\log^{1/2}(n)\bigr)$ almost surely by summing over all $j \in \{1,\ldots,d\}$. Taking the union bound over all $n$ rows then shows that 
\begin{align}
	\sigma_{d}(\mP)^{-1/2}\|(\mA-\mP)\mV_\mP\|_{2 \to \infty} = \mathrm{O}\Bigl(\tfrac{\log^{1/2}(n)}{\rho^{1/2}n^{1/2}}\Bigr),
\end{align} 
almost surely and consequently that 
\begin{align}
	\|\mX_\mA - \mX\mL\|_{2 \to \infty} = \mathrm{O}\Bigl(\tfrac{\log^{1/2}(n)}{\rho^{1/2}n^{1/2}}\Bigr)
\end{align} 
almost surely by setting $\mL = \widetilde{\mL} \mW$. The second bound follows from Corollary \ref{L_matrix_norm} and the fact that $\|\mA\mB\|_{2 \to \infty} \leq \|\mA\|_{2 \to \infty} \|\mB\|$. Integrating over all possible sets of latent positions gives the result.

A similar argument is used for the right embedding. Observe that 
\begin{align}
	\mY_\mA - \mY_\mP\mW &= \mV_\mA \mSig_\mA^{1/2} - \mV_\mP \mSig_\mP^{1/2} \mW \\
	&= \mV_\mA \mSig_\mA^{1/2} - \mV_\mP\mV_\mP^\top\mV_\mA\mSig_\mA^{1/2} + \mV_\mP(\mV_\mP^\top\mV_\mA\mSig_\mA^{1/2}- \mSig_\mP^{1/2} \mW) \\
	&= \mV_\mA \mSig_\mA^{1/2} - \mV_\mP\mV_\mP^\top\mV_\mA\mSig_\mA^{1/2} + \mR_{2,1}.
\end{align}

Noting that $\mA^\top\mU_\mA\mSig_\mA^{-1/2} = \mV_\mA\mSig_\mA^{1/2}$ and $\mV_\mP\mV_\mP^\top\mP^\top = \mP^\top$, we see that
\begin{align}
	\mY_\mA - \mY_\mP\mW &= \mA^\top\mU_\mA\mSig_\mA^{-1/2} - \mV_\mP\mV_\mP^\top\mA^\top\mU_\mA\mSig_\mA^{-1/2} + \mR_{2,1} \\
	&= (\mA-\mP)^\top\mU_\mA\mSig_\mA^{-1/2} - (\mV_\mP\mV_\mP^\top\mA^\top - \mP^\top)\mU_\mA\mSig_\mA^{-1/2} + \mR_{2,1}\\
	&= (\mA-\mP)^\top\mU_\mA\mSig_\mA^{-1/2} - \mV_\mP\mV_\mP^\top(\mA-\mP)^\top\mU_\mA\mSig_\mA^{-1/2} + \mR_{2,1} \\
	&= (\mI - \mV_\mP\mV_\mP^\top)(\mA-\mP)^\top\mU_\mA\mSig_\mA^{-1/2} + \mR_{2,1} \\
	&= (\mI-\mV_\mP\mV_\mP^\top)(\mA-\mP)^\top(\mU_\mP\mW + (\mU_\mA - \mU_\mP\mW))\mSig_\mA^{-1/2} + \mR_{2,1} \\
	&= (\mA-\mP)^\top\mU_\mP\mW\mSig_\mA^{-1/2} + \mR_{2,3}+ \mR_{2,2} + \mR_{2,1} \\
	&= (\mA-\mP)^\top\mU_\mP\mSig_\mP^{-1/2}\mW + \mR_{2,4} + \mR_{2,3} + \mR_{2,2} + \mR_{2,1}.
\end{align}

Applying Proposition \ref{residual_bounds} once more, we find that 
\begin{align}
	\|\mY_\mA - \mY_\mP\mW\|_{2 \to \infty} = \|(\mA-\mP)^\top\mU_\mP\mSig_\mP^{-1/2}\|_{2 \to \infty} + \mathrm{O}\Bigl(\tfrac{\epsilon^2 k^{1/2}\log(n)}{\rho^{1/2} n^{3/4}}\Bigr)
\end{align} 
almost surely and consequently that 
\begin{align}
	\|\mY^{(r)}_\mA - \mY^{(r)}_\mP\mW\|_{2 \to \infty} \leq \sigma_{d}(\mP)^{-1/2}\|(\mA^{(r)}-\mP^{(r)})^\top\mU_\mP\|_{2 \to \infty} + \mathrm{O}\Bigl(\tfrac{\epsilon^2 k^{1/2}\log(n)}{\rho^{1/2} n^{3/4}}\Bigr).
\end{align} 
almost surely.

If we condition on a set of latent positions, an identical argument to before shows that 
\begin{align}
	\sigma_d(\mP)^{-1/2}\|(\mA^{(r)}-\mP^{(r)})^\top\mU_\mP\|_{2 \to \infty} = \mathrm{O}\Bigl(\tfrac{\log^{1/2}(n)}{\rho^{1/2}n^{1/2}}\Bigr),
\end{align} 
and consequently that 
\begin{align}
	\|\mY^{(r)}_\mA - \mY\mR_r\|_{2 \to \infty} = \mathrm{O}\Bigl(\tfrac{\log^{1/2}(n)}{\rho^{1/2}n^{1/2}}\Bigr)
\end{align} 
by setting $\mR_r = \widetilde{\mR}_r \mW$. Proposition \ref{L_matrices} shows that $\widetilde{\mL}\widetilde{\mR}_r^\top = \mLam_r$, and the remaining bounds follow from Corollary \ref{L_matrix_norm} as for the left embedding. As before, integrating over all possible sets of latent positions gives the final result. 
\end{proof}

\subsection*{Proof of Theorem \ref{CLT}}\label{CLT_proof}
\addcontentsline{toc}{subsection}{\nameref{CLT_proof}}

\begin{proof} We first consider the left embedding $\mX_\mA$. Recall from the proof of Theorem \ref{consistency} that 
\begin{align}
	n^{1/2}(\mX_\mA \mL^{-1} - \mX) = n^{1/2}(\mA-\mP)\mV_\mP\mSig_\mP^{-1/2}\widetilde{\mL}^{-1} + n^{1/2}\mR,
\end{align} 
where the residual term $\mR$ satisfies $\|n^{1/2}\mR\|_{2 \to \infty} \to 0$ by Proposition \ref{residual_bounds} and our assumptions in Section \ref{subsec:asymptotics}.

The first of the right-hand terms may be rewritten as 
\begin{align}
	n^{1/2}(\mA-\mP)\mV_\mP\mSig_\mP^{-1/2}\widetilde{\mL}^{-1} = n^{1/2}\sum_{r=1}^k\bigl(\mA^{(r)}-\mP^{(r)}\bigr)\mY^{(r)}\widetilde{\mR}_r\mSig_\mP^{-1}\widetilde{\mL}^{-1}
\end{align} 
by splitting $(\mA-\mP)\mV_\mP$ into the individual terms $\bigl(\mA^{(r)}-\mP^{(r)}\bigr)\mV^{(r)}_\mP$ and noting that 
\begin{align}
	\mV^{(r)}_\mP \mSig_\mP^{-1/2} = \mY^{(r)}_\mP\mSig_\mP^{-1} = \mY^{(r)}\widetilde{\mR}_r\mSig_\mP^{-1}.
\end{align}

Consequently, 
\begin{align}
	n^{1/2}(\mX_\mA \mL^{-1} - \mX)_i^\top &\to n^{1/2}\sum_{r=1}^k(\widetilde{\mR}_r\mSig_\mP^{-1}\widetilde{\mL}^{-1})^\top\bigl[(\mA^{(r)}-\mP^{(r)})\mY^{(r)}\bigr]_i^\top\\
	&= n (\mLam_\epsilon\mY^\top\mY\mLam_\epsilon^\top)^{-1}\sum_{r=1}^k\Bigl[\tfrac{1}{n^{1/2}}\sum_{j=1}^{n_r} (\mA^{(r)}_{ij} - \mP^{(r)}_{ij})\mLam_{\epsilon,r}\mY^{(r)}_j\Bigr]
\end{align}
almost surely by Proposition \ref{L_matrix_identities}.

Noting that $\mY^{(r)}_j = \rho^{1/2}\upsilon^{(r)}_j$ and that $n(\mLam_\epsilon\mY^\top\mY\mLam_\epsilon^\top)^{-1} \to \tfrac{1}{\rho}\Delta_{\mLam,Y}^{-1}$ almost surely by the law of large numbers, we see that 
\begin{align}
	n^{1/2}(\mX_\mA \mL^{-1} - \mX)_i^\top \to \Delta_{\mLam,Y}^{-1}\sum_{r=1}^k\Bigl[\tfrac{1}{\rho^{1/2} n^{1/2}}\sum_{j= 1}^{n_r} (\mA^{(r)}_{ij} - \mP^{(r)}_{ij})\mLam_{\epsilon,r}\upsilon^{(r)}_j\Bigr]
\end{align} 
almost surely.

If $\mA^{(r)}$ is not bipartite, we may disregard the diagonal terms in our asymptotic analysis, as each of the $k$ such terms satisfies 
\begin{align} 
	\bigl\|\tfrac{1}{\rho^{1/2} n^{1/2}}\mP^{(r)}_{ii}\mLam_{\epsilon,r}\upsilon^{(r)}_i\bigr\| \to 0
\end{align} 
almost surely.

Assume first that each of the matrices $\mY^{(r)}$ is independent of the others. Conditional on $\xi_i = \mx$, we have $\mP^{(r)}_{ij} = \rho\mx^\top\mLam_{\epsilon,r}\upsilon^{(r)}_j$, and so 
\begin{align}
	\tfrac{1}{\rho^{1/2}n^{1/2}}\sum_{j} (\mA^{(r)}_{ij} - \mP^{(r)}_{ij})\mLam_{\epsilon,r}\upsilon^{(r)}_j
\end{align} 
is a scaled sum of independent, identically distributed, zero-mean random variables, each with covariance matrix given by 
\begin{align}
	\mathbb{E}\bigl[\mx^\top\mLam_{\epsilon,r}\upsilon_r(1 - \rho\mx^\top\mLam_{\epsilon,r}\upsilon_r)\cdot\mLam_{\epsilon,r}\upsilon_r\upsilon_r^\top\mLam_{\epsilon,r}^\top\bigr]
\end{align} 
which implies (by the multivariate central limit theorem) that 
\begin{align}
	\tfrac{1}{\rho^{1/2} n^{1/2}}\sum_{j} (\mA^{(r)}_{ij} - \mP^{(r)}_{ij})\mLam_{\epsilon,r}\upsilon^{(r)}_j \to \mathcal{N}\bigl(\mathbf{0},c_r\mLam_r\mSig^{(r)}_Y(\mx)\mLam_r^\top\bigr)
\end{align} 
if $\epsilon_r = 1$, and vanishes otherwise, and thus we find that 
\begin{align}
	n^{1/2}(\mX_\mA \mL^{-1} - \mX)_i^\top \to \mathcal{N}\bigl(\mathbf{0},\Delta_{\mLam,Y}^{-1}\mLam_*\mSig_Y(\mx)\mLam_*^\top\Delta_{\mLam,Y}^{-1}\bigr)
\end{align} 
almost surely, where $\mSig_Y(\mx) = c_1\mSig_Y^{(1)}(\mx) \oplus \cdots \oplus c_\kappa\mSig_Y^{(\kappa)}(\mx)$. We deduce the Central Limit Theorem by integrating over all possible values of $\mx \in \mathcal{X}$.

The same statement holds even if there is dependence between any of the matrices $\mY^{(r)}$. Indeed, suppose that the matrices $\mY^{(r)}$ are dependent for all $r \in \mathcal{K}$, for some subset $\mathcal{K}$ of $\{1,\ldots,k\}$. Then for any $i$ and $j$, 
\begin{align}
	\tfrac{1}{\rho^{1/2} n^{1/2}}\sum_{j} \sum_{r\in \mathcal{K}}(\mA^{(r)}_{ij} - \mP^{(r)}_{ij})\mLam_{\epsilon,r}\upsilon^{(r)}_j
\end{align} 
is a sum of independent, identically distributed, zero-mean random variables, for which the covariance matrices 
\begin{align}
	\mathbb{E}\Bigl[\sum_{r,s \in \mathcal{K}}(\mA^{(r)}_{ij} - \mP^{(r)}_{ij})(\mA^{(s)}_{ij}-\mP^{(s)}_{ij})\cdot\mLam_{\epsilon,r}\upsilon^{(r)}_j{\upsilon^{(s)}_j}^\top\mLam_{\epsilon,s}^\top\Bigr]
\end{align}
reduce to 
\begin{align}
	\mathbb{E}\Bigl[\sum_{r \in \mathcal{K}} \mx^\top\mLam_{\epsilon,r}\upsilon_r(1 - \rho\mx^\top\mLam_{\epsilon,r}\upsilon_r)\cdot\mLam_{\epsilon,r}\upsilon_r\upsilon_r^\top\mLam_{\epsilon,r}^\top\Bigr]
\end{align} 
since the terms 
\begin{align}
	\mathbb{E}\Bigl[(\mA^{(r)}_{ij} - \mP^{(r)}_{ij})(\mA^{(s)}_{ij}-\mP^{(s)}_{ij})\cdot\mLam_{\epsilon,r}\upsilon^{(r)}_j{\upsilon^{(s)}_j}^\top\mLam_{\epsilon,s}^\top\Bigr]
\end{align}
vanish if $r$ and $s$ are not equal.

Similarly, for the right embedding we observe that 
\begin{align}
	n^{1/2}(\mY^{(r)}_\mA\mR_r^{-1} - \mY^{(r)}) = n^{1/2}(\mA^{(r)}-\mP^{(r)})^\top\mU_\mP\mSig_\mP^{-1/2}\widetilde{\mR}_r^{-1} + n^{1/2}\mR,
\end{align} 
where the residual term $\mR$ satisfies $\|n^{1/2}\mR\|_{2 \to \infty} \to 0$.

Again, we may rewrite the first of the right-hand terms, obtaining the expression 
\begin{align}
	n^{1/2}(\mA^{(r)}-\mP^{(r)})^\top\mU_\mP\mSig_\mP^{-1/2}\widetilde{\mR}_r^{-1} = n^{1/2}(\mA^{(r)}-\mP^{(r)})^\top\mX\widetilde{\mL}\mSig_\mP^{-1}\widetilde{\mR}_r^{-1}
\end{align} 
by noting that $\mU_\mP \mSig_\mP^{-1/2} = \mX\widetilde{\mL}\mSig_\mP^{-1}$, and consequently find that
\begin{align}
n^{1/2}(\mY^{(r)}_\mA\mR_r^{-1} - \mY^{(r)}) &\to n^{1/2}(\widetilde{\mL}\mSig_\mP^{-1}\widetilde{\mR}_r^{-1})^\top\bigl[(\mA^{(r)}-\mP^{(r)})^\top\mX\bigr]_i^\top \\
	&= n\mLam_{\epsilon,r}^{-1}(\mX^\top\mX)^{-1}\Bigl[\tfrac{1}{n^{1/2}}\sum_{j=1}^n (\mA^{(r)}_{ji} - \mP^{(r)}_{ji})\mX_j\Bigr]
\end{align}
almost surely by Proposition \ref{L_matrix_identities}.

Noting that $\mX_j = \rho^{1/2}\xi_j$ and that $n(\mX^\top\mX)^{-1} \to \tfrac{1}{\rho}\Delta_X^{-1}$ almost surely by the law of large numbers, we see that 
\begin{align}
	n^{1/2}(\mY^{(r)}_\mA\mR_r^{-1} - \mY^{(r)}) \to \mLam_r^{-1}\Delta_X^{-1}\Bigl[\tfrac{1}{\rho^{1/2}n^{1/2}}\sum_{j=1}^n (\mA^{(r)}_{ji} - \mP^{(r)}_{ji})\xi_j\Bigr]
\end{align} 
almost surely.

As before, if $\mA^{(r)}$ is not bipartite we may disregard the diagonal term in our asymptotic analysis, as it satisfies 
\begin{align} 
	\bigl\|\tfrac{1}{\rho^{1/2}n^{1/2}}\mP^{(r)}_{ii}\xi_i\bigr\| \to 0
\end{align} 
almost surely.

 Conditional on $\upsilon^{(r)}_i = \my$, we have $\mP^{(r)}_{ji} = \rho\xi_j^\top\mLam_{\epsilon,r}\my$, and so 
\begin{align}
	\tfrac{1}{\rho^{1/2}n^{1/2}}\sum_j (\mA^{(r)}_{ji} - \mP^{(r)}_{ji})\xi_j
\end{align} 
is a scaled sum of independent, zero-mean random variables, each with covariance matrix given by 
\begin{align}
	\mathbb{E}\bigl[\xi^\top\mLam_{\epsilon,r}\my(1 - \rho\xi^\top\mLam_{\epsilon,r}\my)\cdot\xi\xi^\top\bigr]
\end{align} 
which implies (by the multivariate central limit theorem) that, provided $\epsilon_r = 1$, 
\begin{align}
	\tfrac{1}{\rho^{1/2}n^{1/2}}\sum_j (\mA^{(r)}_{ji} - \mP^{(r)}_{ji})\xi_j \to \mathcal{N}\bigl(\mathbf{0},\mSig^{(r)}_X(\my)\bigr),
\end{align}
and thus we find that, provided $\epsilon_r = 1$, 
\begin{align}
	n^{1/2}\bigl(\mY^{(r)}_\mA\mR_r^{-1} - \mY^{(r)}\bigr)_i^\top \to \mathcal{N}\bigl(\mathbf{0},\mLam_r^{-1}\Delta_X^{-1}\mSig_X^{(r)}(\my)\Delta_X^{-1}\mLam_r^{-\top}\bigr)
\end{align} 
almost surely. We deduce the Central Limit Theorem by integrating over all possible values of $\my \in \mathcal{Y}_r$. 
\end{proof}

\end{document}